\documentclass[11pt]{article} 

\usepackage[letterpaper,margin=1in]{geometry}

\usepackage{booktabs}       
\usepackage{array}  

\usepackage[T1]{fontenc}
\usepackage{microtype}
\usepackage{wrapfig}

\usepackage{bm}
\usepackage{times}
\usepackage[linesnumbered,ruled]{algorithm2e}
\usepackage{algorithmic}
\usepackage{multirow}
\usepackage{diagbox,xcolor}
\usepackage{titlesec}
\titleformat*{\paragraph}{\bfseries}

\usepackage{amsthm,amsmath,amssymb,amsfonts,amssymb,mathtools}
\usepackage{amsmath,amssymb,amsfonts,amssymb,mathtools}
\usepackage{xcolor}
\usepackage{empheq}
\usepackage{dsfont}
\usepackage{mathrsfs}
\usepackage{graphicx}
\usepackage{enumitem}
\usepackage{verbatim}
\usepackage{aliascnt} 
\usepackage{xspace}
\usepackage{booktabs}
\usepackage{bbm}

\usepackage{caption}
\usepackage{tikz}
\usetikzlibrary{patterns}
\usetikzlibrary{arrows,shapes,automata,backgrounds,petri,positioning}
\usetikzlibrary{shadows}
\usetikzlibrary{calc}
\usetikzlibrary{spy}
\usetikzlibrary{angles, quotes}
\usetikzlibrary{matrix}
\usepackage{pgf,pgfplots}
\usepackage{pgfmath,pgffor}
\pgfplotsset{compat=1.17}

\usepackage{framed}

\definecolor[named]{ACMBlue}{cmyk}{1,0.1,0,0.1}
\definecolor[named]{ACMYellow}{cmyk}{0,0.16,1,0}
\definecolor[named]{ACMOrange}{cmyk}{0,0.42,1,0.01}
\definecolor[named]{ACMRed}{cmyk}{0,0.90,0.86,0}
\definecolor[named]{ACMLightBlue}{cmyk}{0.49,0.01,0,0}
\definecolor[named]{ACMGreen}{cmyk}{0.20,0,1,0.19}
\definecolor[named]{ACMPurple}{cmyk}{0.55,1,0,0.15}
\definecolor[named]{ACMDarkBlue}{cmyk}{1,0.58,0,0.21}

\usepackage[colorlinks,citecolor=blue,linkcolor=magenta,bookmarks=true]{hyperref}
\usepackage[nameinlink]{cleveref}
\crefname{ineq}{Inequality}{Inequality}
\creflabelformat{ineq}{#2{\upshape(#1)}#3}
\crefname{sub}{Subsection}{Subsection}
\creflabelformat{Subsection}{#2{\upshape(#1)}#3}
\crefname{sdp}{SDP}{SDP}
\creflabelformat{sdp}{#2{\upshape(#1)}#3}
\crefname{lp}{LP}{LP}
\creflabelformat{lp}{#2{\upshape(#1)}#3}
\usepackage{aliascnt}
\usepackage{cleveref}
\crefname{ineq}{Inequality}{Inequality}
\creflabelformat{ineq}{#2{\upshape(#1)}#3}
\crefname{sub}{Subsection}{Subsection}
\creflabelformat{Subsection}{#2{\upshape(#1)}#3}
\crefname{sdp}{SDP}{SDP}
\creflabelformat{sdp}{#2{\upshape(#1)}#3}
\crefname{lp}{LP}{LP}
\creflabelformat{lp}{#2{\upshape(#1)}#3}



\newcommand{\ac}{\mathsf{AC}_0}

\newtheorem{theorem}{Theorem}[section]

\newtheorem{lemma}[theorem]{Lemma}
\newtheorem{informal theorem}[theorem]{Theorem (informal statement)}

\newtheorem{proposition}[theorem]{Proposition}
\newtheorem{corollary}[theorem]{Corollary}

\newtheorem{fact}[theorem]{Fact}

\newtheorem{definition}[theorem]{Definition}

\newcommand\norm[1]{\left\| #1 \right\|}

\DeclareMathOperator*{\E}{\mathbb{E}}

\def\multichoose#1#2{\ensuremath{\left(\kern-.3em\left(\genfrac{}{}{0pt}{}{#1}{#2}\right)\kern-.3em\right)}}

\newcommand{\R}{\mathbb{R}}

\newcommand{\poly}{\mathrm{poly}}

\newcommand{\cube}[1]{\{\pm 1\}^{#1}}

\newcommand{\citet}{\cite}
\newcommand{\citep}{\cite}

\newcommand{\opnorm}[1]{\|#1\|_{\mathrm{op}}}
\usepackage{titlesec}

\usepackage{float}

\crefname{equation}{}{}
\crefrangeformat{equation}{(#3#1#4) --~(#5#2#6)}
\crefmultiformat{equation}{(#2#1#3)}{, (#2#1#3)}{, (#2#1#3)}{, (#2#1#3)}
\crefname{lem}{Lemma}{Lemmas}
\crefname{section}{Section}{Sections}
\crefname{subsubsubsection}{Section}{Sections}
\crefname{rem}{Remark}{Remarks}
\crefname{cor}{Corollary}{Corollaries}
\crefname{figure}{Figure}{Figures}
\crefname{table}{Table}{Tables}
\Crefname{lem}{Lemma}{Lemmas}
\newcommand{\mutwonorm}[2]{\|#1\|_{2,#2}}
\Crefname{line}{Line}{Lines}
\Crefname{fact}{Fact}{Facts}
\crefname{thm}{Theorem}{Theorems}
\crefname{def}{Definition}{Definitions}
\crefname{assumption}{Assumption}{Assumptions}

\newcommand{\Stab}{\operatorname{Stab}}
\newcommand{\NS}{\operatorname{NS}}

\newcommand{\Vr}{\mathcal{P}_{r}}

\newcommand{\Var}{\operatorname{Var}}
\newcommand{\Cov}{\operatorname{Cov}}

\newcommand{\dist}{\operatorname{dist}}

\titlespacing*{\paragraph}
  {0pt}   
  {1ex plus 0.6ex minus 0.2ex}   
  {1em}   
\title{Low-Degree Learning from Localization Schemes\thanks{This paper subsumes an earlier version of this work \cite{DBLP:journals/corr/abs-2604-06109} by the same authors.}}\author{
Gautam Chandrasekaran\footnote{University of Texas at Austin. Email: \texttt{gautamc@cs.utexas.edu}} 
 \and
Jason Gaitonde\footnote{Duke University. Email: \texttt{jason.gaitonde@duke.edu}} 
\and
Ankur Moitra\footnote{Massachusetts Institute of Technology. Email: \texttt{moitra@mit.edu}}
\and
Arsen Vasilyan\footnote{Aarhus University. Email: \texttt{arsenvasilyan@gmail.com}} 
}
\usepackage{enumitem}

\begin{document}
\date{}
\maketitle
\thispagestyle{empty}
\begin{abstract}
    The low-degree algorithm is a cornerstone in computational learning theory. For product or Gaussian distributions, it shows that efficiently learnability follows from Fourier concentration. However, a longstanding obstacle towards building such a theory over nonproduct distributions is that we can no longer rely on having closed-form expressions for a good basis of orthogonal functions. This hints that the analysis of the low-degree algorithm might be brittle to the underlying distributional assumptions.

    We refute this intuition by showing that the low-degree algorithm provably succeeds for a large class of graphical models under mild high-temperature conditions. As an application, for nearly every Ising model that is known to rapidly mix, we obtain learning algorithms for every function class for which the Fourier approach succeeds for product distributions (e.g. low-depth circuits, halfspaces, intersections, monotone, among others) with no significant loss. Our framework is based on a strikingly clean connection with the \emph{coordinate-by-coordinate} localization scheme of Chen and Eldan introduced for establishing rapid mixing in high-temperature models, as well as other analytical \emph{transference} principles to product distributions.
\end{abstract}
\newpage
\setcounter{page}{1}

\section{Introduction}
A cornerstone of computational learning theory is the relation between Fourier concentration and learnability. In particular, we say that a class of functions $\mathcal{C}$ has Fourier concentration if the following holds:
\begin{center}
    \parbox{0.75\textwidth}{%
    \centering
    \emph{
        For every function $f$ in $\mathcal{C}$ on $n$ inputs, all but an $\varepsilon$-fraction of the Fourier\\ spectrum of $f$ is concentrated on degree up to $\alpha(\varepsilon, n)$.}
}
\end{center}
In a landmark work, Linial Mansour and Nisan showed that if a class of functions has Fourier concentration then it can be PAC learned under the uniform distribution \cite{lmn}. Furthermore, they made a striking connection to Hastad's Switching Lemma \cite{hastad1986almost} and through this showed that the class of $\mathsf{AC}^0$ circuits has Fourier concentration at degree $\alpha(\varepsilon,n)=\text{polylog}(n/\varepsilon)$ and hence can be PAC learned under the uniform distribution in quasipolynomial time.

In the intervening years this theory has been strengthened and extended in many different ways. Klivans, O'Donnell and Servedio showed how to get Fourier concentration bounds by going through noise sensitivity rather than studying the Fourier properties directly \cite{klivans2004learning}. This allowed them to get the first polynomial time algorithms for learning functions of any constant number of halfspaces up to any constant $\varepsilon>0$ error, again under the uniform distribution. In another foundational work, Klivans, O'Donnell and Servedio showed how to adapt the theory to the setting of learning under the Gaussian distribution \cite{klivans2008learning}. The key is the Hermite expansion of functions $f \in \mathcal{C}$ and bounds on their Hermite concentration. The main takeaway is that other families of orthogonal polynomials can be a stand-ins for the Fourier basis, and it is possible to carry over important concepts like concentrability and noise sensitivity. As a consequence they obtained a meta-theorem that function classes with bounded Gaussian surface area can be learned under the Gaussian distribution, and they instantiated this theory for halfspaces, intersections of halfspaces, balls and arbitrary convex sets.

It turns out that concentrability implies not just learnability but also various forms of robust learnability. The meta algorithm of Linial Mansour and Nisan is called \emph{the low-degree algorithm}. Essentially, it computes the orthogonal projection of $f\in \mathcal{C}$ onto the space of low-degree polynomials. Fourier concentration implies that this projection is close to $f$ under the uniform or Gaussian distribution. In the agnostic learning model we are given samples $(x,y)$ where $y$ is the label but we no longer assume $y=f(x)$ for some unknown $f \in \mathcal{C}$. Rather, without making any assumptions on the labels, we want to find the $f \in \mathcal{C}$ that approximately minimizes the classification error.\footnote{To be fully precise, in the \emph{improper} agnostic learning setting we discuss here the predictor $h$ does not itself have to be in $\mathcal{F}$ but it needs to be $\varepsilon$-competitive with the best predictor in $\mathcal{F}$.} In a breakthrough work, Kalai, Klivans, Mansour and Servedio  showed that $\ell_1$-regression on the space of low-degree polynomials yields agnostic learners under Fourier concentration \cite{kalai2008agnostically}. As a consequence they gave the first agnostic learning algorithm for halfspaces under the uniform distribution. In a major work, Awasthi, Balcan and Long gave a semiagnostic learner for halfspaces under isotropic log-concave distributions that runs in polynomial time and removes the doubly exponential dependence on the target accuracy $\varepsilon>0$ \cite{awasthi_power_2014}. Recently Klivans, Stavropoulos and Vasilyan \cite{KlivansSV25} showed how to combine low-degree learning with outlier removal to work under the even more challenging malicious noise model, where an adversary is allowed to corrupt both the covariates $x$ and labels $y$.  In particular, they gave a quasipolynomial time algorithm for learning $\mathsf{AC}^0$ circuits with respect to the uniform distribution under malicious noise.\footnote{They used stronger analytic notions than Fourier concentration, but subsequent work showed that Fourier concentration is all that is needed \cite{klivans2025power}.}

While we have made the case that Fourier concentration and polynomial regression are the engine behind many of the greatest hits of computational learning theory over the last thirty years, there is an uncomfortable truth that is widely known:
\begin{center}
    \parbox{\textwidth}{%
    \centering
    \emph{
        All of this hinges on having closed-form expressions for a good basis of orthogonal functions.
        }
}
\end{center}
In particular, under the uniform distribution we can as write the Fourier expansion of any function $f: \{-1, +1\}^n \rightarrow \R$ as follows

\[
f(x)
=
\sum_{S \subseteq [n]}
\hat{f}_S \chi_S(x),\text{ where } \chi_S(x)=\prod_{i \in S} x_i
\]

\noindent Now Fourier concentration amounts to the statement that for any $f \in \mathcal{C}$ we have
\[
\sum_{\substack{
S \subseteq [n]
\\
|S| > \alpha(\varepsilon, n)
}}
\hat{f}_S^2
\leq \varepsilon
\]
Moreover both the influence and noise sensitivity have exact expressions in terms of the Fourier coefficients. Thus bounds on them can be translated into statements about Fourier concentration.

But learning under the uniform or Gaussian distribution is rarely practically motivated. At best we hope it is a convenient starting point for building a more general theory. A dream scenario would be that for a wide variety of distributions that might be generated by nature and any function class that is similarly stable under these distributions, we can still hope to learn them efficiently via natural algorithms. Even more optimistically, these algorithms might themselves be general purpose by not requiring explicit distributional knowledge in their specification. But what could such a general theory look like? And how could we possibly get it off the ground if there are no closed-form expressions for families of orthogonal functions with respect to the distribution? 

In a work that ought to be better known, Kanade and Mossel gave some first steps towards such a theory~\cite{DBLP:conf/colt/KanadeM15}. Specifically, they asked whether results about learnability under the uniform distribution could be transferred to distributions where there is a rapidly mixing Markov chain for sampling from them. Their main idea is that the most direct generalization of the low-degree algorithm for an arbitrary distribution can substitute low-degree polynomials with \emph{eigenfunctions of the Glauber dynamics}, a canonical Markov chain generalizing the random walk on $\{\pm 1\}^n$ that is reversible with respect to $\mu$. Now just as low-degree polynomials form a canonical space of stable functions under the noise operator, so too the span of the low-energy eigenfunctions should define a natural space of stable functions under the Glauber dynamics.

However, directly transferring the blueprint of how to analyze the low-degree algorithm turns out to run into a number of technical barriers. The trouble is that not much is known about the fine-grained eigenstructure of rapidly mixing distributions \cite{DBLP:conf/stoc/KoehlerLMM24,DBLP:journals/corr/abs-2609-24687}. Kanade and Mossel sidestep these issues by instead making strong and unproven assumptions about the spectrum of the transition matrix, namely that it is approximately discrete with sharp drops between clusters. They were able to numerically check small examples, such as an Ising model on a $10$ node cycle for various values of the inverse temperature $\beta$. Furthermore they need to assume being given a well-conditioned starting basis of functions, whose evolution through the power method they then track. Under the discrete spectrum assumption, they showed that halfspaces under bounded degree Ising models are agnostically learnable.


In this work, we sidestep these difficulties by changing the \emph{noise process} in a way that is naturally adapted to polynomial degree. As we will show, an alternative way to inject noise relative to a distribution $\mu$ is according to the noise sensitivity:
\begin{equation}
\label{eqn:noise_sensitivity_definition_intro}
    \text{NS}_{\delta, \mu}(f)
    =
    \Pr_{\bm{x}\sim \mu, \bm{y}\sim N_{\delta, \mu}(\bm{x})}[f(\bm{x}) \neq f(\bm{y}) ],
\end{equation}
where $N_{\delta, \mu}$ erases coordinates in $x$ independently with probability $\delta$ and then resamples them according to right conditional distribution with respect to $\mu$. The associated noise operator $T_{1-\delta,\mu}$ is 
\begin{equation}
\label{eqn:noise_operator_intro}
T_{1-\delta,\mu}f(\bm{x})=\E_{\bm{y}\sim N_{\delta}(\bm{x})}[f(\bm{y})].
\end{equation}
Note that the Glauber dynamics and this random block resampling are equivalent for products up to time-change. But for general distributions, this noise process is well-studied in the sampling literature as the (time-reversed) \emph{coordinate-by-coordinate localization scheme}\footnote{In the information theory literature, this noise process is the binary erasure channel that progressively masks coordinates; it is also the time-reversal of ``masked diffusion'' in the sampling literature.} of Chen and Eldan~\cite{DBLP:conf/focs/ChenE22}. Their work introduced localization schemes generally as a unifying framework to understand mixing time bounds for the types of distributions we hope to learn from. Their approach made explicit via powerful connections to recent approaches like spectral and entropic independence that had previously introduced through the language of high-dimensional expansion~\cite{DBLP:journals/siamcomp/AnariLG24,DBLP:conf/stoc/AnariJKPV22}.
As we discuss in more detail in \Cref{sec:what-noise}, the Chen-Eldan philosophy (as explicated by El Alaoui and Montanari~\cite{el2022information,montanari2023sampling}) is that as we vary $\delta$, the posterior distributions of $X\sim \mu$ encountered along the path captures important statistical information about any fixed test function $f$. The existing technology gives bounds on the Poincare inequality, which controls how fast any non-constant function converges to the constant function. But can we prove quantitatively stronger bounds for how fast functions that are orthogonal to the space of low-degree polynomials decay? Towards this end, an important observation is that if a function $f$ is orthogonal to all low-degree functions, we provably get \emph{no} information about $f$ when revealing only a small number of coordinates ($\delta\approx 1)$; this direct relationship between localization and degree is precisely what will make the analysis tractable compared to prior attempts.

Once we have set up this framework, the resulting low-degree approximation theory turns out to be remarkably clean and general. By leveraging well-studied high-temperature properties that naturally arise to control the evolution of $f(X)$ along the erasure process, we will be able to give strong bounds on the noise-sensitivity in \Cref{eqn:noise_sensitivity_definition_intro} for all $\delta\in [0,1]$ when $f$ is orthogonal to low-degree polynomials.  Two of our main results, which form the backbone of our approach, only need the Markov chain to be robustly mixing, as captured by the following condition:
\begin{center}
    \parbox{0.9\textwidth}{%
    \centering
 
        \textbf{Pinned Poincar\'e condition.}
   \emph{For some $0<\alpha\leq 1$, every pinning $\nu$ of $\mu$,
on any set $F$ of free coordinates, satisfies
$
\Var_{\nu}(h)
\leq \frac{1}{\alpha}\sum_{i\in F}
\mathbb{E}_{\nu}\left[
\Var_{\nu}(h\mid X_{F\setminus\{i\}})
\right]
$
for every $h:\{\pm1\}^F\to\mathbb{R}$.}
}
\end{center}
In other words, whenever we condition on some $n-k$ coordinates, the induced Markov chain on the remaining $k$ coordinates needs to have a spectral gap of at least $\alpha/k$.

Our main theorem states that functions orthogonal to low-degree polynomials experience exponential decay under the noise operator when it satisfies the pinned Poincar\'e inequality. To formalize this property, we make the following definition that captures the rate of decay for every polynomial degree under noise:
\begin{definition}[$(c,\beta)$-decay]
\label{defn:noise_decay}
    A distribution $\mu$ over $\{\pm 1\}^n$ exhibits $(c,\beta)$-decay if the following holds: for any positive integer $r$, $\delta\in (0,1)$ and function $g$ that is $\mu$-orthogonal to all polynomials of degree at most $r-1$, it holds that 
   \begin{equation*}
       \left\langle g, T_{1-\delta, \mu} g \right\rangle_\mu:=\mathbb{E}_{X\sim \mu}[g(X)\cdot T_{1-\delta,\mu}g(X)]\leq (1-\delta)^{cr^{\beta}}\cdot \|g\|_{2,\mu}^2.
   \end{equation*}
\end{definition}

With this key definition in hand, our first main theorem proves the following:

\begin{theorem}[Exponential decay under noise, informal]
\label{Thm: spectral decay}
If a distribution $\mu$ satisfies a pinned Poincar\'e condition with constant $\alpha>0$, then it satisfies $(c_{\alpha},\alpha)$-decay for a constant $c_{\alpha}>0$ depending only on the Poincar\'e constant. In fact, if $\mu$ satisfies additional mild high-temperature assumptions, then it satisfies $(c_{\alpha},1)$-decay.
\end{theorem}

More specifically, the stronger conclusion of \Cref{Thm: spectral decay} holds for  any graphical model with bounded degree, and Ising models with bounded or mean-field interactions. These mild assumptions thus allow for nearly \emph{every} high-temperature distribution whose mixing time has been bounded using general-purpose frameworks.\footnote{A notable exception of a well-studied high-temperature Ising model not captured by any of our assumptions is the zero-external field Sherrington-Kirkpatrick model recently claimed to have a spectral gap up to the critical inverse temperature $\beta=1$ by OpenAI~\cite{openai2026sherrington} (see also~\cite{du2026sampling,DBLP:journals/corr/abs-2609-30590} for slightly weaker Poincar\'e inequalities). These proofs go through a detailed analysis of the stochastic localization process that appears tailored to the zero-field setting and does not appear to directly establish a pinned Poincar\'e inequality. The analysis of Wang~\cite{DBLP:journals/corr/abs-2608-22159} does not depend on the external fields, and hence holds for all pinnings.} These conditions (see \Cref{sec:stability-decay} for specifics) are modular and easy to check, and we include well-studied models captured by these assumptions in \Cref{tab:families}.

To give a concrete example, recall that Ising models are distributions on $\{\pm 1\}^n$ defined according to the equation:
$$\mu(\bm{\sigma}) = \frac{\exp\left(\frac{1}{2}\cdot{\bm{\sigma}}^\top A {\bm{\sigma}} + {\bm{h}}^\top {\bm{\sigma}}\right)}{Z}.$$
Here $A\in \mathbb{R}^{n\times n}$ is a symmetric matrix called the interaction matrix, $\bm{h}$ is a vector of external fields, and $Z$ is the \emph{partition function} that ensures the distribution is properly normalized. A well-studied assumption on these models, particularly in the \emph{distribution learning literature} starting from the seminal work of Klivans and Meka~\cite{DBLP:conf/focs/KlivansM17}, is that the $\ell_1$-width is bounded as follows:
\[
\max_i \sum_{j\neq i} \vert A_{i,j}\vert\leq B.
\]
As we will show, this width condition is already sufficient to obtain the optimal $\beta=1$ decay in \Cref{Thm: spectral decay} in tandem with the pinned Poincar\'e inequality. We also show this conclusion holds under a bounded-interaction degree assumption up to critical temperatures, as well as for mean-field models like the Sherrington-Kirkpatrick model of spin glasses up to inverse temperature $1/2$~\cite{DBLP:journals/corr/abs-2608-22159}.

Returning to main applications, we then show that from $(c,\beta)$-decay one immediately can obtain the following reduction from noise sensitivity bounds (as defined in \Cref{eqn:noise_sensitivity_definition_intro}) to low-degree approximability and learning guarantees.

\begin{theorem}[Noise sensitivity and decay, informal]
\label{Thm: NS implies approximation}
Suppose that a distribution $\mu$ satisfies $(c,\beta)$-decay. Then if every $f \in \mathcal{C}$ 
satisfies  $\NS_{\delta,\mu}(f) \leq \varepsilon$, then $\mathcal{C}$ can be $\varepsilon$-approximated in either $\ell_1$ or $\ell_2$ with respect to $\mu$ by a polynomial of degree:
\begin{equation*}
    r=\left(\frac{4\log(4/\varepsilon)}{c\delta}\right)^{1/\beta}.
\end{equation*}
Hence, $\mathcal{C}$ can be learned using samples from $\mu$ up to accuracy $\varepsilon$ in time $n^{O(r)}$.
\end{theorem}

Thus, we do not need to build new bases of functions in a case-by-case manner. Rather, the moral which has been part of the canon of learning theory remains true: the low-degree algorithm is all you need!

Now we are in a position to methodically map out how uniform distribution learning results, both in the realizable and noisy settings, transfer to distributions that satisfy various mixing conditions. In particular, we 
 give concrete applications for learning a wide range of hypotheses classes for most high-temperature Ising models\footnotemark[4] studied in the literature. 
We do this by giving general \emph{noise sensitivity transfer theorems} for Ising models. For such models, we have the following noise sensitivity comparison result that holds under a mild bound on the operator norm of the interactions: 
\begin{theorem}[Noise sensitivity transfer theorem, informal]
\label{thm:noise_stability_transfer_intro}
Suppose that an Ising model $\mu_{A,\bm{h}}$ satisfies $\|A\|_{\mathsf{op}}\leq \lambda$. Then there exists a $C_{\lambda}>0$ depending only on $\lambda>0$ such that for any function class $\mathcal{C}$ that is closed under restriction,
\[
\sup_{f\in \mathcal{C}}\NS_{\delta, \mu}(f)
\leq
\sup_{f\in \mathcal{C}}
\NS_{C_{\lambda}\cdot \delta, U}(f).
\]
\end{theorem} Note that on the right of the inequality above we have the conventional noise sensitivity with respect to the \emph{uniform distribution}. The proof of this theorem crucially uses the powerful measure decomposition of Hubbard and Stratonovich \cite{hubbard} that decomposes Ising distributions into a mixture of product distributions induced by a Gaussian signal. This approach to establishing functional inequalities was recently re-introduced by Bauerschmidt and Bodineau~\cite{bauerschmidt} and has since been extensively applied to Ising models~\cite{DBLP:conf/focs/ChenE22,DBLP:conf/focs/LiuMRW24,du2026sampling}. A (to our knowledge, new) insight in our analysis is that the Gaussian noise process in \cite{hubbard} can be simulated by the noise process $N_{\delta,\mu}$ from \Cref{eqn:noise_sensitivity_definition_intro} by appropriately choosing the noise rate (see \Cref{sec:noise_stability_transfer} for more details). To obtain the final inequality, we then use a beautiful reduction of Blais, O'Donnell, and Wimmer~\cite{blais_polynomial_2010} that transfers any uniform noise sensitivity bound to all products.

By lifting known noise-sensitivity bounds over the uniform distribution via \Cref{thm:noise_stability_transfer_intro}, we obtain the corresponding noise sensitivity bounds over high temperature Ising distributions. These bounds, along with \Cref{Thm: spectral decay,Thm: NS implies approximation}, and the low-degree algorithm \cite{lmn, kalai2008agnostically}, gives us the following efficient algorithms for PAC learning over Ising models satisfying bounded width and a pinned Poincar\'e condition, even in the challenging agnostic learning model. 

\begin{table}[H]
\centering
\renewcommand{\arraystretch}{1.4}
\begin{tabular}{|c|c|>{\centering\arraybackslash}p{3cm}|}
\hline
Function Class & Learning run-time over uniform & Learning run-time over Ising models \\
\hline
Linear Classifiers (agnostic learning) & $n^{O(1/\varepsilon^2)}$
 \cite{kalai2008agnostically, diakonikolas2010bounded}
& $n^{\widetilde{O}_{\alpha, B}(1/\varepsilon^2)}$ \\
\hline
Intersections of $m$ Linear Classifiers & $n^{O(\log(m+1)/\varepsilon^2)}$ \cite{kane_average_2014}  & $n^{\widetilde{O}_{\alpha, B}(\log(m+1)/\varepsilon^2)}$ \\
\hline
Functions of $m$ Linear Classifiers & $n^{O(m^2/\varepsilon^2)}$  \cite{klivans2004learning,kalai2008agnostically, diakonikolas2010bounded}  & $n^{\widetilde{O}_{\alpha, B}(m^2/\varepsilon^2)}$ \\
\hline
Degree-$m$ Polynomial Thresholds& $n^{O(m/\varepsilon^2)}$  \cite{openai2026gotsmanlinial} & $n^{\widetilde{O}_{\alpha, B}(m/\varepsilon^2)}$ \\
\hline
Monotone Functions & $n^{O( \sqrt{n}/\varepsilon)}$ \cite{DBLP:journals/jacm/BshoutyT96} & $n^{\widetilde{O}_{\alpha, B}(\sqrt{n}/\varepsilon)}$ \\
\hline
$\mathsf{Poly}(n)$, Depth-$d$ Circuits & $n^{\log^{O(d)}(n/\varepsilon)}$ \cite{lmn}  & $n^{\widetilde{O}_{\alpha, B}(\log^{O(d)}(n/\varepsilon))}$ \\
\hline
\end{tabular}
\caption{Comparison of learning run-times between uniform distribution and Ising models like under $\alpha$-pinned Poincar\'e condition and $B$-bounded width (more generally, satisfying $(c,1)$-decay). 
The run-times for Linear Classifiers refer to agnostic learning (realizable learning is trivial for linear classifiers). For all other function classes listed here, these are the best known run-times in both realizable and agnostic settings.
We obtain polynomially weaker exponents when we only have a Poincar\'e inequality. See Section \ref{sec: concrete applications} for formal statements and details.
}
\label{tab:runtime}
\end{table}
For low-depth circuits, the low-degree approximation is proved through a slightly more refined approach given in \Cref{Sec: polynomial approximation for AC0} to obtain the correct $\varepsilon$-dependence. In addition to bounded width models, we also obtain the same learning results for Ising models that have bounded degree (up to the critical temperatures) or satisfy mean-field behavior like the Sherrington-Kirkpatrick model up to $\beta<1/2$~\cite{DBLP:journals/corr/abs-2608-22159}.

Our analysis reveals that Fourier analysis is convenient but is not necessary in order to analyze the low-degree algorithm under a wide range of graphical models, and especially Ising models. Returning to our earlier narrative, it turns out that our approach yields a useful and intuitive picture of the eigenvalues of the Glauber dynamics. In particular we show the following comparison inequalities for higher eigenvalues:

\begin{theorem}[\Cref{thm:eigen-graphical} and \Cref{thm:eigen-ising}, informal]
\label{thm:eigen-intro}
There is a constant $c>0$ such that following holds. Let $\mu$ be distribution over $\{\pm 1\}^n$ satisfying the Pinned Poincar\'e condition as well as mild additional high-temperature properties (c.f. \Cref{sec:stability-decay}). For every $1\leq k\leq 2^n$, the eigenvalues $0=\lambda_1(\mathcal{L}_{\mu})\leq \lambda_2(\mathcal{L}_{\mu})\leq \ldots \leq \lambda_{2^n}(\mathcal{L}_\mu)$ of the Laplacian matrix $\mathcal{L}_{\mu}$ of the Glauber dynamics for $\mu$ satisfy:
\begin{equation*}
    c\cdot \lambda_k(\mathcal{L}_U)\leq \lambda_k(\mathcal{L}_{\mu})\leq c^{-1}\cdot \lambda_k(\mathcal{L}_U),
\end{equation*}
where $\mathcal{L}_U$ is the Laplacian of the Glauber dynamics for the uniform distribution that resamples a randomly chosen bit. 
\end{theorem}
For context, the eigenvalues for the uniform distribution Laplacian are $\ell$ for $\ell=0,1,\ldots,n$ with multiplicity ${n \choose \ell}$ (see e.g.~\cite{levin2017markov}). The standard Poincar\'e inequality for rapid mixing is just the lower bound in \Cref{thm:eigen-intro} for $k=2$, so the above result shows that \emph{every} eigenvalue of high-temperature models approximates the corresponding eigenvalue of the hypercube in both directions up to constants. To the best of our knowledge, a structural statement of this form was not previously known and further evinces the power of the localization approach to understanding high-dimensional distributions.

\subsubsection*{AI Methodology} This paper developed over the course of the previous two years, dating back to December 2024. The original version of this paper (available at~\cite{DBLP:journals/corr/abs-2604-06109}) introduced the problem of understanding the low-degree algorithm for high-temperature spin systems. In that work, we developed several transference principles for low-degree approximations through sampler-inverter pairs, inspired by the classical argument of Jackson, Smith, and Furst~\cite{furst1991improved}. Our work obtained learning algorithms for $\mathsf{AC}^0$ and low-influence functions whenever the underlying graphical model satisfied \emph{strong spatial mixing} and additionally had polynomial growth of neighborhoods in the dependency graph. Moreover, we also established low-degree approximations for halfspaces by following the well-known template of Diakonikolas, Gopalan, Jaiswal, Servedio, and Viola~\cite{diakonikolas2010bounded} by using the Hubbard-Stratanovich transform to obtain the requisite anti-concentration. Though the authors had believed localization was the key mechanism to obtain low-degree approximations for Ising models, it was not clear how to concretely connect this machinery to low-degree polynomials, so we instead developed the sampler-inverter machinery to bypass this difficulty.

Following the release of GPT-6 Astra in September 2026, we prompted the model to obtain a learning algorithm for $\mathsf{AC}^0$ under the Dobrushin condition, since this was not captured in our prior framework. To our surprise, it produced a (somewhat confusing) proof of a learning algorithm that broadly followed the Kanade-Mossel framework. While many steps were familiar to us, the model used Eldan's stochastic localization~\cite{eldan2013thin,el2022information} as an essential mechanism to link high-degree behavior with learning. According to the documentation, the model closely followed the work of Klartag and Putterman~\cite{AFST_2023_6_32_5_939_0} to establish suitable functional inequalities for high-degree functions that can imply learning. We view this connection as the central AI-generated insight that underlies the development of the new theory compared to our prior approach.

After digesting this initial writeup, the authors realized that the somewhat mysterious use of stochastic localization was neither essential or natural, and in fact quantitatively much weaker than necessary. In further interactions with Claude Opus 5.5 and GPT-6 Astra prompted to use different localization schemes, we eventually obtained what we believe to be the more natural approach to this low-degree theory developed in this paper. In certain cases (e.g. \Cref{thm:eigen-ising}), we supplied high-level approaches to obtain cleaner results than originally obtained by the models; in many others, like the improved high-temperature results of \Cref{sec:ht-ising} and in \Cref{thm:products_to_ising}, the models produced much conceptually cleaner arguments than we had anticipated. Several of the intermediate results produced in this process were quite ad hoc and unmotivated, and the final results appearing in this paper are the cumulative effort of the authors using significant interaction with Claude Opus 5.5 and GPT-6 Astra to simplify and develop cleaner approaches. The resulting learning framework is essentially the ideal theory that we had hoped for when initiating this project.

All text appearing in this paper was written by the human authors. As with any other paper, the authors take responsibility for the content of this manuscript.

\noindent\textbf{Acknowledgments.} JG thanks Elchanan Mossel for very helpful discussions on influence bounds for monotone functions at high-temperature.

Much of this work was done when AV was at UT Austin, supported by the NSF AI Institute for Foundations of Machine Learning (IFML). Much of this work was completed while JG was at the MIT Department of Mathematics, supported by Elchanan Mossel's Vannevar Bush Faculty Fellowship ONR-N00014-20-1-2826 and Simons Investigator Award 622132. AM is supported in part by a Microsoft Trustworthy AI Grant, NSF-CCF 2430381, an ONR grant, and a David
and Lucile Packard Fellowship. Part of this work was done when GC was visiting the Simon's Institute for Theory of Computing, supported in part by the NSF AI Institute for Foundations of Machine Learning (IFML) and a Simons Investigator Award (\#409864, David Zuckerman).

\section{Technical Overview}
In this section, we explain the key conceptual ideas underlying our approach to low-degree learnability under high-temperature spin systems. As we explain shortly, the analysis is heavily based on \emph{localization schemes}, as studied by Chen and Eldan~\cite{DBLP:conf/focs/ChenE22}, to analyze the evolution of functions when progressively revealing/losing information about the sample $X\sim \mu$. We explain in \Cref{sec:what-noise} why stability defined through the coordinate-by-coordinate revealment process is, in some sense, the correct generalization of Fourier analysis on products for establishing low-degree approximation. In \Cref{sec:exp-decay-overview}, we explain the key probabilistic quantities that arise in analyzing this stochastic process and why they remain provably mild under well-studied high-temperature conditions. Finally, we explain in \Cref{sec:approx-from-decay} how stability leads to low-degree approximation and learning through \emph{transference principles} for Ising models that are known to rapidly mix.

\noindent\textbf{Notation.}
For a distribution $\mu$ and functions $f$, we define the distributional norms $\norm{f}_{p,\mu}\coloneq (\E_{\mu}[|f|^p])^{1/p}$. We also define the distributional inner product $\langle f,g\rangle_{\mu}\coloneq \E_{\mu}[fg]$.
We use the informal notation $\gtrsim$ to denote inequality up to an absolute constant. We define the noise stability of a function $f:\{\pm 1\}^n\to \mathbb{R}$ with respect to a erasure rate $1-\rho$:
\begin{equation*}
    \Stab_{\rho,\mu}(f)\coloneq \langle f,T_{\rho}f\rangle_{\mu}.
\end{equation*} 

Every function on the cube is a unique multilinear polynomial, and
$\Vr=\mathrm{span}\{\chi_S:|S|<r\}$
is the space of polynomials of degree less than $r$. We write $\Pi_{<r}$ for the $L^2(\mu)$-orthogonal projection onto $\Vr$ and $\dist_\mu(f,\Vr)=\|f-\Pi_{<r}f\|_{2,\mu}$. We use $U_n$ to denote the uniform distribution on $n$ bits (we use $U$ when the number of bits is obvious from context). Given matrix $A\in \R^{n\times n}$ and vector $h\in \R^{n}$, we use $\mu_{A,h}$ to denote the Ising model with interaction matrix $A$ and external field $\bm{h}$. We also use the notation $\Var(\cdot)$ and $\Cov(\cdot,\cdot)$ to denote variance and covariance of random variables.

\subsection{When can high-degree functions hear noise?}
\label{sec:what-noise}
The fundamental principle underlying all existing analyses of the low-degree algorithm is the following intuition: if ``high-degree'' functions are \emph{sensitive} to the action of a noise operator that perturbs the input, while every function $f\in \mathcal{C}$ remains \emph{stable} under noise, then $\mathcal{C}$ should be approximable by low-degree polynomials. For the uniform distribution $U$, the classical notion of noise is the Markov chain that just erases and resamples a random coordinate. The associated Laplacian $\mathcal{L}$ operator is then defined by:
\begin{equation}
\label{eq:laplacian}
    \mathcal{L}:=\sum_{i=1}^n (I-\mathsf{P}_i),
\end{equation}
where $\mathsf{P}_i$ is the operator that independently resamples $X_i$. In this case, the influence (also called the \emph{Dirichlet form}) of a function $f$ is defined via 
\begin{equation}
\label{eq:laplacian-variance}
    \mathcal{E}_{U}(f,f):=\mathbb{E}_U[f(X)\mathcal{L}f(X)]=\sum_{i=1}^n \mathbb{E}_U[\Var(f(X)\vert X_{-i})].
\end{equation}
For an arbitrary distribution $\mu$, the immediate generalization is the Glauber dynamics: the operator $\mathsf{P}_i$ now re-randomizes $X_i$ \emph{conditioned} on $X_{-i}$ in $\mu$, and the Laplacian $\mathcal{L}$ directly generalizes \Cref{eq:laplacian}.

As a representative example of this principle, we can make the following definition:
\begin{definition}
\label{def:hop}
    A distribution $\mu$ on $\{\pm 1\}^n$ satisfies a \textbf{degree-$r$ Poincar\'e inequality} if for all $r\in \mathbb{N}$ and any function $g:\pm \{\pm 1\}^n\to \mathbb{R}$ that satisfies $g\perp \mathcal{P}_r$, the space of all polynomials of degree at most $r-1$,
    \begin{equation}
    \label{eq:hop}
        \mathcal{E}_{\mu}(g,g)\gtrsim r\cdot \Var_{\mu}(g).
    \end{equation}
\end{definition}
Observe that when $r=1$, this definition is just the standard Poincar\'e inequality, which is well-known to be equivalent to a spectral gap for the Glauber dynamics~\cite{levin2017markov}. As such, the $r=1$ case is a central object of study in the sampling literature since it governs rapid mixing of the Glauber dynamics.
For $r>1$, \Cref{def:hop} asserts that high-degree functions have increasingly local variation about random points on $\{\pm 1\}^n$ that exceeds the global variation by a factor that scales linearly with degree. For product distributions, it is well-known that \Cref{eq:hop} holds with an \emph{optimal} constant $1$ for all $r\in\mathbb{N}$, and is already sufficient for obtaining interesting learning results for low-influence function classes via the low-degree algorithm~\cite{o2021analysis}. The key fact underlying these applications is that the low-degree polynomials are \emph{exactly} the low-energy eigenbasis of the Laplacian $\mathcal{L}$ .

The main focus of our work will aim to establish \Cref{eq:hop} as a consequence of more powerful bounds on stability with respect to noise that we later will show implies low-degree approximability. But for general non-product distributions, it is by no means obvious how to establish analogous higher-degree Poincar\'e inequalities, if they are even true. As discussed earlier, the eigenstructure for general distributions remains poorly understood, which has been a long-standing barrier towards developing a clean Fourier-like theory. An enormous body of work in the sampling literature has been devoted to establishing the standard Poincar\'e inequality for distributions of interest, but these existing bounds do not directly bear on \Cref{eq:hop}: since they apply to \emph{all} functions, they do not leverage high-degree structure in any \emph{a priori} way. In fact, there is a structural obstacle: the re-randomization operator no longer kills variables since it takes a complex conditional expectation involving the other variables, so the Laplacian $\mathcal{L}$ for non-product distributions will typically \emph{increase} the polynomial degree of a test function $g$!

Our challenge is thus that we require a different analysis that can more directly leverage the fact that $g\perp \mathcal{P}_{r}$ with respect to $\mu$ to obtain improved bounds. But rather than exploit the eigenstructure underlying Fourier analysis, we need a more operational definition of $g\perp \mathcal{P}_{r}$ that can be analyzed with respect to a noise process. The key characterization that will be essential to our approach is the following simple observation:

\begin{lemma}[\Cref{lem:orthogonal-ld}, restated]
\label{lem:ortho-ld-intro}
    For any function $g:\{\pm 1\}^n\to \mathbb{R}$, $g\perp \mathcal{P}_{r}$ with respect to $\mu$ is equivalent to $\mathbb{E}[g(X)\vert X_T]\equiv 0$ whenever $\vert T\vert<r$.
\end{lemma}

To exploit this observation, we deviate from interpreting noise via the Glauber dynamics: rather, we will instead define it via a time-changed version of the \emph{coordinate-by-coordinate localization scheme} of Chen and Eldan~\cite{DBLP:conf/focs/ChenE22}. In general, a localization scheme can be viewed as in a statistical or Bayesian way as progressively revealing more information about a sample $X\sim \mu$ starting with no information~\cite{el2022information,montanari2023sampling} and ending with full information. Chen and Eldan show that if one can control the information flow of the posteriors along this evolution, then one can obtain bounds like a Poincar\'e inequality or a modified log-Sobolev inequality to obtain rapid mixing. In our context, their coordinate-by-coordinate revealment process simply reveals a random coordinate $X_{i_t}$ at each time $t=1,\ldots,n$ of a random sample $X\sim \mu$.

In our case, we will study the behavior of a fixed function $g\perp \mathcal{P}_{r}$ by going in reverse: for each $i\in [n]$, define the \emph{masking time} of $i\in [n]$ to be $\sigma_i\sim \mathrm{Exp}(1)$, an independent exponential random variable.\footnote{The choice of exponential distributions is not essential, but it significantly clarifies the analysis since these distributions are memoryless. This enables us to freely condition on the realization at different times without needing to change the masking distribution on the remaining variables.} With $X\sim \mu$, we will consider the stochastic process:
\begin{equation*}
    s\mapsto \mathbb{E}[g(X)\vert X_{V_s}],
\end{equation*}
where $V_s=\{i\in [n]:\sigma_i>s\}$ is the set of \emph{visible variables}. Note that both $X\sim \mu$ and the visible variables are random variables. At $s=0$, clearly $V_0=[n]$ and we observe $g(X)$. For $s>\max_i \sigma_i$, $V_s=\emptyset$ and we have no information about the value of $g(X)$. For convenience, we will define:
\begin{equation*}
    g_{S}(X):= \mathbb{E}[g(X)\vert X_{S}];
\end{equation*}
which can be seen both as a function depending on $X_S$ as well as a random variable when sampling $X\sim \mu$.
 
With this random process, we can make the following Fourier-free definition of stability: 
\begin{equation*}
    \Lambda(s):=\mathbb{E}_{V_s,X}[\vert g_{V_s}(X)\vert^2];
\end{equation*}
this quantity measures the expected error of the best predictor of $g(X)$ in $L^2$ given the visible variables at time $s$. More precisely, if $g$ is centered, then $\Lambda(s)$ monotonically decreases from $\Var(g)$ to $0$ by Jensen's inequality, and at each time $s\geq 0$, quantifies the $L^2$ prediction signal about the value of $g(X)$ remains by conditioning only on the visible set $V_s$.

There are several essential perspectives on this process that will guide our analysis: first, for a product distribution $\nu$, one can easily show that
\begin{equation}
\label{eq:stab-product}
    \Lambda(s)=\mathrm{Stab}_{e^{-s},\nu}(g)=\sum_{S\subseteq [n]}e^{-s\vert S\vert}(\widehat{g_\nu}(S))^2,
\end{equation}
where $\widehat{g_\nu}(S)$ are the biased Fourier coefficients with respect to $\nu$.
The reason is that for products, the rerandomization operator $\mathsf{P}_i$ is \emph{equivalent} to killing terms in the Fourier expansion with subsets including $i\in [n]$. If $g\perp \mathcal{P}_{r}$, then only subsets of cardinality at least $r$ are included and thus $\Lambda(s)=O(\exp(-rs))$: the exponential decay rate is at least the degree it is orthogonal to.

In fact, this degree-revelation remains true for general distributions precisely because $\Lambda(s)$ was defined intrinsically in a Fourier-free way through conditional expectations. Exponential decay rate is equivalent to studying the derivatives of $\log(\Lambda(s))$, and a straightfoward computation shows that:
\begin{proposition}[\Cref{cor:log-reveals-degree}, restated]
\label{prop:log-reveals-degree-intro}
    Let $\mu$ be any distribution on $\{\pm 1\}^n$ and suppose that $g\perp \mathcal{P}_{r}$ with respect to $\mu$. Then
    \begin{gather}
        \frac{-\Lambda'(0)}{\Lambda(0)}=\frac{\mathcal{E}_{\mu}(g,g)}{\Var(g)}\quad\quad
        \lim_{s\to \infty}\frac{-\Lambda'(s)}{\Lambda(s)}\geq r.
    \end{gather}
\end{proposition}

To understand why the first identity holds, consider the visible set $V_h$ for small $h>0$. Because the masking times are independent $\mathrm{Exp}(1)$ random variables, the probability each variable gets masked in $[0,h]$ is roughly $h$, while the chance that there are two or more masked coordinates is $O(h^2)$. In particular, 
\begin{gather*}
    \Lambda(h)=\mathbb{E}_{X,V_h}[\vert \mathbb{E}[g(X)\vert X_{V_h}]\vert^2]\approx \mathbb{E}[g(X)^2]+h\left(\sum_{i=1}^n \left(\mathbb{E}_X[\vert \mathbb{E}[g(X)\vert X_{-i}]\vert^2]-\mathbb{E}[g(X)^2]\right)\right)\\
    \implies \frac{\Lambda(h)-\Lambda(0)}{h}=\left(\sum_{i=1}^n \left(\mathbb{E}_X[\vert \mathbb{E}[g(X)\vert X_{-i}]\vert^2]-\mathbb{E}[g(X)^2]\right)\right)+O(h)
\end{gather*}
By the law of total variance and taking $h\to 0$, we exactly get back the expression in \Cref{eq:laplacian-variance}. For the degree-revealing limiting behavior, the argument requires no Fourier analysis: at large times, the largest nonzero terms in the expansion of $\Lambda(s)$ over set probabilities come from the smallest subsets $S$ such that $g_S\not\equiv 0$. Any such subset necessarily has $\vert S\vert\geq r$ by \Cref{lem:ortho-ld-intro}, so by definition of the exponential distribution, the dominant terms are order $\exp(-r_0s)$ for some $r_0\geq r$. We carry out the formal derivation in \Cref{prop:lambda-derivatives} to make this intuition precise, but the key takeaway of this localization approach is that:
\begin{quote}
\centering
    \emph{The stability  $\Lambda(s)$ defined by the coordinate-by-coordinate localization scheme naturally interpolates between higher-degree Poincar\'e inequalities and the degree parameter for any distribution $\mu$.}
\end{quote}

Our task is therefore to determine whether it is possible to propagate this large time degree-revealing behavior of $\varphi:=-\Lambda'/\Lambda$ to all times $s$. 

\subsection{Exponential decay of stability at high-temperature}
\label{sec:exp-decay-overview}
In our setting of high-temperature graphical models, there is no hope of showing perfect \emph{monotonicity} of $\varphi$ for reasons that we describe shortly. But if we can instead just show that $\varphi(s)\gtrsim r$ for all $s\geq 0$, then we will still obtain essentially the same exponential stability decay depending on degree as products that implies \Cref{eq:hop} among other later applications. To obtain such a lower bound, our approach will simply be to take logarithms once more and differentiate, because then
\begin{equation*}
    \int_s^{\infty} \frac{\mathrm{d}}{\mathrm{d}u}\log(\varphi(u))\mathrm{d}u<C\implies \varphi(s)\geq r/C,
\end{equation*}
by integrating the derivative and using the degree-revealing behavior from \Cref{prop:log-reveals-degree-intro}. 

In the product setting with \Cref{eq:stab-product}, one can show directly that $\Lambda(s)$ is already log-convex, meaning $\varphi(s)$ non-increasing and thus degree-revelation translates the entire way.\footnote{One way to see this is that \Cref{eq:stab-product} gives a weighted average over the Fourier weights that progressively puts less mass on the higher-order terms as $s$ increases.} In our case, we will need to take an additional intermediate logarithm to get quasimonotonicity of $\varphi$, which can be viewed as a second-order analogue of the Chen-Eldan analysis. Their goal was to establish ``conservation of variance'' to show that the Dirichlet form retains much of the global variance by showing that the derivative of the log-variance is never too large. In our setting, we will need to use similar structural properties of the measure to show that the derivatives of the \emph{Rayleigh quotients} are not too large along the revealment process. The high-level approach of interpolating Rayleigh quotients was recently studied by Klartag and Putterman~\cite{AFST_2023_6_32_5_939_0}, where they showed that Rayleigh quotients are monotonic along heat flow for any log-concave distribution. In fact, the general framework of localization schemes can also prove low-degree approximation for strongly log-concave distributions by instead using Eldan's stochastic localization~\cite{eldan2013thin}.\footnote{We had originally intended to present this low-degree application of their analysis within the localization framework more formally. But while preparing this paper, we became aware of a recent talk abstract by Pierre Bizeul~\cite{bizeul} on unpublished joint work with Bo'az Klartag that announces similar polynomial approximation results for strongly log-concave distributions using the heat flow. We have thus chosen not to formally present these results and instead refer the reader to their results when released. But at a high level, instead of viewing the erasure channel $Y_s:=X_{V_s}$ as the localization scheme, one can instead consider the Gaussian signals $Y_s:=X+B_s$ for a Brownian motion $B_s$. It turns out that the long time behavior of $\mathbb{E}[g(X)\vert Y_s]$ also reveals degree and can be directly analyzed using results of Klartag and Putterman~\cite{AFST_2023_6_32_5_939_0}. To our knowledge, these results do not directly improve existing learning guarantees~\cite{kane2013learning}; nonetheless, we view these connections as strong evidence that localization schemes may be a unifying framework for analyzing low-degree approximation relative to various high-dimensional distributions.} 

As it turns out, the derivatives of $\log(\varphi(s))$ are remarkably clean and are naturally expressed in terms of covariances encountered along the random process:

\begin{proposition}[Consequence of \Cref{cor:log-decay}]
\label{prop:log-decay-intro}
    Let $r\geq 1$ and let $g:\{\pm 1\}^n\to \mathbb{R}$ be nonzero and orthogonal to $\mathcal{P}_{r}$. Then for all $s>0$, 
    \begin{equation*}
        \frac{\mathrm{d}}{\mathrm{d}s}\log \varphi(s)\leq \frac{K(s)}{(e^{s}-1)G(s)},
    \end{equation*}
    where
    \begin{equation*}
    G(s)
    :=\mathbb{E}_{V_s}
    \sum_{i\notin V_s}
    \mathbb{E}_{\mu}\left[
    \operatorname{Var}\left(
    g_{V_s\cup\{i\}}\mid X_{V_s}
    \right)\right],\quad\quad
    K(s):=\mathbb{E}_{V_s}
    \sum_{\substack{i\neq j\\i,j\notin V_s}}
    \mathbb{E}_{\mu}\left[
    \operatorname{Cov}\left(
    g_{V_s\cup\{i\}},
    g_{V_s\cup\{j\}}
    \mid X_{V_s}
    \right)\right].
    \end{equation*}
    In particular, for any $s\geq 0$,
    \begin{equation}
    \label{eq:varphi-lb-intro}
    \varphi(s)\geq r\cdot\exp\left(
    -\int_s^\infty
    \frac{K(u)}{(e^u-1)G(u)}\,\mathrm{d}u
    \right).
    \end{equation}
\end{proposition}

To understand this differential inequality, the non-monotonicity of $\log(\varphi)$ that will prevent us from propagating the degree information at infinity can be attributed to covariance terms that are quite similar to those studied in the \emph{spectral independence} framework~\cite{DBLP:journals/siamcomp/AnariLG24}. Another way to to interpret them is as second-order \emph{curvature} terms in $K(s)$ that capture the correlation in the information revealed about $g(X)$ by masking coordinates $X_i$ and $X_j$ starting from a current visible set $V_s$; in products, this correlation is identically zero. The exact derivative of $\log(\varphi(s))$ can be thought of as a kind of ``Bochner identity'' for the generator of the masking process, as is commonly studied in the diffusion literature to measure the extent to which ``direction $i$ and $j$ fail to commute" (see e.g. Chewi~\cite{Chewi26Book} for a textbook treatment).

To this point, we have made no structural assumptions about the distribution $\mu$, so we cannot expect this differential inequality to remain well-behaved without imposing assumptions---in the context of log-concave diffusions, this is often handled by Bakry-\'Emery assumptions that do not apply here. But as we now describe, the ratio $K(s)/((e^s-1)G(s))$ can be cleanly controlled using other \emph{well-studied high-temperature properties of the distribution closely related to spectral independence}~\cite{DBLP:journals/siamcomp/AnariLG24}. We will do this in two regimes:
\begin{enumerate}
    \item[(1)] \textbf{Long times $s\gg 0$}: First, note that the $e^s-1$ term in the denominator of \Cref{eq:varphi-lb-intro} is very favorable away from $0$ since this is already decaying rapidly. Therefore, we can use simple bounds on $K(s)/G(s)$ in this regime to control the integral away from $0$. In particular, notice that for any fixed visible set $V_s$, $K(s)$ can be viewed as the \emph{off-diagonal} terms when expanding the conditional variance of $$h(X_{[n]\setminus V_s}):=\sum_{i\not \in V_s} g_{V_s\cup \{i\}}(X_i)$$ But here the Poincar\'e inequality, which we assumed is true for all pinnings, is especially useful: if we apply it to this function $h$, then 
    \begin{align*}
        \Var(h(X_{[n]\setminus V_s})\vert X_{V_s})&\leq \frac{1}{\alpha}\sum_{i\not\in V_s}\mathbb{E}\left[\Var\left(\sum_{j\not\in V_s} g_{V_s\cup \{j\}}\bigg\vert X_{-i}\right)\bigg\vert X_{V_s}\right]\\
        &=\frac{1}{\alpha}\sum_{i\not\in V_s}\mathbb{E}\left[\Var\left( g_{V_s\cup \{i\}}\big\vert X_{-i}\right)\mid X_{V_s}\right].
    \end{align*}
    The reason for the second identity is that once we condition on all spins but $X_i$, the only term that does not get fixed in the variance is simply $g_{V_s\cup \{i\}}$. But this remaining sum is at most $G(s)$ by the law of total variance! If we manipulate this inequality slightly more carefully (\Cref{lem:long-times}), it follows from pinned Poincar\'e that for all $s\geq 0$:
    \begin{equation}
    \label{eq:long-time-overview}
        K(s)\leq (1/\alpha-1)\cdot G(s).
    \end{equation}

    Note that \Cref{eq:long-time-overview} is ineffective close to zero since $1/(e^s-1)\approx 1/s$ is not integrable there. However, there is a simple charging argument that ensures $\varphi(s)$ cannot decay too fast using the direct definition of the masking process (\Cref{lem:short-ub}). By trading off the slow decay of $\varphi$ far from zero via \Cref{eq:long-time-overview} with this direct argument at close times, we can already obtain the slightly weaker bounds of \Cref{Thm: spectral decay} just using the pinned Poincar\'e assumption.

    \item[(2)] \textbf{Short times near $s=0$}: To obtain a truly linear in $r$ lower bound for $\varphi(s)$ at all times, we need a stronger argument that leverages additional high-temperature properties near $s=0$ to give a \emph{constant} bound on $K(s)/(e^s-1)\cdot G(s)$. We show, through somewhat more involved arguments, that any of the following conditions suffices: (i) bounded interaction degree in $\mu$, and for Ising models specifically, either (ii) bounded $\ell_1$-width or (iii) suitable mean-field bounds on the interaction matrix. We defer discussion on this at this stage, but the high-level idea in all three arguments is that the off-diagonal terms in $K(s)$ appear when \emph{both} $i$ and $j$ do not belong to $V_s$, while a variance term is included even when only one of them is. When $s$ is small, their co-occurrence is unlikely and so we should give a more careful comparison of $K(s)$ and $G(s)$ that exploits this extra randomness. Formally, we do this by giving a deterministic upper bound conditioned on any $X_{V_s}$ that more directly charges the costs of both $i$ and $j$ appearing to the variance terms in $G(s)$, and it turns out we can cancel the singular factor $(e^s-1)$ using a change of measure argument.

    Putting this argument with the long-time behavior of \Cref{eq:long-time-overview}, we obtain the final conclusion that:
    \begin{equation*}
        \varphi(s)\gtrsim r
\quad\forall s\geq 0, \end{equation*}
which by our reductions above shows $\Lambda(s)\leq \exp(-csr)$. By inspecting the behavior near $s=0$, we can immediately obtain, among our other applications, the higher-degree Poincar\'e inequality of \Cref{eq:hop} for any $g\perp \mathcal{P}_{r}$ that matches what holds in the case of product distributions.
\end{enumerate}

In either case, after doing the time-change, we can recover \Cref{Thm: spectral decay} as well as consequences like higher-degree Poincar\'e inequalities for a wide range of models like those listed in \Cref{tab:families}. 

\begin{table}[h]
\centering
\footnotesize
\setlength{\tabcolsep}{3pt}
\renewcommand{\arraystretch}{1.5}
\begin{tabular}{@{}p{0.5\linewidth}cccc p{0.25\linewidth}@{}}
\toprule
Class & \hyperref[ass:loc-pi]{$(\mathrm{PI}_{\alpha})$} & \hyperref[ass:loc-bd]{$(\mathrm{BD}_{\Delta})$} & \hyperref[ass:loc-bw]{$(\mathrm{BW}_{B})$} & \hyperref[ass:loc-mf]{$(\mathrm{MF}_{\omega,L})$} & References\\
\midrule

\multicolumn{6}{@{}l}{
\textit{I. General pinned Poincar\'e}}\\

\textbf{Smooth Hamiltonians.}
$\mu\propto e^H$ with
$\sup_{x\in[-1,1]^n}\|\nabla^2 H(x)\|_{\mathrm{op}}
\le c_0$ for a small constant $c_0>0$;
& $\checkmark$ & & & &
\cite[Thms.~2 and~6]{DBLP:conf/soda/AnariJKPV24}\\

\textbf{Spectrally bounded Ising models.}
$\lambda_{\max}(A)-\lambda_{\min}(A)
\le 1-\delta$, with arbitrary external fields.
& $\checkmark^{(a)}$ & & & &
\cite{eldan2022spectral,DBLP:conf/focs/ChenE22}\\

\midrule
\multicolumn{6}{@{}l}{
\textit{II. Pinned Poincar\'e and bounded degree}}\\

\textbf{Two-spin systems up to tree uniqueness.}
Ising models with arbitrary fields and
antiferromagnetic two-spin systems in the tree-uniqueness
regime.
& $\checkmark$ & $\checkmark$ & $\checkmark^{(b)}$ & &
\cite{DBLP:conf/focs/ChenFYZ21,
DBLP:conf/stoc/ChenLV21}\\

\textbf{Spectrally independent graphical models.}
Bounded-degree systems with uniform spectral
independence and marginal bounds.
& $\checkmark$ & $\checkmark$ & & &
\cite[Thm.~1.12]{DBLP:conf/stoc/ChenLV21}\\

\midrule
\multicolumn{6}{@{}l}{
\textit{III. Pinned Poincar\'e and bounded $\ell_1$ width}}\\

\textbf{Dobrushin Ising models.}
$\max_i\sum_j\tanh|A_{ij}|\le1-\delta$,
with arbitrary external fields.
& $\checkmark$ & & $\checkmark$ & &
\cite{dobrushin}\\

\midrule
\multicolumn{6}{@{}l}{
\textit{IV. Pinned Poincar\'e and mean-field interactions}}\\

\textbf{Sherrington--Kirkpatrick model.}
$A=\kappa W/\sqrt n$, $W$ GOE with zero
diagonal, $\kappa<1/2$, and arbitrary fields.
& $\checkmark$ & & & $\checkmark$ &
\cite{DBLP:journals/corr/abs-2608-22159}\\

\bottomrule
\end{tabular}
\caption{
Here, we list general families of graphical models that are captured by \Cref{Thm: spectral decay}. Our results hold under essentially all general purpose frameworks for establishing rapid mixing for these models. In (a), it more generally suffices that $\|A\|_{\mathsf{op}}$ is bounded and $\lambda_{\max}(\mathrm{Cov}(\mu_{tA,\bm{h}}))\le\chi$ for every $t\in[0,1]$ and any field $\bm{h}$~\cite{DBLP:conf/stoc/AnariJKPV22,DBLP:conf/focs/ChenE22}. In (b), the width condition requires positive distributions, so does not apply for the hard-core model.}
\label{tab:families}
\end{table}

As an interesting consequence of these results, we actually obtain new structural understanding of the \emph{eigenspaces} of Glauber dynamics. For bounded degree graphical models and Ising models satisfying the above high-temperature conditions, we can use these results to establish (to our knowledge) new bounds on \emph{all} eigenvalues of Glauber dynamics: they have the same eigendecay as product distributions at \emph{all} levels. This is a vast generalization of the Poincar\'e inequality, which simply asserts the that the second eigenvalue is similar to product. But given our results, the proof is a basic consequence of the Courant-Weyl minimax inequality using $\mathcal{P}_r$ and $\mathcal{P}_r^{\perp}$ as test subspaces to bound the eigenvalues. We defer to \Cref{sec:eigenspaces} for more information.

\subsection{Polynomial approximations from decay, and transferring noise sensitivity bounds}
\label{sec:approx-from-decay}
The discussion in the previous subsection gave an overview of the techniques needed to establish $(c,\beta)$-decay for a given distribution. In this section, we sketch how such decay leads to polynomial approximability, and hence learnability, for various well-studied function classes. 

Henceforth, we assume that our distribution $\mu$ exhibits $(c,\beta)$-decay. The first step towards constructing polynomial approximators is to establish the implication that functions satisfying a noise sensitivity bound over $\mu$ have good low-degree polynomial approximations. This was stated earlier in \Cref{Thm: NS implies approximation}. 
\begin{lemma}[Informal, see \Cref{thm:noise_stability_implies_approximation}]
    Let $\mu$ be a distribution satisfying $(c,\beta)$ decay and let $f$ be a Boolean function satisfying $\NS_{\delta,\mu}(f)\leq \varepsilon/16$ for some $\delta>0$. Then, $f$ has polynomial approximators of degree $\widetilde{O}((1/c\delta)^{1/\beta})$ that have $\ell_2$ error $\varepsilon$.
\end{lemma}
    The proof has two steps. First, we use the noise sensitivity bound to infer that $\mutwonorm{T_{\rho,\mu}(f)-f}{\mu}\leq O(\varepsilon)$ for $\rho=1-\delta$. This follows almost immediately from expanding out the expression of the norm and using the definition of noise sensitivity. Then all that remains is to argue that $T_{\rho,\mu}(f)$ has good low-degree polynomial approximators. Let $g$ be the projection of $T_{\rho,\mu}f$ that is orthogonal to the space of low-degree polynomials. We will argue that \[ \mutwonorm{g}{\mu}^2\leq \sqrt{\Stab_{\rho,\mu}(g)\Stab_{\rho,\mu}(f)}\leq \rho^{cr^\beta/2}\mutwonorm{g}{\mu}.\] 
This gives degree bounds for $\ell_2$ approximation but techniques from \cite{feldman2017tight} can be adapted to get sharper bounds for $\ell_1$ approximation (see \Cref{thm:noise_stability_implies_approximation L1}).

 All that remains is to bound the noise sensitivity of various function classes of interest. Many excellent noise sensitivity bounds are already known over the uniform distribution \cite{DBLP:journals/jacm/BshoutyT96,peres2004noisestabilityweightedmajority,klivans2004learning,kane_average_2014,openai2026gotsmanlinial}. Our goal is to show that these bounds can be lifted to Ising distributions and we proceed by invoking a powerful measure decomposition due to Hubbard and Stratonovich \cite{hubbard}:

\begin{proposition}[Hubbard-Stratanovich~\cite{hubbard}, see e.g.  Theorem 3.12 of~\cite{DBLP:conf/focs/LiuMRW24}]
\label{prop:hs}
For any Ising model $\mu_{A,h}$ with PSD
interaction matrix $A$ and external field $h$,
sample $X\sim\mu_{A,\bm{h}}$ and
$G\sim\mathcal N(0,I)$ independently, and let
\[
Z=AX+A^{1/2}G+\bm{h}.
\]
Then, the conditional distribution of $X$
given $Z=\bm{z}$ is $\mu_{0,\bm{z}}$. Thus, $\mu$ can be written as a mixture of product distributions
\[
\mu_{A,\bm{h}}=\E_Z[\mu_{0,Z}].
\]
\end{proposition}

The above decomposition allows us to express an Ising distribution as a mixture of product distributions. This decomposition has proven extremely in the context of sampling, beginning with the pioneering work of Bauerschmidt and Bodineau~\cite{bauerschmidt,DBLP:conf/focs/ChenE22,DBLP:conf/focs/LiuMRW24}. This is a promising first step as the work of Blais, O'Donnell and Wimmer \cite{blais_polynomial_2010} shows that  bounds on the noise sensitivity under the uniform distribution can be transferred to all product distributions. We complete the story with the following theorem, from which bounds on Ising noise sensitivity follow (see \Cref{clry:uniform_to_products}):
\begin{theorem}[Noise-stability Transfer: Products to Ising, see \Cref{thm:products_to_ising}]
\label{thm:products_to_ising_overview}
    Let $\lambda>0$ and let $\mu=\mu_{A,h}$ be an Ising distribution with $0\preceq A\preceq \lambda$. Let $Z$ be the random variable defined in \Cref{prop:hs}. There exists a constant $C_\lambda$ such that for any function $f:\cube{n}\to\R$ and $\rho\in [0,1]$, we have
    \begin{equation}
        \label{eqn:noise_stability_transfer_ising_overview}
        \Stab_{\rho,\mu}(f)\leq \E_{Z}[\Stab_{\rho,\mu_{0,Z}}(f)]\leq \Stab_{1-C_{\lambda}(1-\rho),\mu}(f).
    \end{equation} 
\end{theorem}
\begin{proof}[Proof Sketch]
    A key idea in the proof is that revealing more information about $X$ can only increase the expected noise stability. That is, for any $W$ that is a (possibly randomized) function of $X\sim \mu$, we have that 
    \begin{equation}  
    \label{eqn:stabilty_increases_with_information_overview}
    \Stab_{\rho,\mu}(f)\leq \E_{W}[\Stab_{\rho,\mu(.|W)}(f)]
    \end{equation}
    where $\mu(.|W)$ denotes the conditional distribution of $X$ given $W$. This follows by an application of Jensen's inequality to the defintion of noise stability.
    Taking $W=Z$ (from \Cref{prop:hs}) proves the first inequality. 

  The proof of the second inequality also proceeds using the same principle and one additional insight. The insight is that the random variable $Z$ (from \Cref{prop:hs}) can be generated by independently revealing each coordinate of $X$ with probability $1-C_{\lambda}$, and then apply a randomized function to the revealed coordinates. Thus, the noise process used in Hubbard-Stratonovich can be simulated by just using the masking noise process we have been considering so far at an appropriately chosen noise scale. Then, \Cref{eqn:stabilty_increases_with_information_overview} would imply that revealing these coordinates increases the noise stability at least as much as revealing $Z$, which corresponds to the right hand side of the inequality \Cref{eqn:noise_stability_transfer_ising_overview}.

    For $Y=\sqrt{\lambda}X+H$ with $X\sim \mu$ and $H\sim N(0,I)$, we have that 
    $Z=\bm{h}+\frac{A}{\sqrt{\lambda}}Y+G'$ where $G'\sim N(0,A-A^2/\lambda)$. All that remains is to argue that $Y$ can be be generated by independently revealing each coordinate of $X$ with probability $1-C_{\lambda}$ for $C_\lambda$ that we will choose.  

    Observe that conditioned on $X_i=x_i$, the density of $Y_i$ follows $N(x_i\sqrt{\lambda},1)$. Let $C_\lambda$ be the common mass between $N(\sqrt{\lambda},1)$ and $N(-\sqrt{\lambda},1)$. Thus, $Y_i$ can be sampled as a mixture of two components: one independent of $X_i$ with mass $C_\lambda$ and the remaing mass dependent on $X_i$.

 Thus, for $S=R_{1-C_{\lambda}}$ (where $R_{\eta}$ is the set formed by including each index with probability $\eta$), we have that $Y$ (and hence $Z$) can be generated (given $X_{S})$ by sampling each $Y_i$ from the common component if $i\notin S$ and from the component depending on $X_i$ otherwise.  Since $Z$ is a randomized function of $Y$, this implies that the distributions $\mu(.|Z,S,X_s)$ is identical to $\mu(.|S,X_S)$. We now apply \Cref{eqn:stabilty_increases_with_information_overview} to $\mu(.|Z)$ by also revealing $(S,X_S)$. 
\end{proof}

Some special remarks are in order for the class $\ac$ of polynomial sized, constant depth circuits. We want to obtain approximating polynomials whose degree depends polylogarithmically on $1/\varepsilon$. But noise sensitivity bounds for $\ac$ do not actually yield such strong bounds directly. Rather, we will find it useful to consider the noise stability at rates greater than one (equivalently noise-sensitivity at negative noise rates). This may seem like a strange quantity to consider, but then again, bounds on it imply exponential fourier tails over the uniform distribution. Now we can apply our transfer theorems to this quantity instead to get the desired bounds on the polynomial approximation error for $\ac$. We discuss this in detail in \Cref{Sec: polynomial approximation for AC0}.

\section{Further Related Work}
\label{sec: related work}

\noindent\textbf{PAC-Learning Beyond the Uniform Distribution.} Distribution-free learning is widely believed to be intractable for many function classes. For instance, the quasipolynomial-time algorithm for learning $\mathsf{AC}^0$ under uniform inputs~\cite{lmn} is believed to be tight due to surprising cryptographic lower bounds of Kharitonov~\cite{DBLP:journals/jcss/Kharitonov95}. Even for learning DNFs of polynomial size \--- i.e. polynomial size circuits of depth two \--- the best known algorithm of Klivans and Servedio \cite{DBLP:journals/jcss/KlivansS04} runs in time $2^{O(n^{1/3})}$. These bounds have not been improved upon in over two decades. Moreover, there are strong lower bounds against all correlational statistical query algorithms, including $2^{\Omega(n^{1/3})}$ lower bounds for learning DNFs and $2^{\Omega(n^{1-\delta})}$ lower bounds for learning $\mathsf{AC}^0$, again, under general distributions ~\cite{DBLP:journals/jcss/Feldman12,DBLP:journals/siamcomp/Sherstov11,klivans_ac0}. This suggests that some distributional assumption is necessary, and our work shows that the low-degree algorithm broadly extends to a large class of natural distributions without modification.

As we discussed, the only work that attempts to learn at the level of generality we consider is that of Kanade and Mossel~\cite{DBLP:conf/colt/KanadeM15}. In recent work, Chandrasekaran and Klivans~\cite{chandrasekaran2026learning} have shown that logarithmic size juntas can be efficiently PAC-learned given samples from a general Markov random field under smoothed external fields, broadening a result of Kalai, Samarodnitsky, and Teng~\cite{DBLP:conf/focs/KalaiST09} that also succeeds for decision trees and DNFs for smoothed product distributions.

The recent work of Heidari and Khardon~\cite{DBLP:conf/colt/HeidariK25} develops an analogue of the standard Fourier expansion for any distribution $\mathcal{D}$ in terms of its representation as a Bayesian network. They then show that given access to both this representation and query access to an underlying function, one can implement the well-known Kushilevitz-Mansour algorithm~\cite{DBLP:journals/siamcomp/KushilevitzM93} for simple DNF formulas to estimate large Fourier coefficients under restrictive necessary conditions on conditional probabilities in the model. In contrast, our algorithm requires only sample access, learns more general function classes, and succeeds under a vast array of high-temperature models that have been considered broadly in the sampling community.

A much larger line of work has developed methods to computationally extend from the uniform distribution to more general product distributions or mixtures.  Furst, Jackson, and Smith~\cite{furst1991improved} generalized LMN to products using $p$-biased Fourier analysis (they showed that distributional inverters imply learning, this was studied systematically \cite{binnendyk2022learning}). In a preliminary version of our paper, we extended this approach to bounded-degree Ising distributions that satisfy \emph{strong spatial mixing} by constructing new samplers that admit low-degree inverters. While this approach required assuming that the underlying dependency graph had polynomial neighbourhood growth, a subsequent work \cite{feng2026learning} removed this growth assumption for certain high-temperature models in incomparable temperature regimes by carefully simulating the systematic scan with few fewer dependencies. These analyses do not seem able to handle graphical models beyond bounded degree and additionally require ad hoc constructions of sampling algorithms to do the transference. 

Blais, O'Donnell, and Wimmer~\cite{DBLP:journals/ml/BlaisOW10} provide a beautiful general reduction from product distributions to the uniform case by showing how low-degree concentration for the Low-Degree Algorithm can extend to an arbitrary product. They further provide applications to learning from small mixtures of product distributions.  A recent line of work on lifting theorems for PAC-learning generalizes these results and shows how to use uniform distribution learners to learn over mixtures of sub-cubes \cite{blanc2023lifting,pmlr-v291-blanc25a}. Kou and Meka \cite{kou_smoothed_2025} have studied learning halfspaces over discrete distributions with strictly subexponential tails in a smoothed-analysis framework.

Another line of work studies learning and polynomial approximations over real valued distributions. Kane, Klivans and Meka \cite{kane2013learning} gave approximations for halfspaces, and functions of constant number of halfspaces over log-concave distributions. Chandrasekaran, Klivans, Kontonis, Meka and Stavropoulos \cite{chandrasekaran_smoothed_2024} gave polynomial approximations for the larger class of low-dimensional functions with bounded Gaussian surface area over distributions with strictly subexponential tails (and also in a smoothed analysis model). These results were extended to exponential tails by Koehler and Wu \cite{koehler2026constructive}. Recently, Klivans, Stavropoulos and Vasilyan \cite{sandwiching_ksv} gave sandwiching approximators for these functions classes over distributions with strictly subexponential tails.

\noindent\textbf{Boolean Function Analysis Beyond Products.} As we discussed above, developing an analogue of Fourier analysis is challenging due to the intractability of eigenstructure. As a result, we are not aware of many  In the context of high-temperature spin systems as considered in our work, Koehler, Lifschitz, Minzer, and Mossel~\cite{DBLP:conf/stoc/KoehlerLMM24} develop analogues of the seminal Kahn-Kalai-Linial Theorem~\cite{DBLP:conf/focs/KahnKL88} and Talagrand's $L_1-L_2$ inequality for high-temperature, bounded degree graphical models satisfying marginal boundedness. The bounded degree restriction has recently been removed by Chen, Chen, and Zhang~\cite{DBLP:journals/corr/abs-2609-24687} among several other results establishing product-like isoperimetric inequalities via a discrete notion of Bakry-\'Emery theory. For monotone measures that satisfy the FKG inequality, earlier seminal works in percolation theory of Grimmett and Graham~\cite{9650da55-eeca-3034-b7d1-95263bf137b0} and Duminil-Copin, Raoufi, and Tassion~\cite{0f1e1d61-7884-30c9-9e6b-aa3efdf4c5db} develop analogues of classical influence theorems originally stated for product distributions.

In the high-dimensional expansion (HDX) literature, there has been significant progress on developing an analogous theory of Boolean function analysis. In fact, the spectral independence approach to bounding mixing times pioneered by Anari, Liu, and Oveis Gharan~\cite{DBLP:journals/siamcomp/AnariLG24} was originally formulated in the HDX language since Glauber dynamics is the same as the ``down-up walk," and their analysis employs machinery developed for analyzing random walks on higher-order simplicial complexes~\cite{DBLP:journals/combinatorica/KaufmanO20,DBLP:conf/stoc/AlevL20}. Towards applications in PCPs and coding theory, there has been a parallel line of work developing an approximate Fourier basis using the levels of the complex~\cite{DBLP:journals/combinatorica/DiksteinDFH24,DBLP:conf/soda/BafnaHKL22} to establish properties like hypercontractivity and pseudorandomness~\cite{DBLP:conf/stoc/GurLL22,DBLP:conf/stoc/BafnaHKL22,DBLP:conf/approx/Hopkins26}. However, these results typically operate in different regimes from the settings we are interested in: the objects studied there are typically much sparser and also require much stronger spectral expansion to ensure errors do not compound in the analysis.

\noindent\textbf{Sampling and Localization.} There has been an enormous body of work on sampling from spin systems; a comprehensive overview is far beyond the scope of this work (see e.g.~\cite{martinelli,levin2017markov} for standard references). Establishing rapid mixing of the Glauber dynamics has been a central object of study, and of particular relevance to our work are the various recent breakthroughs obtained via the framework of \emph{spectral/entropic independence}~\cite{DBLP:journals/siamcomp/AnariLG24,DBLP:conf/stoc/AnariJKPV22,DBLP:journals/siamcomp/ChenLV23,DBLP:conf/focs/ChenE22}. Our analysis is heavily inspired by the Chen and Eldan framework of localization schemes, which show how spectral gaps and modified log-Sobolev inequalities can be obtained by tracking the evolution of various functionals along suitable revealment processes. In particular, their work gives a direct reduction from spectral independence~\cite{DBLP:journals/siamcomp/AnariLG24} to rapid mixing by tracking conservation of variance along the coordinate-by-coordinate revealment process. Our own analysis can be viewed as a refinement of the Chen-Eldan framework since we aim to establish stronger functional inequalities when $g\perp \mathcal{P}_{r}$ compared to a vanilla Poincar\'e inequality.

\section{Useful Definitions}
Recall that for a distribution $\mu$ and functions $f$, we define the distributional norms $\norm{f}_{p,\mu}\coloneq (\E_{\mu}[|f|^p])^{1/p}$. 
We recall the definition of noise operator over any distribution $\mu$.
\begin{definition}[Noise operator]
\label{defn:noise_operator}
Let $\mu$ be a distribution on $\{-1,1\}^n$ and
$0\leq \rho\leq 1$. Let $R_\rho\subseteq[n]$ be
obtained by independently including each coordinate
with probability $\rho$. Define
\[
T_{\rho,\mu}f(\bm{x})
=
\E_{R_\rho}
\left[
\E_{X\sim\mu}
[f(X)\mid X_{R_\rho}=\bm{x}_{R_\rho}]
\right].
\]
\end{definition}

We are now ready to define noise stability.
\begin{definition}[Noise stability]
\label{defn:noise_stability}
The noise stability of $f$ under $\mu$ at parameter
$\rho\in [0,1]$ is
\[
\operatorname{Stab}_{\rho,\mu}(f)
=
\langle f,T_{\rho,\mu}f\rangle_\mu
=
\E_{R_\rho}
\E_{X\sim \mu}\left[
\left(\E_{Y\sim \mu}[f(Y)\mid Y_{R_\rho}=X_{R_\rho}]
\right)^2
\right].
\]
\end{definition}

Finally, we define noise sensitivity.
\begin{definition}[Noise sensitivity]
\label{defn:noise_sensitivity}
Let $f:\cube{n}\to\{\pm 1\}$.
The noise sensitivity of $f$ under $\mu$ at
parameter $\delta\in [0,1]$ is
\[
\NS_{\delta,\mu}(f)
=
\frac{1-\Stab_{1-\delta,\mu}(f)}{2}.
\]
Equivalently,
\[
\NS_{\delta,\mu}(f)
=
\Pr[f(X)\neq f(X')],
\]
where $X\sim\mu$, $R=R_{1-\delta}$, and $X'$ is sampled
independently from $\mu$ conditioned on $X'_R=X_R$.
\end{definition}
Let $\varphi$ be the $1$-dimensional mean-zero Gaussian probability density function and let $\Phi$ be the corresponding cumulative distribution function. We represent a product distribution on $\cube{n}$ by a vector $\nu\in [-1,1]^{n}$ where $\nu_i$ denotes the mean of the $i$-th coordinate. 

\section{Stability Decay of Coordinate-by-Coordinate Localization}
\label{sec:stability-decay}

In this section, we will establish our key functional inequality: under the random masking process, all high-degree functions with respect to a high-temperature spin system $\mu$ exhibit $(c,\beta)$-decay (\Cref{defn:noise_decay}), where the parameters improve depending on the precise assumptions on the distribution. As we show later, this decay is sufficient to establish low-degree approximability for many natural function classes. 

It will be convenient to reparameterize time via $\rho=\exp(-s)$, so that as $s$ ranges from $0$ to $\infty$, $\rho$ decreases from $1$ to $0$. Under the original process where the random set $R_{\rho}$ is obtained by including each coordinate $i\in [n]$ independently with probability $\rho$, we can equivalently consider the evolution of the random set $V_s=R_{\rho}$ where each coordinate is now included with probability $\rho=e^{-s}$ at time $s$. An equivalent interpretation is that the masking time of each coordinate is an independent $\mathrm{Exp}(1)$ random variable; this representation makes clear that the masking process is \emph{memoryless}, which will make some of the following calculations cleaner.

For any fixed function $g:\{\pm 1\}^n\to \mathbb{R}$, we can now define $\Lambda:[0,\infty)\to [0,\infty)$ by
\begin{equation*}
    \Lambda(s):=\operatorname{Stab}_{e^{-s},\mu}(g)=\E_{V_s}\left[\left\|g_{V_s}\right\|_{2,\mu}^2\right],
\end{equation*}
where $g_{V}(\mathbf{x})=\mathbb{E}_{\mu}\left[g(X)\vert X_{V}=\mathbf{x}_V\right]$. When the argument $\mathbf{x}$ is omitted, we will mean the random variable evaluated at $\mathbf{x}\sim \mu$. Note the following elementary fact that is immediate from the tower law (also called the law of total expectation):

\begin{lemma}
\label{lem:tower}
    Let $V\supseteq V'$ be deterministic subsets. Then 
    \begin{gather*}
        \|g_{V}\|_{2,\mu}^2-\|g_{V'}\|_{2,\mu}^2=\E\left[\Var(g_V\vert X_{V'})\right]\\
        \mathbb{E}[g_Vg_{V'}] = \|g_{V'}\|_{2,\mu}^2.
    \end{gather*}
\end{lemma}
\begin{proof}
Note that as functions, $\mathbb{E}[g_V\vert X_{V'}]=g_{V'}$ by the tower law. Therefore, 
\begin{equation*}
    \|g_{V}\|_{2,\mu}^2-\|g_{V'}\|_{2,\mu}^2=\E\left[g_V^2-g_{V'}^2\right]=\mathbb{E}\left[g_V^2-\mathbb{E}[g_V\vert X_{V'}]^2\right]=\mathbb{E}[\Var(g_V\vert X_{V'})].
\end{equation*}
The proof of the second identity is also an immediate application of the tower law conditioning on $X_{V}'$.
\end{proof}

We will also use the following easy fact about orthogonality to low-degree polynomials:
\begin{lemma}
\label{lem:orthogonal-ld}
    Fix a function $g:\{\pm 1\}^n\to \mathbb{R}$. Then $g\perp \mathcal{P}_{r}$ is equivalent to $g_V\equiv 0$ whenever $\vert V\vert<r$.
\end{lemma}
\begin{proof}
    By definition, $g\perp \mathcal{P}_{r}$ means that for any polynomial $p$ of degree less than $r$, $\mathbb{E}_{\mu}[g(X)p(X)]=0$, while $g_V\equiv 0$ is equivalent to $\mathbb{E}_{\mu}[g(X)\phi(X_V)]=0$ for any function $\phi$. But since any polynomial over $\{\pm 1\}^n$ can be taken to be multilinear, it is clear that these conditions are equivalent.
\end{proof}

\subsection{Derivatives of Stability}

To show decay of stability, the high-level approach will be to carefully study its derivatives and then try to propagate information that $g\perp \mathcal{P}_{r}$. The following calculation is essentially an application of the Russo-Margulis formula with an exponential clock, in tandem with the tower law for conditional expections:

\begin{proposition}
\label{prop:lambda-derivatives}
    Let $\Lambda(s)$ be defined as above for a fixed function $g:\{\pm 1\}^n\to \mathbb{R}$. Then for all $s\geq 0$,
    \begin{align}
        \Lambda'(s)&=\mathbb{E}_{V_s}\left[\sum_{i\in V_s}\left(\|g_{V_s\setminus \{i\}}\|^2_{2,\mu}-\|g_{V_s}\|_{2,\mu}^2\right)\right]=-\mathbb{E}\left[\sum_{i\in V_s}\Var(g_{V_s}\vert X_{V_s\setminus \{i\}})\right]\\
        \Lambda''(s)&=-\Lambda'(s)+\mathbb{E}\left[\sum_{i\neq j\in V_s}\left(\|g_{V_s}\|_{2,\mu}^2-\|g_{V_s\setminus \{i\}}\|^2_{2,\mu}-\|g_{V_s\setminus \{j\}}\|^2_{2,\mu}+\|g_{V_s\setminus \{i,j\}}\|_{2,\mu}^2\right)\right]\\
        &= \mathbb{E}\left[\left(\sum_{i\in V_s} (g_{V_s}-g_{V_s\setminus\{i\}})\right)^2\right] - \mathbb{E}\left[\sum_{i\neq j\in V_s}\Cov\left(
g_{V_s\setminus\{i\}},
g_{V_s\setminus\{j\}}
\mid X_{V_s\setminus\{i,j\}}
\right)\right]
    \end{align}
\end{proposition}

To interpret these quantities, the masking process to first order considers the effect of dropping at most one coordinate infinitesimally, while the second derivative accounts for dropping at most two coordinates. The first derivative is simply the combined expected information loss for predicting $g$ by dropping each currently revealed coordinate in $V_s$, treating them separately. In the second derivative, the first term, which is manifestly positive, would also be the expected combined information loss by each revealed coordinate if the information revealed by each variable were \emph{independent} (so that the cross terms were zero). Indeed, this is indeed the case for product distributions. But for general non-product distributions, we must consider the second-order curvature induced by the fact that the information loss from different coordinates have overlapping direct information about $g$ and therefore we obtain a correction term depending on the covariances of this process.

\begin{proof}
    For a fixed time $s\geq 0$, condition on the unmasked set $V_s=V$. By the memorylessness of the exponential distribution, each $i\in V$ gets masked independently in the interval $[s,s+h]$ with probability $1-e^{-h}=h-h^2/2+O(h^3)$; in particular, the probability that at least three coordinates in $V$ are masked is $O(h^3)$. Therefore, by writing out the conditional expectation over the law of $V_{s+h}$ given $V_s=V$ in terms of $h$ and collecting powers of $h$ to second order:
    \begin{align*}
        \mathbb{E}_{V_{s+h}}\left[\|g_{V_{s+h}}\|_{2,\mu}^2\vert V_s=V\right]&=e^{-\vert V\vert h}\|g_{V}\|_{2,\mu}^2+e^{-(\vert V\vert-1)h}(1-e^{-h})\sum_{i\in V}\|g_{T\setminus \{i\}}\|^2_{2,\mu}\\
        &+e^{-(\vert V\vert-2)h}(1-e^{-h})^2\sum_{i< j\in V}\|g_{V\setminus \{i,j\}}\|_{2,\mu}^2+O(h^3)\\
        &=\left(1-\vert V\vert h+\frac{\vert V\vert^2h^2}{2}\right)\|g_{V}\|_{2,\mu}^2+\left(h-\left(\vert V\vert-\frac12\right)h^2\right)\cdot\sum_{i\in V}\|g_{T\setminus \{i\}}\|^2_{2,\mu}\\
        &+h^2\cdot \sum_{i< j\in V}\|g_{V\setminus \{i,j\}}\|_{2,\mu}^2+O(h^3)\\
        &=\|g_{V}\|_{2,\mu}^2+h\cdot \sum_{i\in V}\left(\|g_{V\setminus \{i\}}\|^2_{2,\mu}-\|g_V\|_{2,\mu}^2\right)+\frac{h^2}{2}\sum_{i\in V} \left(\|g_V\|_{2,\mu}^2-\|g_{V\setminus \{i\}}\|^2_{2,\mu}\right)\\
        &+\frac{h^2}{2}\sum_{i\neq j\in V}\left(\|g_V\|_{2,\mu}^2-\|g_{V\setminus \{i\}}\|^2_{2,\mu}-\|g_{V\setminus \{j\}}\|^2_{2,\mu}+\|g_{V\setminus \{i,j\}}\|_{2,\mu}^2\right)+O(h^3).
    \end{align*}
    Averaging over the random unmasked set $V_s=V$ gives $\Lambda(s+h)=\Lambda(s)+h\cdot \Lambda'(s)+\frac{h^2}{2}\cdot \Lambda''(s)+O(h^3)$, so we obtain the claim by comparing coefficients. The first identity then follows Lemma~\ref{lem:tower}.

    To obtain the second identity for $\Lambda''(s)$, we use the following key observation based on the tower law (as well as the polarization identity/completing the square): fixing $V_s=V$ and $i\neq j\in V_s$,
    \begin{align*}
        \|g_{V}\|_{2,\mu}^2-\|g_{V\setminus \{i\}}\|^2_{2,\mu}-\|g_{V\setminus \{j\}}\|^2_{2,\mu}+\|g_{V\setminus \{i,j\}}\|_{2,\mu}^2
&=\mathbb{E}_{\mu}\left[
(g_V-g_{V\setminus\{i\}})
(g_V-g_{V\setminus\{j\}})\right]\\&-
\mathbb{E}_{\mu}\left[
\Cov\left(
g_{V\setminus\{i\}},
g_{V\setminus\{j\}}
\mid X_{V\setminus\{i,j\}}
\right)\right].
    \end{align*}
    To see this identity, we can simply expand the product on the right and apply Lemma~\ref{lem:tower}, which recovers the first three terms on the left hand side but instead has $\mathbb{E}_{\mu}[g_{V\setminus \{i\}}g_{V\setminus \{j\}}]$ in place of the final term. But conditional on $X_{V\setminus \{i,j\}}$, both of these conditional expectations have conditional mean $g_{V\setminus \{i,j\}}$, so 
    \begin{equation*}
        \mathbb{E}_{\mu}[g_{V\setminus \{i\}}g_{V\setminus \{j\}}]=\mathbb{E}_{\mu}[g_{V\setminus \{i,j\}}^2]+\mathbb{E}_{\mu}\left[
\Cov\left(
g_{V\setminus\{i\}},
g_{V\setminus\{j\}}
\mid X_{V\setminus\{i,j\}}
\right)\right].
    \end{equation*}
    Putting everything together with $\Lambda'(s)$, which forms the diagonal below, it follows that
    \begin{equation*}
        \Lambda''(s) = \mathbb{E}\left[\left(\sum_{i\in V_s} (g_{V_s}-g_{V_s\setminus\{i\}})\right)^2\right] - \mathbb{E}\left[\sum_{i\neq j\in V_s}\Cov\left(
g_{V_s\setminus\{i\}},
g_{V_s\setminus\{j\}}
\mid X_{V_s\setminus\{i,j\}}
\right)\right].\qedhere
    \end{equation*}
\end{proof}

While the above expression is valid for any function, we need to relate these derivatives to the assumption that $g\perp P_r$ now are in position to see where the high-degree assumption plays a critical role in the revealment process. The essential observation above where this assumption plays a role is recorded below: the degree is revealed via the \emph{logarithmic derivative} of $\Lambda(s)$ at large times.

\begin{corollary}
\label{cor:log-reveals-degree}
    Let $\Lambda(s)$ be defined as above for a fixed nonzero function $g:\{\pm 1\}^n\to \mathbb{R}$ which further satisfies $g\perp \mathcal{P}_{r}$ with respect to $L^2(\mu)$. Then
    \begin{gather}
        \frac{-\Lambda'(0)}{\Lambda(0)}=\frac{\mathcal{E}_{\mu}(g,g)}{\Var(g)}\quad\quad
        \lim_{s\to \infty}\frac{-\Lambda'(s)}{\Lambda(s)}\geq r.
    \end{gather}
    In particular, the (negative) log-derivative of $\Lambda$ at zero yields the Rayley quotient with respect to Glauber dynamics while the behavior at infinity reveals the degree.
\end{corollary}
\begin{proof}
    For the first identity, we simply compute from Proposition~\ref{prop:lambda-derivatives}:
    \begin{equation*}
        \Lambda'(0)=-\sum_{i\in [n]}\mathbb{E}[\Var(g\vert X_{-i})],
    \end{equation*}
    which is precisely the (negative of the) Dirichlet form for Glauber dynamics. On the other hand, $\Lambda(0)=\|g\|_{2,\mu}^2=\Var(g)$ since $g$ must be mean zero by the assumption it is orthogonal to $\mathcal{P}_{r}$.

    For the second, since by Lemma~\ref{lem:orthogonal-ld} $g_V\equiv 0$ whenever $\vert V\vert< r$, it follows by writing out the probabilities that
    \begin{equation*}
        \Lambda(s)=\sum_{V\subseteq [n]} \Pr(V_s=V) \|g_V\|_{2,\mu}^2=\sum_{V:\vert V\vert\geq r} e^{-s\vert V\vert}(1-e^{-s})^{n-\vert V\vert}\|g_V\|_{2,\mu}^2
    \end{equation*}
    Letting $r_0\geq r$ be the smallest integer such that there exists a subset $V$ of cardinality $r_0$ with $g_V\not\equiv 0$, we thus find for large $s$ that
    \begin{equation*}
        \Lambda'(s)=\sum_{V:\vert V\vert\geq r_0} (-\vert V\vert)e^{-s\vert V\vert}(1-e^{-s})^{n-\vert V\vert}\|g_V\|_{2,\mu}^2+O(e^{-s(r_0+1)}).
    \end{equation*}
    If we divide by $-\Lambda(s)$, each term-by-term ratio is at least $r_0$, and hence by averaging:
    \begin{equation*}
        \frac{-\Lambda'(s)}{\Lambda(s)}\geq  r_0+O(e^{-s}).
    \end{equation*}
    Taking limits completes the proof.
\end{proof}

Our goal is now to establish exponential decay of a function $g\perp \mathcal{P}_{r}$, with exponent improving with $r$ under this noise operator. Since we know that $\varphi(s):=-\Lambda'(s)/\Lambda(s)$,  is asymptotically at least $r$ at large times by Corollary~\ref{cor:log-reveals-degree}, we need to show that $\varphi(s)$ remains large throughout the process. To do this, a natural approach is to show that the derivative of the \emph{logarithm} of $\varphi$ remains well-behaved; in particular, if it is integrable, then it will follow by integration that $\varphi(s)\gtrsim r$ for all $s$, which can then be integrated again to show decay of stability. While we will not be able to show \emph{quite} this lower bound in all applications, this motivates the following key calculation from which we will derive all of our subsequent results:

\begin{corollary}
\label{cor:log-decay}
    Let $r\geq 1$ and let $g:\{\pm 1\}^n\to \mathbb{R}$ be nonzero and orthogonal to $\mathcal{P}_{r}$. Then for all $s>0$,
    \begin{align*}
    \frac{\mathrm d}{\mathrm ds}\log\varphi(s)
    &=-\frac{1}{-\Lambda'(s)}
    \left(\mathbb{E}[D_s^2]
    -\frac{\mathbb{E}[g_{V_s}D_s]^2}
    {\mathbb{E}[g_{V_s}^2]}\right)
    +\frac{K(s)}{(e^s-1)G(s)}\\
    &\leq \frac{K(s)}{(e^s-1)G(s)}.
    \end{align*}
    where we define:
    \begin{align*}
    D_s&:=\sum_{i\in V_s}
    (g_{V_s}-g_{V_s\setminus \{i\}})\\
    G(s)
    &:=\mathbb{E}_{V_s}
    \sum_{i\notin V_s}
    \mathbb{E}_{\mu}\left[
    \operatorname{Var}\left(
    g_{V_s\cup\{i\}}\mid X_{V_s}
    \right)\right]\\
    K(s)&:=\mathbb{E}_{V_s}
    \sum_{\substack{i\neq j\\i,j\notin V_s}}
    \mathbb{E}_{\mu}\left[
    \operatorname{Cov}\left(
    g_{V_s\cup\{i\}},
    g_{V_s\cup\{j\}}
    \mid X_{V_s}
    \right)\right].
    \end{align*}
    In particular, for any $s\geq 0$
    \begin{equation}
    \label{eq:phi-lower-bound}
    \varphi(s)\geq r\cdot\exp\left(
    -\int_s^\infty
    \frac{K(u)}{(e^u-1)G(u)}\,\mathrm{d}u
    \right).
    \end{equation}
\end{corollary}

\begin{proof}
    From elementary calculus, we have
    \begin{align*}
        \frac{\mathrm{d}}{\mathrm{d}s}\log(\varphi(s))
        =\frac{\varphi'(s)}{\varphi(s)}
        &=\frac{\Lambda''(s)}{\Lambda'(s)}
        -\frac{\Lambda'(s)}{\Lambda(s)}=-\frac{1}{-\Lambda'(s)}
        \left(\Lambda''(s)
        -\frac{(\Lambda'(s))^2}{\Lambda(s)}\right).
    \end{align*}
    Note that we would immediately obtain nonpositivity if we knew that $\Lambda(s)$ were log-convex, which is easy to show directly in the product distribution case using \Cref{fact:expressions_of_noise_stability}. In that case, the Rayley quotient along the masking process would always be at least $r$ by Corollary~\ref{cor:log-reveals-degree}. However, we generally will \emph{not} have such strong control due to the second-order curvature terms of Proposition~\ref{prop:lambda-derivatives}, so we will instead rewrite everything in terms of the curvature part along with a clearly nonpositive part.

    To do this, we require the following elementary, but slightly tricky manipulations: first, observe that
    \begin{align*}
        \Lambda'(s)
        &= \mathbb{E}_{V_s}\left[
        \sum_{i\in V_s}
        \left(\|g_{V_s\setminus \{i\}}\|^2_{2,\mu}
        -\|g_{V_s}\|_{2,\mu}^2\right)\right]=\mathbb{E}\left[
        \sum_{i\in V_s}
        \left(g_{V_s\setminus \{i\}}^2-g_{V_s}^2\right)\right]\\
        &=-\mathbb{E}\left[
        \sum_{i\in V_s}
        \left(g_{V_s\setminus \{i\}}-g_{V_s}\right)^2
        \right]
        =-\mathbb{E}\left[g_{V_s}D_s\right].
    \end{align*}
    The second equality is by definition (pulling the average over the sample outside), while the third is by the fact $g_{V_s\setminus\{i\}}$ is an orthogonal projection of $g_{V_s}$, and the last is an application of Lemma~\ref{lem:tower}.
    Similarly, it is easy to see that
    \begin{equation*}
        \Lambda''(s)
        = \mathbb{E}[D_s^2]
        - \mathbb{E}\left[
        \sum_{i\neq j\in V_s}
        \Cov\left(
        g_{V_s\setminus\{i\}},
        g_{V_s\setminus\{j\}}
        \mid X_{V_s\setminus\{i,j\}}
        \right)\right].
    \end{equation*}
    Putting this together along with
    $\Lambda(s) = \mathbb{E}[g_{V_s}^2]$, we obtain
    \begin{align}
    \label{eq:almost-square}
        \frac{\mathrm{d}}{\mathrm{d}s}\log(\varphi(s))
        &=-\frac{1}{-\Lambda'(s)}
        \left(\mathbb{E}[D_s^2]
        -\frac{\mathbb{E}[g_{V_s}D_s]^2}
        {\mathbb{E}[g_{V_s}^2]}\right)+
        \frac{
        \mathbb{E}\left[
        \sum_{i\neq j\in V_s}
        \Cov\left(
        g_{V_s\setminus\{i\}},
        g_{V_s\setminus\{j\}}
        \mid X_{V_s\setminus\{i,j\}}
        \right)\right]}
        {-\Lambda'(s)}.
    \end{align}

    To obtain the final identity, we need a simple change-of-measure argument to replace the random tuple $(V_s\setminus \{i,j\},i,j)$ with $i\neq j\in V_s$ with the random tuple $(V_s,i,j)$ with $i\neq j\not\in V_s$. Formally, for any fixed subset $T$ not containing $i$ or $j$,
    \begin{equation*}
        \frac{\Pr(V_s=T\cup\{i,j\})}
        {\Pr(V_s=T)}
        =\frac{e^{-2s}}{(1-e^{-s})^2}
        =\frac{1}{(e^s-1)^2}.
    \end{equation*}
    Using this simple identification of tuples along with the likelihood ratio,
    \begin{align*}
        &\mathbb{E}\left[
        \sum_{i\neq j\in V_s}
        \Cov\left(
        g_{V_s\setminus\{i\}},
        g_{V_s\setminus\{j\}}
        \mid X_{V_s\setminus\{i,j\}}
        \right)\right]\\
        &=\sum_{V\subseteq[n]}
        \sum_{i\neq j\in V}
        \Pr(V_s=T)\,
        \mathbb{E}_{\mu}\left[
        \Cov\left(
        g_{V\setminus\{i\}},
        g_{V\setminus\{j\}}
        \mid X_{V\setminus\{i,j\}}
        \right)\right]\\
        &=\sum_{V\subseteq[n]}
        \sum_{\substack{i\neq j\\i,j\notin V}}
        \Pr(V_s=V\cup\{i,j\})\,
        \mathbb{E}_{\mu}\left[
        \Cov\left(
        g_{V\cup\{i\}},
        g_{V\cup\{j\}}
        \mid X_T
        \right)\right]\\
        &=\frac{1}{(e^s-1)^2}
        \sum_{V\subseteq[n]}
        \Pr(V_s=V)
        \sum_{\substack{i\neq j\\i,j\notin V}}
        \mathbb{E}_{\mu}\left[
        \Cov\left(
        g_{V\cup\{i\}},
        g_{V\cup\{j\}}
        \mid X_V
        \right)\right]\\
        &=\frac{K(s)}{(e^s-1)^2}.
    \end{align*}
    By an identical change-of-measure argument using the variance representation of $\Lambda'$ in Proposition~\ref{prop:lambda-derivatives}, we also have
    \begin{equation*}
        -(e^s-1)\Lambda'(s)=G(s).
    \end{equation*}
    Putting the preceding displays with Equation~\ref{eq:almost-square}, we obtain the stated identity on the derivative.

    The upper bound follows immediately from the fact that $\Lambda'(s)$ is clearly non-positive (and non-zero) by Proposition~\ref{prop:lambda-derivatives} while by Cauchy-Schwarz:
    \begin{equation*}
        \mathbb{E}[g_{V_s}D_s]^2
        \leq
        \mathbb{E}[g_{V_s}^2]\mathbb{E}[D_s^2].
    \end{equation*}

    Finally, for any $0<s<t$, we may write
    \begin{equation*}
        \log(\varphi(s)/\varphi(t))
        =\log(\varphi(s))-\log(\varphi(t))
        =-\int_s^t
        \frac{\mathrm{d}}{\mathrm{d}u}\log(\varphi(u))
        \,\mathrm{d}u
        \geq
        -\int_s^t
        \frac{K(u)}{(e^u-1)G(u)}
        \,\mathrm{d}u.
    \end{equation*}
    Taking limits as $t\to \infty$ and using Corollary~\ref{cor:log-reveals-degree}, we deduce
    \begin{equation*}
        \log(\varphi(s))
        \geq \log r-\int_s^{\infty}
        \frac{K(u)}{(e^u-1)G(u)}\,\mathrm{d}u,
    \end{equation*}
    and Equation~\ref{eq:phi-lower-bound} follows by exponentiation. The case $s=0$ follows by continuity.
\end{proof}

\subsection{Decay at High-Temperature}
In this section, we will show how to bound the key integral of Corollary~\ref{cor:log-decay} under different high-temperature assumptions on the underlying spin system to establish faster decay of high-degree functions under the masking process. Already, it is apparent that the change in the decay rate is intimately related to the behavior of $K(s)/G(s)$, which relates the covariance between the information revealed by distinct coordinates with the overall sum of variances of each such revelation separately. In particular, this ratio measures the relative size of the \emph{off-diagonal} parts of a pinned covariance matrix with the diagonal part, and our goal is to show that this ratio is well-behaved under well-studied high-temperature conditions.

\subsubsection{Uniform Poincar\'e Constants}

In light of this discussion, in all of our applications to spin systems $\mu$ on $\{\pm 1\}^n$, we will make the very natural assumption that the model satisfies a uniform Poincar\'e inequality on each pinning:

\medskip
\noindent\phantomsection\label{ass:loc-pi}
\textbf{$(\mathrm{PI}_{\alpha})$ Pinned Poincar\'e.}
For some $0<\alpha\leq 1$, every pinning $\nu$ of $\mu$,
on any set $F$ of free coordinates, satisfies
\[
\Var_{\nu}(h)
\leq \frac{1}{\alpha}\sum_{i\in F}
\mathbb{E}_{\nu}\left[
\Var_{\nu}(h\mid X_{F\setminus\{i\}})
\right]
\]
for every $h:\{\pm1\}^F\to\mathbb{R}$.

An interesting example of graphical models that satisfy \hyperref[ass:loc-pi]{$(\mathrm{PI}_{\alpha})$} but none of the further assumptions later on are higher-order $p$-spin models. It is well-known (see e.g. Theorem 6 of Anari, Koeher, Jain, Pham, and Vuong~\cite{DBLP:conf/soda/AnariJKPV24}) that $p$-spin Gaussian models satisfy an $O(1)$-Poincar\'e inequality for all pinnings at sufficiently high, but constant temperature with high probability.
From \hyperref[ass:loc-pi]{$(\mathrm{PI}_{\alpha})$} alone, it turns out that we can establish decay of high-degree functions:

\begin{theorem}[Decay under pinned Poincar\'e]\label{thm:loc-pi}
Suppose that the distribution $\mu$ satisfies Assumption~\hyperref[ass:loc-pi]{$(\mathrm{PI}_{\alpha})$}
 and that $g\perp \mathcal{P}_{r}$ for an integer $r\geq 1$. Then, for every $0\leq\rho\leq 1$,
\begin{equation*}
    \operatorname{Stab}_{\rho,\mu}(g)
    \leq\rho^{\alpha\cdot ((e-1)/e^2)\cdot r^{\alpha}}\,\|g\|_{2,\mu}^2.
\end{equation*}
\end{theorem}

To show this, we will require the following simple bounds that will be useful in our other applications as well. The first is a simple bound using the Poincar\'e inequality that shows that the relevant integral is easy to handle for \emph{long times} in the masking process.
\begin{lemma}
\label{lem:long-times}
    Suppose that a distribution $\mu$ satisfies \hyperref[ass:loc-pi]{$(\mathrm{PI}_{\alpha})$} and let $g:\{\pm 1\}^n\to \mathbb{R}$ be any fixed function. Then for all $s\geq 0$,
    \begin{equation*}
        K(s)\leq \left(\frac{1}{\alpha}-1\right)\cdot G(s).
    \end{equation*}
\end{lemma}
\begin{proof}
    First, observe that we have the following simple calculation:
    \begin{align*}
        0&\leq \mathbb{E}\left[\Var\left(\sum_{i\not\in V_s} (g_{T_{s}\cup \{i\}}-g_{V_s})\bigg\vert X_{V_s}\right)\right]\\
        &=\mathbb{E}\left[\sum_{i\not\in V_s} \Var(g_{V_s\cup \{i\}}-g_{V_s}\vert X_{V_s})\right]+\mathbb{E}\left[\sum_{i\neq j\not\in V_s} \Cov\left(g_{V_s\cup \{i\}}-g_{V_s},g_{V_s\cup \{j\}}-g_{V_s}\big\vert X_{V_s}\right)\right]\\
        &=\mathbb{E}\left[\sum_{i\not\in V_s} \Var(g_{V_s\cup \{i\}}\vert X_{V_s})\right]+K(s)=G(s)+K(s)
    \end{align*}
    In the last step, we use the definition of $K(s)$ as well as the fact that conditional on $X_{V_s}$, $g_{V_s\cup \{i\}}-g_{V_s}$ is centered.  If we now apply the Poincar\'e inequality on the pinned measure conditional on any fixed value of $X_{V_s}$ to the sum of the centered terms, we also obtain
    \begin{align*}
        \Var\left(\sum_{i\not\in V_s}(g_{V_s\cup\{i\}}-g_{V_s})\bigg\vert X_{V_s}\right)&\leq \frac{1}{\alpha}\sum_{j\not\in V_s} \mathbb{E}\left[\Var\left(\sum_{i\not\in V_s}(g_{V_s\cup\{i\}}-g_{V_s})\bigg\vert X_{-j}\right)\bigg\vert X_{V_s}\right]\\
        &=\frac{1}{\alpha}\sum_{i\not\in V_s}\mathbb{E}\left[\Var(g_{V_s\cup\{i\}}\vert X_{-i})\big\vert X_{V_s}\right].
    \end{align*}
    In the last step, we use the fact that $g_{V_s\cup\{i\}}-g_{V_s}$ is measurable with respect to $X_{-j}$ whenever $i\neq j$, so only the term $i=j$ contributes to each conditional variance. Averaging over $X_{V_s}$ and $V_s$, this yields
    \begin{equation*}
        G(s)+K(s)\leq \frac{1}{\alpha}\mathbb{E}\left[\sum_{i\not\in V_s}\Var(g_{V_s\cup \{i\}}\vert X_{-i})\right]\leq \frac{1}{\alpha}G(s),
    \end{equation*}
    where the last inequality is the law of total variance. Rearranging gives the desired conclusion.
\end{proof}

\begin{corollary}
\label{cor:far-decay}
    Suppose that a distribution $\mu$ satisfies Assumption~\hyperref[ass:loc-pi]{$(\mathrm{PI}_{\alpha})$} and further that a nonzero function $g:\{\pm 1\}^n\to \mathbb{R}$ satisfies $g\perp \mathcal{P}_{r}$. Then for all $s\geq 0$, with the right-hand side interpreted by continuity at $s=0$, it holds that
    \begin{equation*}
        \varphi(s)\geq r\cdot (1-\exp(-s))^{1/\alpha-1}.
    \end{equation*}
    In particular, for any $0\leq s\leq t$
    \begin{equation*}
        \Lambda(t)\leq \Lambda(s)\cdot \exp\left(-r\cdot\int_s^t (1-e^{-u})^{1/\alpha-1}\mathrm{d}u\right)
    \end{equation*}
\end{corollary}
\begin{proof}
    By Lemma~\ref{lem:long-times}, observe that for $s>0$,
    \begin{equation*}
        \int_s^{\infty} \frac{K(u)}{(e^u-1)G(u)}\mathrm{d}u\leq \left(\frac{1}{\alpha}-1\right)\cdot \int_{s}^{\infty} \frac{1}{e^u-1}\mathrm{d}u=-\left(\frac{1}{\alpha}-1\right)\cdot \log(1-e^{-s});
    \end{equation*}
    the last integral can be obtained by substituting $v=e^s-1$ and then doing a simple partial fraction decomposition. By Corollary~\ref{cor:log-decay}, it follows that
    \begin{equation*}
        \varphi(s)\geq r\cdot \exp\left(\left(\frac{1}{\alpha}-1\right)\cdot \log(1-e^{-s})\right)=r\cdot (1-e^{-s})^{1/\alpha-1}.
    \end{equation*}
    The final claim then follows by plugging this inequality into the identity:
    \begin{equation*}
        \log(\Lambda(s))-\log(\Lambda(t))=\int_{s}^t \varphi(u)\mathrm{d}u.
    \end{equation*}
    The case $s=0$ follows by continuity.
\end{proof}
The issue with using Corollary~\ref{cor:far-decay} along to obtain Theorem~\ref{thm:loc-pi} is that when $s$ is close to zero, the bound becomes vacuous in any non-product setting as $\alpha<1$. However, there is a simple fix: we can show, using the Poincar\'e inequality again, that any decay as we go from $\Lambda(s)$ to $\Lambda(t)$ can be \emph{accounted} for by $\Lambda'(s)$. The idea is simply that the difference is a conditional variance, which after applying Poincar\'e is the same as $\Lambda'(s)$ up to the probability that each unmasked variable at time $s$ becomes masked by time $t$. This will allows us to derive a suitable lower bound for $-\Lambda'(s)/\Lambda(s)$ if we choose a good value of $t$, which we will then integrate to obtain stability decay. 

To carry out this argument, we need the following comparison lemma:

\begin{lemma}
\label{lem:short-ub}
    For any function $g:\{\pm 1\}^n\to \mathbb{R}$, any distribution $\mu$ satisfying Assumption~\hyperref[ass:loc-pi]{$(\mathrm{PI}_{\alpha})$}, and any times $0\leq s\leq t$,
    \begin{equation*}
        0\leq \Lambda(s)-\Lambda(t)\leq \frac{1-e^{-(t-s)}}{\alpha}\cdot (-\Lambda'(s)).
    \end{equation*}
\end{lemma}
\begin{proof}
    By definition, $\Lambda(s)-\Lambda(t)=\mathbb{E}[g_{V_s}^2-g_{T_t}^2]$ where we can additionally trivially couple the masking process over time so that $T_t$ is obtained by retaining each element of $V_s$ with probability $e^{-(t-s)}$. For any fixed $T_t\subseteq V_s$ and using the tower law in Lemma~\ref{lem:tower}, we may deterministically apply Assumption~\hyperref[ass:loc-pi]{$(\mathrm{PI}_{\alpha})$} to find
    \begin{align*}
        \mathbb{E}_X[g_{V_s}^2-g_{T_t}^2]=\mathbb{E}_X[\Var(g_{V_s}\vert X_{T_t})]&\leq \frac{1}{\alpha}\mathbb{E}_X\left[\sum_{i\in T_{s}\setminus T_t}\mathbb{E}[\Var(g_{V_s}\vert X_{V_s\setminus \{i\}})\vert X_{T_t}]\right]\\
        &= \frac{1}{\alpha}\mathbb{E}_X\left[\sum_{i\in V_s\setminus T_t}\Var(g_{V_s}\vert X_{V_s\setminus \{i\}})\right].
    \end{align*}
    If we condition on $V_s$ and now average over $T_{t}$, each element of $V_s$ contributes to the sum with probability exactly $1-e^{-(t-s)}$, so we deduce
    \begin{equation*}
        \Lambda(s)-\Lambda(t)\leq\frac{1-e^{-(t-s)}}{\alpha}\cdot \mathbb{E}\left[\sum_{i\in V_s}\Var(g_{V_s}\vert X_{V_s\setminus \{i\}})\right]=\frac{1-e^{-(t-s)}}{\alpha}\cdot(-\Lambda'(s)),
    \end{equation*}
    where the last identity is Proposition~\ref{prop:lambda-derivatives}.
\end{proof}

We are now ready to establish Theorem~\ref{thm:loc-pi}:
\begin{proof}[Proof of Theorem~\ref{thm:loc-pi}]
    The case $g=0$ is immediate, so assume $g\neq 0$. Fix any $s>0$. Our goal will be to obtain a universal bound on $\varphi(s)$ by combining the previous two results with a suitable choice of $t$ to leverage long-time decay with the short-term stability of $\Lambda'(s)$. 

    Indeed, for any $t\geq s$, \Cref{lem:long-times} and \Cref{lem:short-ub} imply that
    \begin{equation*}
        \Lambda(s)\cdot \left(1-\exp\left(-r\cdot\int_s^t (1-e^{-u})^{1/\alpha-1}\mathrm{d}u\right)\right)\leq \Lambda(s)-\Lambda(t)\leq \frac{1-e^{-(t-s)}}{\alpha}\cdot (-\Lambda'(s)),
    \end{equation*}
    so rearranging:
    \begin{equation}
    \label{eq:phi-poincare-comparison}
        \varphi(s)=-\frac{\Lambda'(s)}{\Lambda(s)}\geq \alpha\left(\frac{1-\exp\left(-r\cdot\int_s^t (1-e^{-u})^{1/\alpha-1}\mathrm{d}u\right)}{1-e^{-(t-s)}}\right)
    \end{equation}
    To massage the numerator, we use the following simple lower bound:
    \begin{align*}
        \int_s^t (1-e^{-u})^{1/\alpha-1}\mathrm{d}u&\geq \int_0^{t-s}(1-e^{-u})^{1/\alpha-1}\mathrm{d}u\\
        &=\int_0^{1-e^{-(t-s)}}\frac{v^{1/\alpha-1}}{1-v}\mathrm{d}v\\
        &\geq \alpha\cdot (1-e^{-(t-s)})^{1/\alpha}\\
        \implies \varphi(s)&\geq \alpha\left(\frac{1-\exp(-r\alpha(1-e^{-(t-s)})^{1/\alpha})}{1-e^{-(t-s)}}\right)
    \end{align*}
    where in the last step we just dropped the denominator.

    To conclude, observe that if $r\alpha\leq 1$, taking $t\to \infty$ in Equation~\ref{eq:phi-poincare-comparison} shows that $\varphi(s)\geq \alpha$ since the integral then diverges (the integrand tends to $1$). Then if $g\perp V_{<1}$,
    \begin{equation*}
        \log(\Lambda(s))-\log(\Lambda(0))=-\int_0^s \varphi(u)\mathrm{d}u\leq -s\alpha\iff \Lambda(s)\leq e^{-\alpha s}\|g\|_{2,\mu}^2=\rho^{\alpha}\|g\|_{2,\mu}^2.
    \end{equation*}
    To conclude, simply observe that $r^{\alpha}\leq \alpha^{-\alpha}\leq e$, and so 
    \begin{equation*}
        \Lambda(s)\leq \rho^{(\alpha/e)r^{\alpha}}\cdot \|g\|_{2,\mu}^2.
    \end{equation*}
    
    For $r\alpha>1$, we can instead argue as follows: if $\theta = 1-e^{-(t-s)}$, then
    \begin{equation*}
        \varphi(s)\geq \alpha \left(\frac{1-\exp(-r\alpha \theta^{1/\alpha})}{\theta}\right).
    \end{equation*}
    To approximately optimize this over $\theta\in (0,1)$, we can take $\theta=(\alpha r)^{-\alpha}$ (which is valid since $\alpha r>1$) to obtain
    \begin{equation*}
        \varphi(s)\geq \alpha \left(1-1/e\right)\cdot (\alpha r)^{\alpha}\geq \alpha ((e-1)/e^2)\cdot r^{\alpha}.
    \end{equation*}
    The same derivative argument shows that
    \begin{equation*}
        \Lambda(s)\leq \exp(-s\alpha((e-1)/e^2)r^{\alpha})\cdot\|g\|_{2,\mu}^2=\rho^{\alpha\cdot ((e-1)/e^2)\cdot r^{\alpha}}\cdot \|g\|_{2,\mu}^2.
    \end{equation*}
    Therefore, the stated bound holds in both cases and completes the proof.
\end{proof}

For our improvements to this bound in the subsequent sections to obtain \emph{linear} degree $r$ dependence for most models of interest, we will need to find alternative strengthenings of the bounds to handle the behavior of $\varphi(s)$ near 0. To set this up, we record the following elementary lemma that shows that the conditional variance of any variable in predicting $g$ can be charged against the conditional variance when including another fixed variable under the Poincar\'e inequality:

\begin{lemma}
\label{lem:two-poincare}
    Fix any function $g:\{\pm 1\}^n\to\mathbb R$ and any
    distribution $\mu$ satisfying Assumption~
    \hyperref[ass:loc-pi]{$(\mathrm{PI}_{\alpha})$}.
    For any deterministic subset $T\subseteq[n]$ and
    distinct $i,j\notin T$, we have
    \begin{equation*}
        \Var(g_{T\cup\{i\}}\mid X_T)
        \leq \frac{1}{\alpha}\mathbb E\Big[
        \Var(g_{T\cup\{i,j\}}\mid X_{T\cup\{j\}})+\Var(g_{T\cup\{i,j\}}\mid X_{T\cup\{i\}})
        \,\Big\vert X_T\Big].
    \end{equation*}
\end{lemma}

\begin{proof}
    Fix $T$ not containing $i$ or $j$, and condition
    on any value of $X_T$. The marginal distribution
    of $(X_i,X_j)$ satisfies $(\mathrm{PI}_{\alpha})$.
    Indeed, we may apply the pinned Poincar\'e inequality
    to functions depending only on $X_i,X_j$, noting
    that conditioning on additional coordinates can
    only decrease the expected conditional variances. Then by the law of total variance and then the
    Poincar\'e inequality on this two-coordinate marginal,
    \begin{align*}
        \Var(g_{T\cup\{i\}}\mid X_T)
        &\leq \Var(g_{T\cup\{i,j\}}\mid X_T)\\
        &\leq \frac{1}{\alpha}\mathbb E\Big[
        \Var(g_{T\cup\{i,j\}}\mid X_{T\cup\{j\}})
        +\Var(g_{T\cup\{i,j\}}\mid X_{T\cup\{i\}})
        \,\Big\vert X_T\Big].\qedhere
    \end{align*}
\end{proof}
\subsection{Improved Decay under Additional Assumptions}
We now list additional assumptions on our distribution that imply decay with $\beta=1$. As these arguments require somewhat more technical details, we defer the proofs to \Cref{sec:improved} after showing their applications to learning.

First, we consider graphical models with a bounded degree dependency graph.

\medskip
\noindent\phantomsection\label{ass:loc-bd}
\textbf{$(\mathrm{BD}_{\Delta})$ Bounded degree.}
The distribution $\mu$ is Markov with respect to a graph
$H=([n],E)$ of maximum degree at most $\Delta$.
Equivalently, for every $i\in[n]$,
\[
X_i\perp X_{[n]\setminus(N_H(i)\cup\{i\})}
\mid X_{N_H(i)}.
\]
\begin{theorem}[Decay for bounded-degree graphical models]\label{thm:loc-bd}
Suppose that the distribution $\mu$ satisfies Assumptions
\hyperref[ass:loc-pi]{$(\mathrm{PI}_{\alpha})$} and
\hyperref[ass:loc-bd]{$(\mathrm{BD}_{\Delta})$}, and that $g\perp \mathcal{P}_{r}$.
Then there is a constant $c_{\alpha,\Delta}>0$ such that, for every $0\leq\rho\leq1$,
\begin{equation*}
    \operatorname{Stab}_{\rho,\mu}(g)
    \leq\rho^{c_{\alpha,\Delta}r}\,\|g\|_{2,\mu}^2,
\end{equation*}
where $c_{\alpha,\Delta}=1$ if $\Delta=0$ or $\alpha=1$, and otherwise $c_{\alpha,\Delta}=\left(\frac{1}{e(1+2\Delta/\alpha)}\right)^{1/\alpha-1}$.
\end{theorem}

For Ising models specifically, we can additionally optimal decay under other mild assumptions that apply for a wide variety of models studied in the sampling community. The first condition asserts \emph{any} constant bound on the $\ell_1$ norms of interactions, which is a standard (and necessary) assumption for the well-studied problem of distribution learning~\cite{DBLP:conf/focs/KlivansM17}.

\medskip
\noindent\phantomsection\label{ass:loc-bw}
\textbf{$(\mathrm{BW}_{B})$ Bounded width.}
The distribution $\mu$ is an Ising model with interaction
matrix $A$ satisfying
\[
\max_{i\in[n]}\sum_{j\neq i}|A_{ij}|\leq B.
\]

To handle interesting mean-field models, like spin glasses, that do not satisfy bounded width, we instead impose an incomparable condition on the eigenvalue spread of submatrices. For instance, in the Sherrington-Kirkpatrick model with Gaussian disorder, the width will be polynomially large. However, all submatrices are Gaussian which admit strong operator norm bounds, and in particular one can obtain analogous \emph{normalized} bounds on \emph{all} submatrices. Formally, we will consider the following assumption:

\medskip
\noindent\phantomsection\label{ass:loc-mf}
\textbf{$(\mathrm{MF}_{\omega,L})$ Mean-field interactions.}
There is a concave, nondecreasing function $\omega:[0,1]\to \mathbb{R}_+$ with $\omega(0)=0$ that satisfies the following properties. The distribution $\mu$ is an Ising model with interaction
matrix $A$ such that, for every nonempty $F\subseteq[n]$,
\[
\|A_{F,F}\|_{\mathsf{op}}
\leq \omega\left(\frac{|F|-1}{n}\right).
\]
Moreover,
\begin{equation*}
    \int_0^1 \frac{\omega(s)}{s}\mathrm{d}s=L<\infty
\end{equation*}

For instance, one can show that for the Sherrington-Kirkpatrick model at any constant inverse temperature $\beta$, with high probability the eigenvalue spread of all submatrices with associated density $p$ is at most order $\beta \sqrt{p\log(e/p)}$ (see e.g. Proposition 2.2. of Huang and Sellke~\cite{strong-ld-hardness} for the simple proof for a fixed $p\in [0,1]$, but it holds for all $p$ simultaneously by further union bounding over $p=2^{-i}$ for $i=1,\ldots,\log(n)$ with no change). In this case, $\omega(1)$ and $L$ are both $O(\beta)$. One may think of $\omega(x)\asymp x^q$ for some fixed $q\in(0,1]$.

We can now state a precise result establishing $(c,1)$-decay for any Ising model that additionally satisfies either of these conditions beyond Poincar\'e:
\begin{theorem}[Decay for bounded-width or diffuse Ising models]\label{thm:loc-ising}
Suppose that the distribution $\mu$ satisfies Assumption
\hyperref[ass:loc-pi]{$(\mathrm{PI}_{\alpha})$} and  that $g\perp \mathcal{P}_{r}$.
\begin{itemize}
    \item If additionally
     Assumption
\hyperref[ass:loc-bw]{$(\mathrm{BW}_{B})$} holds, then $\mu$ exhibits $(c_{\alpha,B},1)$-decay
where $c_{\alpha,B}=1$ if $B=0$ or $\alpha=1$ and otherwise:
\begin{equation*}
\gamma_{\alpha,B}
:=\frac{2B}{\alpha}
\left(2+\frac{16Be^{B/2}}{\alpha}\right),
\qquad
\eta_{\alpha}:=\frac{1}{\alpha}-1,
\qquad
c_{\alpha,B}:=\left(\dfrac{\min\{\gamma_{\alpha,B},\eta_{\alpha}\}}{e(\gamma_{\alpha,B}+\eta_{\alpha})}\right)^{\min\{\gamma_{\alpha,B},\eta_{\alpha}\}}
\end{equation*}
\item If instead Assumption \hyperref[ass:loc-mf]{$(\mathrm{MF}_{\omega,L})$} holds, then $\mu$ exhibits $(c_{\alpha,\omega,L},1)$-decay where 
\begin{equation*}
 c_{\alpha,\omega,L}=\exp\left(-\left(2+\frac{16 \omega(1)\exp(\omega(1)/2)}{\alpha}\right)L\right)\cdot \left(\frac{\alpha}{\alpha+2}\right)^{1/\alpha-1}.
\end{equation*}
\end{itemize}
\end{theorem}

\section{Polynomial Approximation from Noise Sensitivity and Decay}

We also introduce some notation for degree of approximation. For any distribution $\mu$ and function $f$, we define $\deg(f,\mu,\varepsilon)=\min\{r\mid \min_{p\in V_{<r+1}}\E_{\mu}[(f-p)^2]\leq \varepsilon\}$. For a function class $\mathcal{F}$, we define $\deg(\mathcal{F},\mu,\varepsilon)=\max_{f\in \mathcal{F}}\deg(f,\mu,\varepsilon)$

We now show that the following holds for any distribution $\mu$ having $(c,\beta)$ decay for $c,\beta>0$: functions with low noise sensitivity have low-degree polynomial approximators. 

\begin{theorem}[Noise sensitivity implies polynomial approximation]
\label{thm:noise_stability_implies_approximation}
   Let $\mu$ be a distribution on $\cube{n}$ that exhibits $(c,\beta)$-decay, and let $f:\cube{n}\to\{\pm 1\}$ be any boolean function. Then, for any $0\leq \delta\leq 1$ , it holds that 
   \begin{equation}
       \label{eqn:noise_sensitivity_to_poly_approx}
    \mutwonorm{f-\Pi_{<r}f}{\mu}^2\leq 8\NS_{\delta,\mu}(f)+2(1-\delta)^{cr^{\beta}}.
   \end{equation}
   Furthermore, if $\NS_{\delta,\mu}(f)\leq \varepsilon/16$, then 
   \begin{equation}
       \label{eqn:degree_bound_noise_sensitivity}
       \deg(f,\mu,\varepsilon)\leq \left(\frac{\log(4/\varepsilon)}{c\delta}\right)^{1/\beta}.
   \end{equation}
\end{theorem}
\begin{proof}
    Let $\rho=1-\delta$. The proof proceeds in two steps. First, we will argue that $f$ is well approximated by $T_{\rho,\mu}f$. This will be true from the definition of noise sensitivity. We will then argue that $T_{\rho,\mu}f$ is well approximated by low-degree polynomials, using the $(c,\beta)$-decay exhibited by $\mu$. 

    We now present the first step. We have that 
    \begin{align*}
        \mutwonorm{f-T_{\rho,\mu}f}{\mu}^2&=\langle f-T_{\rho,\mu}f,f-T_{\rho,\mu}f\rangle_{\mu}\\
        &=\mutwonorm{f}{\mu}^2+\mutwonorm{T_{\rho,\mu}f}{\mu}^2-2\langle f,T_{\rho,\mu}f\rangle_{\mu}\\
        &\leq 2-2\Stab_{\rho,\mu}f\\
        &=4\NS_{\delta,\mu}(f)
    \end{align*}
    where the inequality used the definition of noise stability and the fact that $|f|\leq 1$ is bounded. 

    We now approximate $T_{\rho,\mu}f$ by low-degree polynomials. Let $g=T_{\rho,\mu}f- \Pi_{<r} (T_{\rho,\mu}f)$. We have that $\Pi_{<r} g=0$ and hence it satisfies the premise of \Cref{defn:noise_decay}. Clearly $\mutwonorm{g}{\mu}^2$ equals the approximation error. We now have 
    \begin{align}
    \label{eq: L2 tail bound}
        \mutwonorm{g}{\mu}^2&=\langle g,T_{\rho,\mu} f\rangle\\
        &=\E_{X\sim \mu}\left[g(X)\cdot \E_{R_{\rho},Y\sim \mu}[f(Y)\mid Y_{R_{\rho}}=X_{R_{\rho}}]\right]\\
        &=\E_{R_\rho,X\sim\mu}\left[\E_{Z\sim\mu}[g(Z)\mid Z_{R_\rho}=X_{R_\rho}]\cdot\E_{Y\sim\mu}[f(Y)\mid Y_{R_\rho}=X_{R_\rho}]\right]\\
        &\leq\sqrt{\Stab_{\rho,\mu}(g)\Stab_{\rho,\mu}(f)}\\
    &\leq \rho^{cr^\beta/2}\mutwonorm{g}{\mu}.
    \end{align}
    where the first equality follows from the fact that $g$ is orthogonal to the low-degree part of $T_{\rho,\mu f}$. The first inequality is obtained by applying Cauchy Schwarz and the definition of $\Stab_{\rho,\mu}$ and the last inequality is obtained from the fact that (1) $\mu$ exhibits $(c,\beta)$-decay and (2) $\Stab_{\rho,\mu}f\leq 1$. Thus, we have $\mutwonorm{g}{\mu}^2\leq \rho^{cr^{\beta}}$.

    Now, combining the two steps with a triangle inequality gives us that 
    \[
    \mutwonorm{f-\Pi_{<r}f}{\mu}^2\leq 8\NS_{\delta,\mu}(f)+2(1-\delta)^{cr^{\beta}}
    \]
    which completes the proof of \cref{eqn:noise_sensitivity_to_poly_approx}. \cref{eqn:degree_bound_noise_sensitivity} follows by solving for $r$ such that $(1-\delta)^{cr^{\beta}}\leq \varepsilon/2$.
\end{proof}

In fact, we can adapt the technique from \cite{feldman2017tight} to obtain bounds for $\ell_1$-approximation.
\begin{theorem}[Noise sensitivity implies polynomial approximation in $\ell_1$]
\label{thm:noise_stability_implies_approximation L1}
   Let $\mu$ be a distribution on $\cube{n}$ that exhibits $(c,\beta)$-decay, and let $f:\cube{n}\to\{\pm 1\}$ be any boolean function. Then, for any $0\leq \delta\leq 1$ , we have that
   \begin{equation}
       \|f-\Pi_{<r}f\|_{1,\mu}\leq \label{eqn:noise_sensitivity_to_poly_approx L1 version}
   2\NS_{\delta,\mu}(f)+(1-\delta)^{cr^{\beta}/2}.
   \end{equation}
   And as a result if $\NS_{\delta,\mu}(f)\leq \varepsilon/4$, we have for $r=\left \lceil \left(\frac{2\log(2/\varepsilon)}{c\delta}\right)^{1/\beta} \right \rceil$ that $\|f-\Pi_{<r}f\|_{1,\mu}\leq \varepsilon$.
\end{theorem}
\begin{proof}
Below, we adapt the uniform-distribution argument of \cite{feldman2017tight}. 
    Let $\rho=1-\delta$. The proof proceeds in two steps. First, we will argue that $f$ is well approximated by $T_{\rho,\mu}f$. This will be true from the definition of noise sensitivity. We will then argue that $T_{\rho,\mu}f$ is well approximated by low-degree polynomials, using the $(c,\beta)$-decay exhibited by $\mu$. 

    We now present the first step. Using the fact that $f$ is $\pm 1$-valued, we see that 
    \begin{align*}
        \norm{f-T_{\rho,\mu}f}_{1, \mu}&=
        \E_{x\sim \mu, z \sim_\rho x}
        [f(x)-f(z)]=
        \frac{1}{2}
        \E_{x\sim \mu, z \sim_\rho x}
        [(f(x)-f(z))^2]\\
       & =
        1-E_{x\sim \mu, z \sim_\rho x}
        [f(x) f(z)])
        = 
        1-
        \langle f, T_{\rho,\mu}f\rangle_{\mu} 
        = 2 \NS_{\delta,\mu}(f).
    \end{align*}
    We now again approximate $T_{\rho,\mu}f$ by low-degree polynomials in the $\ell_1$ norm. As before, let $g=T_{\rho,\mu}f- \Pi_{<r} (T_{\rho,\mu}f)$. We have that $\Pi_{<r} g=0$ and hence it satisfies the premise of \Cref{defn:noise_decay}. Clearly $\mutwonorm{g}{\mu}^2$ equals the approximation error. From Equation \ref{eq: L2 tail bound} and the inequality between norms we see 
    \begin{align*}
    \norm{g}_{1,\mu}
    \leq
        \mutwonorm{g}{\mu}
\leq \rho^{cr^\beta/2}
    \end{align*}
    Again, combining the two steps with a triangle inequality gives us that 
    \[
    \norm{f-\Pi_{<r}f}_{1, \mu}\leq 2\NS_{\delta,\mu}(f)+(1-\delta)^{cr^{\beta}/2}
    \]
    which completes the proof of \cref{eqn:noise_sensitivity_to_poly_approx L1 version}. Substitution $r=\left \lceil \left(\frac{2\log(2/\varepsilon)}{c\delta}\right)^{1/\beta} \right \rceil$ we obtain the rest of the proposition.
\end{proof}

\section{Noise Sensitivity Transfer}
\label{sec:noise_stability_transfer}
In this section, we will show that we can transfer noise sensitivity bounds over the uniform distribution to noise sensitivity bounds over the Ising distribution. 
\subsection{From Uniform to all Products}
We first state a result from \cite{blais_polynomial_2010} which states that noise stability/sensitivity over any product distribution can be bounded by the corresponding value over the uniform distribution as long as the function class is closed under restrictions and taking negations of the inputs. In fact, they show an exact identity. 
\begin{lemma}[Noise-Stability Transfer, see Proof of Theorem~3.2 \cite{blais_polynomial_2010}]
    \label{lem:bow_noise_stability_transfer}
    Let $\nu$ be a product distribution on $\cube{n}$, and let $f:\cube{n}\to\R$. For any $\bm{x},\bm{y}\in \cube{n}$, define $h_{\bm{x},\bm{y}}:\cube{n}\to\R$ as 
    \[
    h_{\bm{x},\bm{y}}(\bm{z})=f\left( \left(\frac{x_i+y_i}{2}+\frac{x_i-y_i}{2}\cdot z_i\right)_{i=1}^{n}.\right)    \]
    Then, for any $\rho\in [1/2,1]$, we have that
    \begin{equation}
        \label{eqn:stability_product to uniform.}
        \Stab_{\rho,\nu}(f)=\E_{X\sim \nu, Y\sim \nu}[\Stab_{2\rho-1,U}(h_{x,y})].
    \end{equation}
\end{lemma}

The following statement about noise sensitivity is an immediate corollary. 
\begin{corollary}[Noise-sensitivity Transfer: Uniform to Products]
\label{clry:uniform_to_products}
Let $\mathcal{F}$ be a class of boolean functions closed under restrictions and negations of input variables. Then, for any product distribution $\nu$ and $\delta\in [0,1/2]$, we have 
\begin{equation}
    \label{eqn:noise_sensitivity_uniform_to_products}
  \sup_{f\in \mathcal{F}}\NS_{\delta,\nu}(f)\leq \sup_{f\in \mathcal{F}}\NS_{2\delta,U}(f).
\end{equation}
    
\end{corollary}
\subsection{From all Products to Ising}
In the previous section, we showed that noise sensitivity bounds over the uniform distribution can be used to obtain bounds over any product distribution. In this section, we will show that noise sensitivity bounds over product distributions imply bounds over Ising distributions with interaction matrices satisfying an operator norm bound.

The key technical tool we use is the Hubbard-Stratonovich transform (see \Cref{prop:hs}). 
We are now ready to prove our theorem on the transfer of noise stability. 
\begin{theorem}[Noise-stability Transfer: Products to Ising]
\label{thm:products_to_ising}
    Let $\lambda>0$ and let $\mu=\mu_{A,h}$ be an Ising distribution with $0\preceq A\preceq \lambda$\footnote{For any interaction matrix with $\opnorm{A}\leq \lambda$, replacing $A$ by $A+\lambda I$ does not change the Ising distribution and gives a positive-semi-definite interaction matrix with operator norm at most $2\lambda$.}. Let $Z$ be the random variable defined in \Cref{prop:hs} Then, for any function $f:\cube{n}\to\R$ and $\rho\in [0,1]$, we have that
    \begin{equation}
        \label{eqn:noise_stability_transfer_ising}
        \Stab_{\rho,\mu}(f)\leq \E_{Z}[\Stab_{\rho,\mu_{0,Z}}(f)]\leq \Stab_{1-C_{\lambda}(1-\rho),\mu}(f).
    \end{equation}
    where $C_\lambda=2\Phi(-\sqrt{\lambda})$.
\end{theorem}
\begin{proof}
    Recall the definition of $\Stab_{\rho,\mu}(f)$. We have that 
    \[
    \Stab_{\rho,\mu}(f)=\E_{R_{\rho},X\sim \mu}\left[(\E_{Y\sim \mu}[f(Y)\mid Y_{R_{\rho}}=X_{R_{\rho}}])^2\right].
    \]
    An important idea in the proof is that revealing more information about $X$ can only increase the expected noise stability. Formally, for any $W$ that is a (possibly randomized) function of $X\sim \mu$, we have that 
    \begin{equation}  
    \label{eqn:stabilty_increases_with_information}
    \Stab_{\rho,\mu}(f)\leq \E_{W}[\Stab_{\rho,\mu(.|W)}(f)]
    \end{equation}
    where $\mu(.|W)$ denotes the conditional distribution of $X$ given $W$.  To see this, observe that
    \begin{align*}
        \Stab_{\rho,\mu}(f)&=\E_{R_\rho,X\sim\mu}\left[\left(\E_{W\mid X_{R_\rho}}\left[\E_{Y\sim\mu(\cdot\mid W)}[f(Y)\mid Y_{R_\rho}=X_{R_\rho}]\right]\right)^2\right]\\
        &\leq \E_{R_\rho,X,W}\left[\left(\E_{Y\sim\mu(\cdot\mid W)}[f(Y)\mid Y_{R_\rho}=X_{R_\rho}]\right)^2\right]\\
        &=\E_W[\Stab_{\rho,\mu(.|W)}(f)]
    \end{align*}
    where the inequality follows from Cauchy-Schwarz.
    Taking $W=Z$ (from \Cref{prop:hs}) proves the first inequality. 

    We now prove the second inequality. The plan is to show that $Z$ can be generated by independently revealing each coordinate of $X$ with probability $1-C_{\lambda}$, and then apply a randomized function to the revealed coordinates. We will then apply \Cref{eqn:stabilty_increases_with_information}, to show that revealing these coordinates increases the noise stability at least as much as revealing $Z$, which would imply the inequality. 

    Define the random variable \[
    Y=\sqrt{\lambda}X+H\] where $X\sim \mu$ and $H\sim N(0,I)$ is independent of $X$. Recall that $Z=AX+A^{1/2}G+\bm{h}$ where $G\sim N(0,I)$ is independent of $X$. Thus, $Z$ is identically distributed to a randomized function (independent of $X$) of $Y$:
    \[
    Z=\bm{h}+\frac{A}{\sqrt{\lambda}}Y+G'
    \]
    where $G'\sim N(0,A-A^2/\lambda)$. We now argue that $Y$ can be be generated by independently revealing each coordinate of $X$ with probability $1-C_{\lambda}$ for $C_\lambda$ that we will choose.  Observe that conditioned on $X_i=x_i$, the density of $Y_i$ follows
    \[
    p_{x_i}(y)=\varphi(y-\sqrt{\lambda}x_i)
    \]
    where $\varphi$ is the one dimensional Gaussian pdf. Let $C_\lambda$ be the common mass between $p_{+1}$ and $p_{-1}$. That is, 
    \[
    C_{\lambda}=\int_{\R}\min \{p_{+1}(y),p_{-1}(y)\}dy=2\Phi(-\sqrt{\lambda}).
    \]
    where $\Phi$ is the Gaussian cdf.  Thus, we can rewrite $p_{x_i}$ as a mixture 
    \[
    p_{x_i}=C_{\lambda}q+(1-C_\lambda)q_{x_i}
    \]
    where $q$ is a density function independent of $x_i$. 

    Define $S=R_{1-C_{\lambda}}$. The above equation tells us that we can generate $Y$ (given $X_{S})$ by sampling each $Y_i$ from $q$ if $i\notin S$ and from $q_{X_i}$ otherwise. 

    Since $Z$ is a randomized function of $Y$, which implies that the distributions $\mu(.|Z,S,X_s)$ is identical to $\mu(.|S,X_S)$. We now apply \Cref{eqn:stabilty_increases_with_information} to $\mu(.|Z)$ by also revealing $(S,X_S)$. We have 
    \begin{align*}
        \E_{Z}[\Stab_{\rho,\mu_{0,Z}}]&\leq \E_{Z,S,X_S}[\Stab_{\rho,\mu(.|Z,S,X_S)}(f)]\\
         &=\E_{S,X_S}[\Stab_{\rho,\mu(.|S,X_S)}(f)]\\
         &=\Stab_{1-C_\lambda(1-\rho),\mu}(f).
    \end{align*}
    where first equality follows from the fact that $\mu(.|Z,S,X_s)=\mu(.|S,X_S)$. The last equality follows because revealing $X_S$ and then retaining each coordinate with probability $\rho$ (this is what the stability term inside the expectation does) is equivalent to retaining each coordinate with probability $1-C_{\lambda}(1-\rho)$. This completes the proof.
\end{proof}

We can now state our result on transfer of noise sensitivity bounds from the uniform distribution to Ising models having interaction matrices with bounded operator norm.

\begin{corollary}[Noise Sensitivity Transfer: Uniform to Ising]
\label{clry:noise_sensitivity_uniform_to_ising}
    Let $\mathcal{F}$ be a class of Boolean functions closed under restriction and negations of input variables. Let $\mu=\mu_{A,h}$ be an Ising distribution satisfying $\opnorm{A}\leq \lambda$. Then, for every $\delta\in [0,C_\lambda/2]$, we have that 
    \[
    \sup_{f\in \mathcal{F}}\NS_{\delta,\mu}(f)\leq \sup_{f\in \mathcal{F}}\NS_{2\delta/C_\lambda,U}(f)
    \]
    where $C_{\lambda}=2\Phi(-\sqrt{2\lambda})$.
\end{corollary}
\begin{proof}
    Use the fact that $\NS_{\delta,\mu}(f)=(1-\Stab_{1-\delta,\mu}(f))/2$ along with \Cref{clry:uniform_to_products} and \Cref{thm:products_to_ising}.
\end{proof}

\section{Polynomial Approximations for $\ac$}
\label{Sec: polynomial approximation for AC0}
The goal of this section is to obtain low-degree approximators for $\ac$ over Ising models satisfying $(c,\beta)$-decay and whose interaction matrix has bounded operator norm. Crucially, we aim to construction approximators that have polylogarithmic dependence on $1/\varepsilon$. The guarantees of \Cref{thm:noise_stability_implies_approximation} are too weak for this purpose as we do not have strong enough bounds for the noise sensitivity of $\ac$. We employ a more hands-on approach, by exploiting the fact that $\ac$ has extremely strong Fourier tail decay over the uniform distribution, and then using the transfer machinery we developed in \Cref{sec:noise_stability_transfer}.
\subsection{Additional Preliminaries}
We recall some basic facts from Fourier analysis over product distributions. For a product distribution $\nu$, the orthogonal Fourier basis is given the functions $\chi_S^{\nu}:\cube{n}\to \R$ defined as $\chi_S^{\nu}(\bm{x})=\prod_{i\in S}\frac{x_i-\nu_i}{\sqrt{1-\nu_i^2}}$. For a function $f:\cube{n}$, let $\{\widehat{f}_\nu(S)\}_{S\subseteq [n]}$ denote the Fourier coefficients of $f$ with respect to the distribution $\nu$. 
\begin{fact}[Expression for Noise Stability over Product Distributions]
\label{fact:expressions_of_noise_stability}
    Let $\nu$ be a product distribution on $\cube{n}$. Then, we have that for all $\rho\in [0,1]$. We then have that for any function $f:\cube{n}\to \R$, it holds that
    \[
   \Stab_{\rho,\nu}(f)=\sum_{S\subseteq [n]}\rho^{|S|}\widehat{f}_\nu(S)^2.
    \]
\end{fact}
\subsection{Noise Stability at $\rho>1$}
A convenient technical tool for this section is to consider the extension of the noise stability function for $\rho>1$. For our purposes, it suffices to study this over product distributions. Recall from \Cref{fact:expressions_of_noise_stability}, that for any product distribution $\nu$, it holds that and $\rho\in [0,1]$
\[
\Stab_{\rho,\nu}(f)=\sum_{S\subseteq [n]}\rho^{|S|}\widehat{f}_\nu(S)^2.
\]

The extension to $\rho>1$ is obtained by treating $\Stab_{\rho,\nu}(f)$ as a polynomial of $\rho$ and extending the domain. 
The motivation for considering noise stability at $\rho>1$ is that bounds on this quantity are essentially equivalent to exponentially small polynomial approximation error for product distributions. Indeed, we have that for any product distribution $\nu$ and $\rho>1$:
\[
\dist_{\nu}(f,\mathcal{P}_{r})^2=\sum_{|S|\geq r}\widehat{f}_\nu(S)^2\leq\rho^{-r}\sum_{S\subseteq [n]}\rho^{|S|}\widehat{f}_\nu(S)^2\leq \rho^{-r}\Stab_{\rho,\nu}(f).
\]

The converse also holds, upto constant factor losses. We now use this connection to get bounds for noise stability of $\ac$ over the uniform distribution at $\rho>1$. We use the following statement on Fourier decay of $\ac$ from \cite{tal_ac0}.
\begin{theorem}[Fourier Decay of $\ac$, Theorem~1.1 from \cite{tal_ac0}]
\label{thm:ac0_decay}
Let $f$ be an $\ac$ circuit with size $s$ and depth $d$. For $B=2^{O(d)}(\log s)^{d-1}$, it holds that for any integer $k$, we have that
\[
\sum_{|S|\geq k}\widehat{f}_U(S)^2\leq 2e^{-k/B}.
\]
\end{theorem}
From this, we have the following bound on noise stability at $\rho>1$ for $\ac$ over the uniform distribution. 
\begin{corollary}
    \label{clry:noise_stability_more_than_one_ac0}
    Let $f$ be an $\ac$ circuit with size $s$ and depth $d$. Taking $B=2^{O(d}(\log s)^{d-1}$ from \Cref{thm:ac0_decay}, we have that for $\rho=1+1/(2B)$, it holds that 
    \[
    \Stab_{\rho,U}(f)\leq 3. 
    \]
\end{corollary}
\begin{proof}
    We have that 
    \begin{align*}
        \Stab_{\rho,U}(f)&=1+(\rho-1)\sum_{k=1}^{n}\rho^{k-1}\sum_{|S|\geq r}\widehat{f}_U(S)^2\\
        &\leq 1+2(\rho-1)\sum_{k=1}^{n}\rho^{k-1}e^{-k/B}\\
        &= 1+2(\rho-1)/(e^{1/B}-\rho)\\
        &\leq 3
    \end{align*}
    where the first equality follows from Parseval's theorem, the second inequality follows from \Cref{thm:ac0_decay}. The penultimate equality follows from computing the geometric sum and the last follows from our choice of $\rho$.
\end{proof}

This bound also can be naturally extended to all product distributions by the identity in \Cref{lem:bow_noise_stability_transfer}.
\begin{theorem}
\label{thm:product_ac0_stability_greater_than_one}
      Let $f$ be an $\ac$ circuit with size $s$ and depth $d$. Taking $B=2^{O(d}(\log s)^{d-1}$ from \Cref{thm:ac0_decay}, we have that for $\rho=1+1/(4B)$, it holds that 
    \[
    \Stab_{\rho,\nu}(f)\leq 3. 
    \]
\end{theorem}
\begin{proof}
    Recall the identity from \Cref{lem:bow_noise_stability_transfer}. 
    \[  
        \Stab_{\rho,\nu}(f)=\E_{X\sim \nu, Y\sim \nu}[\Stab_{2\rho-1,U}(h_{x,y})].
   \]
    Both sides of that equation are polynomials of degree at most $n$ in $\rho$, and equality over the whole interval $[1/2,1]$ implies that the two polynomials must be equal everywhere. Also, note that $h_{x,y}$ is also an $\ac$ circuit of size $s$ and depth $d$. Thus, we can use \Cref{clry:noise_stability_more_than_one_ac0} to complete the proof.
\end{proof}

We also use the following elementary lemma.
\begin{lemma}
    \label{lem:noise_stability_cauchy_schwarz}
    Let $f$ and $g$ be two functions on $\cube{n}$ and let $\nu$ be a product distribution. Then, we have that 
    \[
    \langle f,g\rangle_{\nu}\leq \sqrt{\Stab_{\rho}(f)\cdot \Stab_{1/\rho}(g)}
    \]
     for all $\rho>0$.
\end{lemma}
\begin{proof}
An elementary proof of the lemma using Cauchy Schwarz is the following: 
\begin{align*}
    \langle f,g\rangle_{\nu}&=\sum_{S\subseteq [n]} \rho^{|S|/2}\widehat{f}_\nu(S)\cdot\rho^{-|S|/2}\widehat{g}_{\nu}(S)\\
                            &\leq \sqrt{\sum_{S\subseteq [n]}\rho^{|S|}\widehat{f}_\nu(S)^2}\cdot \sqrt{\sum_{S\subseteq [n]}\rho^{-|S|}\widehat{g}_\nu(S)^2}\\
                &\leq \sqrt{\Stab_{\rho}(f)\cdot \Stab_{1/\rho}(g)}.
\end{align*}
\end{proof}
\subsection{Approximating $\ac$ over Ising Distributions}
We are now ready to state and prove our main theorem on approximating $\ac$ over Ising distributions. 

\begin{theorem}
\label{thm: Ising approximation}
    Let $f$ be an $\ac$ circuit with size $s$ and depth $d$. Let $\mu=\mu_{A,h}$ be an Ising distribution satisfying $\opnorm{A}\leq \lambda$ and exhibiting $(c,\beta)$-decay. Then, it holds that
\[
\deg(f,\mu,\varepsilon)\leq \left(\frac{2^{O(d)}(\log s)^{d-1}\log(1/\varepsilon)}{c\cdot C_\lambda}\right)^{1/\beta}.
\]
where $C_\lambda=2\Phi(-\sqrt{2\lambda})$.
\end{theorem}
\begin{proof}
    For degree parameter $r$ to be chosen later, define $g=f-\Pi_{<r}f$. Let $Z$ be the random variable defined in \Cref{prop:hs} and $\rho=1+1/(4B)$ for $B=2^{O(d)}(\log s)^{d-1}$ from \Cref{thm:product_ac0_stability_greater_than_one}. We have that 
    \begin{align*}
        \mutwonorm{g}{\mu}^4&=\langle f,g\rangle_{\mu}^2\\
        &=\left(\E_Z[\langle f,g\rangle_{\mu_{0,Z}}]\right)^2\\
        &\leq \E_Z\left[\Stab_{\rho,\mu_{0,Z}}(f)\cdot \Stab_{1/\rho,\mu_{0,Z}}(g)\right]\\
        &\leq 3\E_{Z}\left[\Stab_{1/\rho}\mu_{0,Z}(g)\right]
    \end{align*}
    where the second equality followed from \Cref{prop:hs}, the penultimate inequality followed from Cauchy-Schwarz and \Cref{lem:noise_stability_cauchy_schwarz} and the last inequality followed from \Cref{thm:product_ac0_stability_greater_than_one} and our choice of $\rho=1+(1/4B)$.

    We now apply \Cref{thm:products_to_ising}, the fact that $\Pi_<r g=0$ and the $(c,\beta)$-decay exhibited by $\mu$ to obtain
    \begin{align*}
            \mutwonorm{g}{\mu}^4&\leq 3\Stab_{1-C_{\lambda}(1-1/\rho),\mu}(g)\\
            &\leq 3\left(1-C_\lambda(1-1/\rho)\right)^{cr^{\beta}}\mutwonorm{g}{\mu}^2\\
            &\leq 3\left(1-\frac{C_\lambda}{4B+1}\right)^{cr^{\beta}}\mutwonorm{g}{\mu}^2\\
            &\leq 3\exp\left(-\frac{C_\lambda\cdot cr^{\beta}}{5B}\right)\mutwonorm{g}{\mu}^2
    \end{align*}
    Solving for $r$ gives the desired bound.
\end{proof}

\section{Applications to Concrete Hypothesis Classes.}
\label{sec: concrete applications}
Given the previous results, we now show bounds on Ising noise sensitivities of the classes using the known uniform bounds given in Table \ref{tab:approx-degree}.

 \begin{table}[H]
 \centering
 \renewcommand{\arraystretch}{1.4}
 \begin{tabular}{|c|c|>{\centering\arraybackslash}p{4cm}|}
 \hline
 Function Class & $\text{NS}_{\delta}[f]$ over Uniform & $\text{NS}_{\delta, \mu}[f]$ over Ising models \\
 \hline
 Linear Classifiers & $O(\sqrt{\delta})$ \cite{peres2004noisestabilityweightedmajority} & $O_{\alpha, \lambda}(\sqrt{\delta})$  \\
 \hline
 Intersections of $m$ Linear Classifiers & $O(\sqrt{\log(m+1)\delta)}$ \cite{kane_average_2014} & $O_{\alpha, \lambda}(\sqrt{\log(m+1)\delta)}$ \\
 \hline
 Functions of $m$ Linear Classifiers & $O(m\sqrt{\delta})$ \cite{ klivans2004learning}  & $O_{\alpha, \lambda}(m\sqrt{\delta})$  \\
 \hline
 Degree-k Polynomial Thresholds& $O(\sqrt{k\delta})$ \cite{openai2026gotsmanlinial}  & $O_{\alpha, \lambda}(\sqrt{k\delta})$ \\
 \hline
 Monotone Functions & $O(\sqrt{n} \delta)$ \cite{DBLP:journals/jacm/BshoutyT96} & $O_{\alpha, \lambda}(\sqrt{n} \delta)$ \\
 \hline
 \end{tabular}
 \caption{Comparison of noise sensitivity bounds under uniform distribution and the noise sensitivity bounds over Ising models satisfying $\alpha$-pinned Poincar\'e condition and any of the other high-temperature conditions that are sufficient to obtaining $(c,1)$-decay. See Section \ref{sec: concrete applications} for more details. }
 \label{tab:approx-degree}
\end{table}

\begin{proposition}
\label{prop: low-degree approximation}
Suppose that an Ising model 
$\mu_{A,\bm{h}}$ over $\{\pm 1\}^n$ satisfies Assumption
\hyperref[ass:loc-pi]{$(\mathrm{PI}_{\alpha})$} and  
any of \hyperref[ass:loc-bd]{$(\mathrm{BD}_{\Delta})$}, \hyperref[ass:loc-bw]{$(\mathrm{BW}_{B})$}, or \hyperref[ass:loc-mf]{$(\mathrm{MF}_{\omega,L})$}. Then, with $\lambda=\|A\|_{\mathsf{op}}$,
\begin{itemize}
    \item If $f$ is a linear classifier, then $\NS_{\delta,\mu}[f]\leq O(\sqrt{\beta})$ and as a result for some $r = \Theta_{\alpha,\lambda}(\frac{1}{\varepsilon^2} \log \frac{1}{\varepsilon})$ we have $\norm{f-\Pi_{<r}f}_{1,\mu}^2\leq \varepsilon$.
    \item If $f$ is an Intersections of $m$ Linear Classifiers then $\NS_{\delta,\mu}[f]\leq O_{\alpha}(\sqrt{\log(m+1)\delta)}$
    and as a result for some $r = \Theta_{\alpha,\lambda}(\frac{\log(m+1)}{\varepsilon^2} \log \frac{1}{\varepsilon})$ we have $\mutwonorm{f-\Pi_{<r}f}{\mu}^2\leq \varepsilon$.
    \item If $f$ is a Boolean function of $m$ Linear Classifiers then $\NS_{\delta, \mu}[f]\leq O_{\alpha,\lambda}(m\sqrt{\delta})$ and as a result for some $r = \Theta_{\alpha}(\frac{m^2}{\varepsilon^2} \log \frac{1}{\varepsilon})$ we have $\norm{f-\Pi_{<r}f}_{1,\mu}^2\leq \varepsilon$.
    \item If $f$ is a Degree-k Polynomial Threshold function then $\NS_{\delta,\mu}[f]\leq O_{\lambda}(\sqrt{k \delta})$ and as a result for some $r = \Theta_{\alpha,\lambda}(\frac{k}{\varepsilon^2}\cdot\log(1/\varepsilon))$ we have $\norm{f-\Pi_{<r}f}_{1,\mu}^2\leq \varepsilon$.
    \item If $f$ is a Monotone function then $\NS_{\delta,\mu}[f]\leq O_{\alpha, \lambda}(\sqrt{n} \delta)$ and as a result for some $r = \Theta_{\alpha,\lambda}(\frac{\sqrt{n}}{\varepsilon} \log \frac{1}{\varepsilon})$ we have $\norm{f-\Pi_{<r}f}_{1,\mu}^2\leq \varepsilon$.
\end{itemize}
In each case, the implicit constant depends only on the additional parameters of  \hyperref[ass:loc-bd]{$(\mathrm{BD}_{\Delta})$},  \hyperref[ass:loc-bw]{$(\mathrm{BW}_{B})$}, or \hyperref[ass:loc-mf]{$(\mathrm{MF}_{\omega,L})$}.
\end{proposition}
\begin{proof}
Combing the uniform-distribution noise sensitivity bounds of \cite{peres2004noisestabilityweightedmajority,kane_average_2014,klivans2004learning,openai2026gotsmanlinial,DBLP:journals/jacm/BshoutyT96} listed in the middle column of Table \ref{tab:runtime} with the Corollary \ref{clry:noise_sensitivity_uniform_to_ising} we obtain the bounds on $\NS_{\delta, \mu}[f]$. The only subtlety is with the class of Monotone Functions, as this class is not closed under negations. However, we can still apply Lemma \ref{lem:bow_noise_stability_transfer} directly to this class by observing that the function $h_{x,y}(z)$ will be monotone after negating a subset of variables. Afterwards, we apply Theorem \ref{thm:products_to_ising}.  Thus, for monotone functions, $\NS_{\delta,\mu}[f]\leq O_{\alpha, B}(\sqrt{n} \delta)$.

Combining these noise sensitivity bounds with Theorems \ref{thm:loc-ising} and \ref{thm:noise_stability_implies_approximation L1} we obtained the claimed low-degree approximation bounds.
\end{proof}
We also use the following theorem on the performance of $\ell_1$ regression as an agnostic learner. 
\begin{theorem}[$\ell_1$ approximation implies agnostic learning, Theorem~5 in \cite{kalai2008agnostically}]
\label{thm:agnostic_learning_l1}
Let $D$ be a labeled distribution over $\cube{n}\times \{\pm 1\}$ with marginal $D_{X}$ and let $\mathcal{F}$ be a function class. Suppose for each $f\in \mathcal{F}$ and any $\varepsilon>0$, there exists a polynomial $p$ of degree $\ell(\varepsilon)$ such that $\E_{X\sim D_X}[|f(X)-p(X)|]\leq \varepsilon$. Then, there is an algorithm that achieves agnostic learning of the class $\mathcal{F}$ using $n^{O(\ell(\varepsilon))}$ run-time and samples. In other words, given $N=n^{O(\ell(\varepsilon))}$ i.i.d samples from $\mathcal{D}$ labeled by $f$, runs in time $\poly(N,n)$ and with probability $0.9$ outputs a classifier $h$ for which 
    \begin{equation}
    \label{eq: agnostic guarantee}
    \Pr_{(X,y)\sim \mathcal{D}}
    [h(X)\neq y]
    \leq \mathsf{opt}+ \varepsilon
    \end{equation}
    where $\mathsf{opt}\coloneq \min_{f\in \mathcal{F}}\E_{(X,y)\sim D}[f(X)\neq y]$.
\end{theorem}
Combining the theorem above with our results on polynomial approximation, we obtain the following:
\begin{theorem}
Suppose that an Ising model 
$\mu_{A,\bm{h}}$ over $\{\pm 1\}^n$ satisfies Assumption
\hyperref[ass:loc-pi]{$(\mathrm{PI}_{\alpha})$} and  
any of \hyperref[ass:loc-bd]{$(\mathrm{BD}_{\Delta})$}, \hyperref[ass:loc-bw]{$(\mathrm{BW}_{B})$}, or \hyperref[ass:loc-mf]{$(\mathrm{MF}_{\omega,L})$}. Then, with $\lambda=\|A\|_{\mathsf{op}}$,
\begin{itemize}
        \item The class of Linear Classifiers can be agnostically learned using $n^{\widetilde{O}_{\alpha, \lambda}(1/\varepsilon^2)}$ run-time and samples.
       \item The class of Intersections of $m$ Linear Classifiers can be agnostically learned using $n^{\widetilde{O}_{\alpha, \lambda}(\log(m+1)/\varepsilon^2)}$ run-time and samples.
      \item The class of Functions of $m$ Linear Classifiers can be agnostically learned using $n^{\widetilde{O}_{\alpha, \lambda}(m^2/\varepsilon^2)}$ run-time and samples.
      \item The class of Degree-$k$ Polynomial Threshold Functions can be agnostically learned using $n^{\widetilde{O}_{\alpha, \lambda}(k/\varepsilon^2)}$ run-time and samples.
      \item The class of Monotone Functions can be agnostically learned using $n^{\widetilde{O}_{\alpha, \lambda}(\sqrt{n}/\varepsilon)}$ run-time and samples.
      \item The class of Polysize Depth-$d$ Circuits can be agnostically learned using $n^{\widetilde{O}_{\alpha, \lambda}(\log^{O(d)}(n/\varepsilon))}$ run-time and samples.
\end{itemize}
In all cases, the explicit constants depend only on the other parameters in \hyperref[ass:loc-bd]{$(\mathrm{BD}_{\Delta})$}, \hyperref[ass:loc-bw]{$(\mathrm{BW}_{B})$}, or \hyperref[ass:loc-mf]{$(\mathrm{MF}_{\omega,L})$}. 
\end{theorem}
\begin{proof}
    All the agnostic learning results except for the last one follow directly by combining Proposition \ref{prop: low-degree approximation} with Theorem \ref{thm:agnostic_learning_l1}. The claim on agnostic learning of Polysize Depth-$d$ Circuits follows directly by combining Theorem \ref{thm: Ising approximation}, with Theorem \ref{thm:loc-ising} and Theorem \ref{thm:agnostic_learning_l1}.
\end{proof}

\section{Improved Decay and Eigenstructure at High-Temperature}
\label{sec:improved}

In this section, we finally return to the more technical proofs that establish $(c,1)$-decay for graphical models under mild, but slightly stronger conditions than pinned Poincar\'e. We also show that these results imply strong quantitative bounds on the eigenvalues for the Glauber dynamics that match the hypercube up to absolute constants for \emph{all} eigenvalues.

\subsection{Bounded Degree}
\label{sec:bounded-degree}
We now turn to showing that any graphical model satisfying bounded degree admits optimal decay.

\begin{proof}[Proof of \Cref{thm:loc-bd}]
    If $\Delta=0$, then the distribution is product and the claim follows from the product-measure bound. If $\alpha=1$, Corollary~\ref{cor:far-decay} gives $\varphi(s)\geq r$ for $s>0$, so the claim follows as well. Henceforth, assume $\Delta\geq1$ and $0<\alpha<1$.
    Recall that our goal is to handle the singularity in Corollary~\ref{cor:log-decay} near $0$ by deriving a stronger bound than can be obtained just with the Poincar\'e inequality as in Corollary~\ref{lem:long-times}. We claim that the following bound holds for all $s\geq 0$:
    \begin{equation}
    \label{eq:improved-singularity-gm}
        K(s)\leq \frac{2(e^s-1)\Delta}{\alpha}\cdot \left(\frac{1}{\alpha}-1\right)\cdot G(s).
    \end{equation}

    To prove Equation~\ref{eq:improved-singularity-gm}, the essential idea is that a masked term $i$ in $[n]\setminus V_s$ can only contribute covariances in $K(s)$ when it is not \emph{isolated} in the masked set; if all its graph neighbors are revealed in $V_s$, then $X_i$ is conditionally independent of the rest of the masked variables in the expression, and therefore $g_{V_{s}\cup i}$ is conditionally independent of all other terms. After paying the Poincar\'e inequality on this \emph{random} set, we will be able to decouple terms to get the extra factor of $e^s-1$ to remove the singularity at zero.

    To make this precise, condition on the unmasked set $V_s=T$ at time $s$ and let $U_T$ denote the set of nonisolated coordinates in $[n]\setminus T$. Then conditional independence of elements with the Poincar\'e inequality means
    \begin{align*}
        \sum_{i\neq j\not\in T} \mathbb{E}_{\mu}\left[\Cov(g_{V_s\cup \{i\}},g_{V_s\cup \{j\}}\vert X_T)\right]&=\sum_{i\neq j\in U_T} \mathbb{E}_{\mu}\left[\Cov(g_{T\cup \{i\}},g_{T\cup \{j\}}\vert X_T)\right]\\
        &=\mathbb{E}_{\mu}\left[\Var\left(\sum_{i\in U_T}g_{T\cup\{i\}}\big\vert X_T\right)\right]-\sum_{i\in U_T}\mathbb{E}_{\mu}\left[\Var\left(g_{T\cup\{i\}}\vert X_T\right)\right]\\
        &\leq \frac{1}{\alpha}\cdot \mathbb{E}_{\mu}\left[\sum_{i\in U_T}\Var\left(g_{T\cup\{i\}}\big\vert X_{-i}\right)\right]-\sum_{i\in U_T}\mathbb{E}_{\mu}\left[\Var\left(g_{T\cup\{i\}}\vert X_T\right)\right]\\
        &\leq \left(\frac{1}{\alpha}-1\right)\cdot \sum_{i\in U_T}\mathbb{E}_{\mu}\left[\Var\left(g_{T\cup\{i\}}\vert X_T\right)\right]\\
        &\leq \left(\frac{1}{\alpha}-1\right)\cdot \sum_{i\notin T}\sum_{j\sim i}\bm{1}\{j\not\in T\}\mathbb{E}_{\mu}\left[\Var\left(g_{T\cup\{i\}}\vert X_T\right)\right].
    \end{align*}
    Here, the second inequality is the Law of Total Variance, while the last is a simple upper bound using the fact that if $i\in U_T$ then at least one neighbor must be outside $T$. In particular, averaging over $V_s$ implies that
    \begin{equation*}
        K(s)\leq \left(\frac{1}{\alpha}-1\right)\sum_{i\in [n]}\sum_{j\sim i}\mathbb{E}\left[\bm{1}\{i,j\not\in V_s\}\Var\left(g_{V_s\cup\{i\}}\vert X_{V_s}\right)\right].
    \end{equation*}
    The issue with this expression is that the event that $j\not\in V_s$ is correlated with the conditional variance term, so we need to decouple them using Lemma~\ref{lem:two-poincare}:
    \begin{align*}
        K(s)&\leq \left(\frac{1}{\alpha}-1\right)\cdot\frac{1}{\alpha}\cdot \sum_{i\in [n]}\sum_{j\sim i}\mathbb{E}\left[\bm{1}\{i,j\not\in V_s\}\left(\Var\left(g_{V_s\cup\{i,j\}}\vert X_{V_s\cup \{j\}}\right)+\Var\left(g_{V_s\cup\{i,j\}}\vert X_{V_s\cup \{i\}}\right)\right)\right]\\
        &=2\cdot \left(\frac{1}{\alpha}-1\right)\cdot\frac{1}{\alpha}\cdot \sum_{i\in [n]}\sum_{j\sim i}\mathbb{E}\left[\bm{1}\{i,j\not\in V_s\}\Var\left(g_{V_s\cup\{i,j\}}\vert X_{V_s\cup \{j\}}\right)\right]
    \end{align*}
    Note that we can apply this result even if $i$ or $j\in V_s$, since then the indicator makes the term zero anyway. At this point, the same change-of-measure argument as in Corollary~\ref{cor:log-decay} on pairs $(j,T)$ with $j\not\in T$ and subsets $S$ containing $j$ yields
    \begin{align*}
        K(s)&\leq (e^s-1)\cdot2\cdot \left(\frac{1}{\alpha}-1\right)\cdot\frac{1}{\alpha}\cdot \sum_{i\in [n]}\sum_{j\sim i}\mathbb{E}\left[\bm{1}\{i\not\in V_s,\ j\in V_s\}\Var\left(g_{V_s\cup\{i\}}\vert X_{V_s}\right)\right]\\
        &\leq (e^s-1)\cdot \underbrace{2\cdot \left(\frac{1}{\alpha}-1\right)\cdot\frac{1}{\alpha}\cdot \Delta}_{:=\gamma}\cdot G(s).
    \end{align*}

    We are now ready to complete the proof. We will lower bound $\varphi(s)$ uniformly over $s$: 
    first, by Corollary~\ref{cor:far-decay}, for any values $t\leq s$, we have 
    \begin{equation*}
        \varphi(s)\geq r\cdot (1-\exp(-s))^{1/\alpha-1}\geq r\cdot (1-\exp(-t))^{1/\alpha-1}.
    \end{equation*}
    On the other hand, by Corollary~\ref{cor:log-decay},  if $t\geq s$ we can integrate $(\log \varphi)'(u)$ between them with the above bound to give
    \begin{equation*}
        \varphi(s)\geq \exp\left(-\gamma(t-s)\right)\cdot \varphi(t)\geq \exp(-\gamma t)\cdot \varphi(t)\geq \exp(-\gamma t)\cdot r\cdot (1-\exp(-t))^{1/\alpha-1}.
    \end{equation*}

    We can now choose $t$ to ensure that the minimum of these lower bounds remains large for any $s$. Since the latter is pointwise smaller, it suffices to approximately optimize this bound, which is done via the simple calculation of Lemma~\ref{lem:lb-optimizer} with the above value of $\gamma$ and $\eta=(1/\alpha-1)$ (whose simple proof we defer to the Appendix):

    \begin{lemma}
\label{lem:lb-optimizer}
    For any $\gamma,\eta>0$, 
    \begin{equation*}
        \max_{t\geq 0} \left\{\exp(-\gamma t)(1-\exp(-t))^{\eta}\right\}=\left(\frac{\gamma}{\gamma+\eta}\right)^{\gamma}\left(\frac{\eta}{\gamma+\eta}\right)^{\eta}.
    \end{equation*}
    This value is always at least
    \begin{equation*}
        \left(\frac{\min(\gamma,\eta)}{e(\gamma+\eta)}\right)^{\min(\gamma,\eta)}.
    \end{equation*}
\end{lemma}
    
    This yields for all $s\geq 0$:
    \begin{equation*}
        \varphi(s)\geq c_{\alpha,\Delta}r:=\left(\frac{1}{e(1+2\Delta/\alpha)}\right)^{1/\alpha-1}
    \end{equation*}
    and so 
    \begin{equation*}
        \Lambda(s)\leq \exp(-r\cdot s\cdot c_{\alpha,\Delta})\cdot \|g\|_{2,\mu}^2=\rho^{c_{\alpha,\Delta}\cdot r}\|g\|_{2,\mu}^2
    \end{equation*}
    as claimed.
\end{proof}

\subsection{High-Temperature Ising Models}
\label{sec:ht-ising}
To prove Theorem~\ref{thm:loc-ising}, we again need to find a better bound for the stability decay near the singularity at zero. In both cases, our approach will be to bound the relevant sums of covariances appearing in $K(s)$ by a sum of variances like those appearing in $G(s)$ with \emph{adaptively} chosen weights depending on the realization of $V_s$. The point of the following \Cref{prop:cov-quadform} will be to \emph{deterministically charge} the weight of each such variance of a hidden variable $X_i$ by weights coming from interactions with other hidden variables. When $s$ is small, we expect that few of the interactions with $X_i$ will remain among the hidden set, and in the proof of \Cref{thm:loc-ising}, we will be able to formally decouple these weights by paying the change of measure argument to get back the $e^{s}-1$ that we need to remove the singularity at $s=0$.

We will carry out this deterministic upper bound using the Hubbard-Stratanovich transform of \Cref{prop:hs} to reduce the $K(s)$ term to elementary \emph{Gaussian} analysis. The key intuition is that, fixing the function $g$ and the visible coordinates $X_{V_s}$, the resulting quantity is a fixed quadratic form on the covariance matrix of the pinned Ising model. We will obtain a weighted upper bound resembling $G(s)$ using the Hubbard-Stratanovich transform on a suitable shift of the interactions $A$, which enables us to write this Ising model as a mixture of product distributions induced by an explicit field that is obtained via a Gaussian observation of a sample $X_{[n]\setminus V_s}$. The law of total covariance then lets us reduce the problem to analyzing the covariance of the \emph{conditional means} of the spins over this explicit, random choice of field. At this point, we can appeal to Poincar\'e inequalities both for these spins under Assumption
\hyperref[ass:loc-pi]{$(\mathrm{PI}_{\alpha})$} as well as for Gaussian vectors to bound terms. At a high-level, the essential point is that the Hubbard-Stratanovich representation reveals a direct dependence on the interaction matrix that we will later flexibly bound to recover the $e^s-1$ factor to remove the singularity. 

\begin{proposition}
\label{prop:cov-quadform}
    Suppose that an Ising model $\mu_{A,\bm{h}}$ over $\{\pm 1\}^n$ with zero-diagonal interaction matrix $A$ and external fields $\bm{h}$ satisfies Assumption \hyperref[ass:loc-pi]{$(\mathrm{PI}_{\alpha})$}. Moreover, let $D=\mathrm{diag}(\bm{d})\succeq 0$ be a diagonal matrix such that $-D\preceq A\preceq D$, and set $M=\max_i d_i$. Then for any vector $\bm{c}\in \mathbb{R}^n$,
    \begin{equation*}
        \sum_{i\neq j} c_ic_j\Cov(X_i,X_j)\leq \left(2+\frac{16Me^{M/2}}{\alpha}\right)
\sum_i d_i c_i^2\operatorname{Var}_\mu(X_i).
    \end{equation*}
\end{proposition}
\begin{proof}
    First, we will reduce the problem to variances of conditional means of $X$ using the Hubbard-Stratanovich transform of \Cref{prop:hs} with the interaction matrix $J=A+D$, which is positive semi-definite by assumption and furthermore does not change the law since the diagonal part does not affect the distribution. Recall that this says that if we sample
\begin{equation*}
    Y = J^{1/2}X+G\quad\quad X\sim \mu\perp G\sim \mathcal{N}(0,I),
\end{equation*}
then we have a representation $\mu_{A,\bm{h}}=\mathbb{E}_{Y}[\mu_{0,J^{1/2}Y+\bm{h}}]$ as a mixture of product distributions. By the Law of Total Covariance, observe that for any $i\neq j$,
\begin{equation*}
    \Cov(X_i,X_j)=\mathbb{E}_Y[\Cov(X_i,X_j\vert Y)]+\Cov_Y(\mathbb{E}[X_i\vert Y],\mathbb{E}[X_j\vert Y])=\Cov_Y(\mathbb{E}[X_i\vert Y],\mathbb{E}[X_j\vert Y]),
\end{equation*}
since conditioned on $Y$, the coordinates of $X$ are independent. Therefore,
\begin{equation*}
    \sum_{i\neq j} c_ic_j\Cov(X_i,X_j)=\sum_{i\neq j} c_ic_j\Cov_Y(\mathbb{E}[X_i\vert Y],\mathbb{E}[X_j\vert Y])\leq \Var\left(\sum_{i=1}^n c_i m_i(Y)\right),
\end{equation*}
where $m_i(Y)=\mathbb{E}[X_i\vert Y]=\tanh(h_i+J_iX+(J^{1/2})_iG)$.

We now turn to bounding this term using the Poincar\'e inequality for $X$ \emph{and} $G$. Indeed, applying the law of total variance now on the pair $(X,G)$,
\begin{equation}
\label{eq:ltv-hubbard}
    \Var\left(\sum_{i=1}^n c_i m_i(Y)\right)=\mathbb{E}_X\left[\Var_G\left(\sum_{i=1}^n c_i m_i(Y)\bigg\vert X\right)\right]+\Var\left(\sum_{i=1}^n c_i\mathbb{E}_G[m_i(Y)\vert X]\right).
\end{equation}
To handle the first term of Equation~\ref{eq:ltv-hubbard}, we may apply the Gaussian Poincar\'e inequality~\cite[Corollary 2.27]{vanhandel} over $G$ to the first term conditional on $X$, and then average over $X$ to find:
\begin{align*}
    \mathbb{E}_X\left[\Var_G\left(\sum_{i=1}^n c_i m_i(Y)\bigg\vert X\right)\right]&\leq \mathbb{E}_{X,G}\left[\left\|\nabla_G\left(\sum_{i}c_im_i(Y)\right)\right\|^2\right]
\end{align*}
To evaluate this, observe that
\begin{align*}
    \nabla_G \left(\tanh(h_i+J_iX+(J^{1/2})_iG)\right)&=(1-\tanh^2(h_i+J_iX+(J^{1/2})_iG))\cdot J^{1/2}\bm{e}_i\\
    &=(1-m_i^2(Y)) \cdot J^{1/2}\bm{e}_i
\end{align*}
where $\bm{e}_i$ is the $i$th standard basis vector. In particular, we obtain
\begin{align*}
    \mathbb{E}_X\left[\Var_G\left(\sum_{i=1}^n c_i m_i(Y)\bigg\vert X\right)\right]&\leq \mathbb{E}_{X,G}\left[\left\|J^{1/2}\left(\sum_i c_i(1-m^2_i(Y))\bm{e}_i\right)\right\|^2\right]\\
    &\leq 2\mathbb{E}_{X,G}\left[\left\|D^{1/2}\left(\sum_i c_i(1-m^2_i(Y))\bm{e}_i\right)\right\|^2\right] &&\text{(since $J\preceq 2D$)}\\
    &=2\sum_{i} d_i c_i^2 \mathbb{E}_{Y}\left[(1-m_i^2(Y))^2\right]\\
    &\leq 2\sum_{i} d_i c_i^2 \mathbb{E}_{Y}\left[1-m_i^2(Y)\right]&&\text{(since $1-m_i^2(Y)\in [0,1]$)}\\
    &=2\sum_{i} d_i c_i^2 \mathbb{E}_{Y}\left[\Var(X_i\vert Y)\right]\\
    &\leq 2\sum_{i} d_i c_i^2 \Var(X_i).
\end{align*}

We now deal with the second term in Equation~\ref{eq:ltv-hubbard}. To simplify it, since $\mu$ satisfies Assumption \hyperref[ass:loc-pi]{$(\mathrm{PI}_{\alpha})$}, we may use the Poincar\'e inequality to undo part of the conditioning:
\begin{equation*}
    \Var\left(\sum_{i=1}^n c_i\mathbb{E}_G[m_i(Y)\vert X]\right)\leq\frac{1}{\alpha}\sum_{k=1}^n \mathbb{E}_{X_{-k}}\left[\Var_{X_k}\left(\mathbb{E}_G\left[\sum_i c_i m_i(Y)\bigg\vert X_{-k}\right]\right)\right]
\end{equation*}
We now treat each of the inner variances. Fix any value of $X_{-k}$ and let $p_{\pm 1}$ denote the conditional probabilities that $X_k=\pm 1$ and define
\begin{equation*}
    \Phi(\bm{z}) = \sum_{i=1}^n c_i\tanh\left(h_i+\sum_{j\neq k}J_{ij}X_j+(J^{1/2})_i\bm{z}\right).
\end{equation*}
Then we can write the conditional variance in the expectation as
\begin{equation*}
    p_{+}p_{-}\left(\mathbb{E}_{G_+}[\Phi]-\mathbb{E}_{G_-}[\Phi]\right)^2,
\end{equation*}
where $G_{\pm}\sim\mathcal{N}(\pm J^{1/2}\bm{e}_k,I)$. 

To handle this expression, we will use the following technical lemma that essentially amounts to bounding a difference of two Gaussian expectations by their $\chi^2$-divergence. The calculation is deferred to the Appendix (\Cref{sec:deferred}).
\begin{lemma}
\label{lem:gauss-shift}
Let $Z\sim\mathcal N(0,I_n)$, let $\bm{x}\in\mathbb R^n$,
and let $f:\mathbb R^n\to\mathbb R$ be continuously differentiable, with the expectations below finite.
For any $t\in[0,1]$, let
\begin{equation*}
P_\pm=\mathcal N(\pm \bm{x},I_n),
\qquad P_t=tP_++(1-t)P_-.
\end{equation*}
Then for all $t\in [0,1]$

\begin{align*}
t(1-t)\left(
\mathbb E_{P_+}f-\mathbb E_{P_-}f
\right)^2\leq
4e^{\|\bm{x}\|^2/2}
\mathbb E_{P_t}\left[(\bm{x}^\top\nabla f)^2\right].
\end{align*}
\end{lemma}

In our application, note that $\|J^{1/2}\bm{e}_k\|^2=d_k\leq M$. Applying Lemma~\ref{lem:gauss-shift}, we obtain
\begin{align*}
    \Var_{X_k}&\left(\mathbb{E}_G\left[\sum_i c_i m_i(Y)\bigg\vert X\right]\right)\leq 4\exp(M/2) \mathbb{E}_{G_{p_+}}\left[\left((J^{1/2}\bm{e}_k)^T\nabla \Phi(G_{p_+})\right)^2\right]\\
    & =4\exp(M/2) \mathbb{E}_{G_{p_+}}\left[\left((J^{1/2}\bm{e}_k)^T\sum_i c_i (1-\tanh^2(h_i + \sum_{j\neq k} J_{ij}X_j+(J^{1/2})_iG_{p_+})) J^{1/2}\bm{e}_i\right)^2\right]\\
    &=4\exp(M/2) \mathbb{E}_{G_{p_+}}\left[\left(\sum_i c_i (1-\tanh^2(h_i + \sum_{j\neq k} J_{ij}X_j+(J^{1/2})_iG_{p_+})) J_{ki}\right)^2\right]
\end{align*}
where $G_{p_+}$ is the same mixture as defined in the lemma. But then taking averages over $X_{-k}$ is the same as averaging over $Y$ in the argument of $\tanh^2$ by definition, and so, writing $\bm{u}(Y)=(c_i(1-m_i^2(Y)))_{i=1}^n$,
\begin{align*}
    \Var\left(\sum_{i=1}^n c_i\mathbb{E}_G[m_i(Y)\vert X]\right)
    &\leq \frac{4\exp(M/2)}{\alpha}\sum_{k=1}^n\mathbb{E}_Y\left[\left(\sum_i J_{ki}u_i(Y)\right)^2\right]\\
    &=\frac{4\exp(M/2)}{\alpha}\mathbb{E}_Y\left[\bm{u}(Y)^T J^2\bm{u}(Y)\right]\\
    &\leq \frac{16M\exp(M/2)}{\alpha}\mathbb{E}_Y\left[\bm{u}(Y)^T D\bm{u}(Y)\right]\\
    &\leq \frac{16M\exp(M/2)}{\alpha}\sum_i d_i c_i^2\operatorname{Var}_\mu(X_i),
\end{align*}
where we used $0\preceq J\preceq 2D\preceq 2MI$, which implies $J^2\preceq 2MJ\preceq 4MD$, and the same conditional variance estimate as above. Combining this with the bound for the first term in Equation~\ref{eq:ltv-hubbard} gives the claim.
\end{proof}

We now show how to apply the previous proposition to obtain both results of Theorem~\ref{thm:loc-ising}:

\begin{proof}[Proof of Theorem~\ref{thm:loc-ising}]
If $\alpha=1$, Corollary~\ref{cor:far-decay} gives $\varphi(s)\geq r$ for $s>0$, so both conclusions follow since their stated exponents are at most $r$. Hence assume $0<\alpha<1$.
In both cases, our goal will be to show that $K(s)/(e^s-1)G(s)$ remains integrable by obtaining stronger bounds near zero leveraging the high-temperature assumptions. Fix nonzero $g\perp \mathcal{P}_{r}$ with respect to $\mu$ throughout the below proof.

Assuming just Assumption
\hyperref[ass:loc-pi]{$(\mathrm{PI}_{\alpha})$} for now, condition on any revealed set $V_s$ and coordinates $X_{V_s}$ and let $\nu$ denote the remaining Ising model on the unmasked coordinates in $F_s:=[n]\setminus V_s$. 
Define 
\begin{equation*}
    K_{\nu}:=\sum_{i\neq j\in F_s} \Cov(g_{V_s\cup\{i\}},g_{V_s\cup\{j\}}\mid X_{V_s})
\end{equation*}
Note that conditioned on $X_{V_s}$, $g_{V_s\cup \{i\}}$ is a linear function of $X_i$ so we may write $g_{V_s\cup \{i\}}=c_i X_i+a_i$ for some fixed values $c_i,a_i$ where $\Var(g_{V_s\cup \{i\}}\vert X_{V_s})=c_i^2 \Var(X_i\vert X_{V_s})$. With this representation, we can write
\begin{equation*}
    K_{\nu}=\sum_{i\neq j\in F_s} c_ic_j \Cov(X_i,X_j\vert X_{V_s}).
\end{equation*}
Since $\nu$ is an Ising model that inherits Assumption
\hyperref[ass:loc-pi]{$(\mathrm{PI}_{\alpha})$}, we may apply Proposition~\ref{prop:cov-quadform} with any diagonal matrix $D_F=\mathrm{diag}(\bm{d})$ satisfying $-D_{F_s}\preceq A_{F_sF_s}\preceq D_{F_s}$ and $M_{F_s}=\max_i d_i$ to obtain
\begin{equation}
\label{eq:det-ub-variances}
    K_{\nu}\leq \left(2+\frac{16M_{F_s}e^{M_{F_s}/2}}{\alpha}\right)
    \sum_{i\in F_s} d_i c_i^2\operatorname{Var}_\mu(X_i\vert X_{V_s})=\left(2+\frac{16M_{F_s}e^{M_{F_s}/2}}{\alpha}\right)
\sum_{i\in F_s} d_i \operatorname{Var}_\mu(g_{V_s\cup \{i\}}\vert X_{V_s}).
\end{equation}

We now split into cases starting from this inequality: first, assume Assumption
\hyperref[ass:loc-bw]{$(\mathrm{BW}_{B})$}. If $B=0$, then $A=0$ and the product-measure bound applies, so assume $B>0$. For this same $F$, we define $D_{F_s}$ by $(D_{F_s})_{ii}=\sum_{j\neq i\in F_s} \vert A_{ij}\vert$. In this case, we may take $M_{F_s}=B$ universally, regardless of the realization of $F_s$. So,
\begin{equation*}
    \left(2+\frac{16M_{F_s}e^{M_{F_s}/2}}{\alpha}\right)
    \sum_{i\in F} d_i c_i^2\operatorname{Var}_\mu(X_i\vert X_{V_s})\leq \left(2+\frac{16Be^{B/2}}{\alpha}\right)
\sum_{i\in F} \left(\sum_{j\neq i\in F_s} \vert A_{ij}\vert\right) \operatorname{Var}_\mu(g_{V_s\cup \{i\}}\vert X_{V_s}).
\end{equation*}
If we now average over $V_s$ and $X_{V_s}$, the preceding inequality with Equation~\ref{eq:det-ub-variances} implies:
\begin{equation*}
    K(s)\leq \left(2+\frac{16Be^{B/2}}{\alpha}\right)\sum_{i=1}^{n}\sum_{j\neq i}\mathbb{E}\left[\mathbf{1}\{i,j\not\in V_s\}\vert A_{ij}\vert \operatorname{Var}_\mu(g_{V_s\cup \{i\}}\vert X_{V_s})\right].
\end{equation*}
From this point, the analysis is identical to the bounded degree case of Theorem~\ref{thm:loc-bd}: on the event $\mathbf{1}\{i,j\not\in V_s\}$ which is independent of the value of $X$, we may apply Lemma~\ref{lem:two-poincare} and the law of total variance in the same way:
\begin{align*}
    \sum_{i=1}^{n}\sum_{j\neq i}&\mathbb{E}\left[\mathbf{1}\{i,j\not\in V_s\}\vert A_{ij}\vert \operatorname{Var}_\mu(g_{V_s\cup \{i\}}\vert X_{V_s})\right]\\
    &\leq \frac{1}{\alpha}\sum_{i=1}^{n}\sum_{j\neq i}\mathbb{E}\left[\mathbf{1}\{i,j\not\in V_s\}\vert A_{ij}\vert \left(\operatorname{Var}_\mu(g_{V_s\cup \{i,j\}}\vert X_{V_s\cup\{i\}})+\operatorname{Var}_\mu(g_{V_s\cup \{i,j\}}\vert X_{V_s\cup\{j\}})\right)\right]\\
    &=\frac{2}{\alpha}\sum_{i=1}^{n}\sum_{j\neq i}\mathbb{E}\left[\mathbf{1}\{i,j\not\in V_s\}\vert A_{ij}\vert \operatorname{Var}_\mu(g_{V_s\cup \{i,j\}}\vert X_{V_s\cup\{j\}})\right]\\
    &=\frac{2}{\alpha}(e^s-1)\sum_{i=1}^n\sum_{j\neq i}\mathbb{E}\left[\mathbf{1}\{i\not\in V_s, j\in V_s\}\vert A_{ij}\vert \operatorname{Var}_\mu(g_{V_s\cup \{i\}}\vert X_{V_s})\right]\\
    &\leq \frac{2}{\alpha}(e^s-1)B\mathbb{E}\left[\sum_{i\not\in V_s} \Var_{\mu}(g_{V_s\cup \{i\}}\vert X_{V_s})\right]\\
    &=\frac{2}{\alpha}(e^s-1)B\cdot G(s).
\end{align*}
The third step is the same change of measure argument on $j,V_s$ to include it at a cost of $e^s-1$, and the last is obtained by simply summing over $j\neq i$ and using the deterministic bound of $B$. Putting this all together, we thus obtain:
\begin{equation*}
    \frac{K(s)}{(e^s-1)G(s)}\leq \underbrace{\frac{2}{\alpha}B\left(2+\frac{16Be^{B/2}}{\alpha}\right)}_{:=\gamma_{\alpha,B}}
\end{equation*}

We now employ the same argument to lower bound $\varphi(s)$ as before: the same argument shows that for all $s\geq 0$,
\begin{equation*}
    \varphi(s)\geq r\cdot \exp(-\gamma_{\alpha,B} t)\cdot (1-\exp(-t))^{1/\alpha-1}.
\end{equation*}
Then with $\eta_{\alpha}:=1/\alpha-1$, by an identical application of the technical estimate Lemma~\ref{lem:lb-optimizer} to optimize this lower bound over all $t\geq 0$, we obtain a lower bound
\begin{equation*}
    \varphi(s)\geq r\cdot \left(\frac{\min\{\gamma_{\alpha,B},\eta_{\alpha}\}}{e(\gamma_{\alpha,B}+\eta_{\alpha})}\right)^{\min\{\gamma_{\alpha,B},\eta_{\alpha}\}}
\end{equation*}
Integrating this bound as before leads to the final conclusion.

We now turn the case where we instead assume Assumption \hyperref[ass:loc-mf]{$(\mathrm{MF}_{\omega,L})$}. We follow an identical argument on the conditional Ising model $\nu$, but with $D_{F_s}=\|A_{F_sF_s}\|_{\text{op}}I_{F_s}$, which we can bound uniformly over $F_s$ by the Cauchy interlacing theorem. Moreover, the profile assumption gives $\|A\|_{\text{op}}\leq\omega(1)$, so we may take $M:=\omega(1)$ in the prefactor. Equation~\ref{eq:det-ub-variances} then reads as
\begin{equation*}
    K_{\nu}\leq \left(2+\frac{16Me^{M/2}}{\alpha}\right)\cdot \|A_{F_sF_s}\|_{\text{op}}\underbrace{\sum_{i\in F_s}  \operatorname{Var}_\mu(g_{V_s\cup \{i\}}\vert X_{V_s})}_{G_\nu}.
\end{equation*}
Averaging over $V_s,X_{V_s}$, we obtain
\begin{equation}
\label{eq:mf-ub1}
    K(s)\leq \left(2+\frac{16Me^{M/2}}{\alpha}\right) \mathbb{E}\left[\|A_{F_sF_s}\|_{\text{op}}G_{\nu} \right]\leq \left(2+\frac{16Me^{M/2}}{\alpha}\right) \mathbb{E}\left[\omega((\vert F_s\vert-1)_+/n)G_{\nu} \right].
\end{equation}
We now need to decouple these terms in the expectation, which we will do using a subtle Jensen argument. Indeed, observe that
\begin{align}
    \mathbb{E}\left[\omega((\vert F_s\vert-1)_+/n)G_{\nu} \right]&=G(s)\mathbb{E}\left[\omega((\vert F_s\vert-1)_+/n)G_{\nu}/G(s) \right] \nonumber\\
    &=G(s)\mathbb{E}'[\omega((\vert F_s\vert-1)_+/n)]\nonumber\\
    &\leq G(s)\omega\left(\mathbb{E}'[(\vert F_s\vert-1)_+/n]\right)\nonumber\\
    &=G(s)\omega\left(\mathbb{E}\left[\left(\frac{(\vert F_s\vert-1)_+}{n}\right)G_{\nu}\right]/G(s)\right) \label{eq:jensen-bound}
\end{align}
Here, $\mathbb{E}'$ is defined as the tilted distribution over $(X,V_s)$ using the Radon-Nikodym derivative $G_{\nu}/G(s)$ (which is nonnegative and mean 1 under the original sampling), and we further use Jensen to the concave function $\omega$.

To proceed further, define $\omega(t)=\omega(1)$ for $t\geq 1$; this will not be important in our application, but this preserves nondecreasingness and ensures the following inequalities will be well-defined and still hold. First, we can bound the expectation in an identical way to the previous case using Lemma~\ref{lem:two-poincare} and change of measure as before: 
\begin{align*}
    \mathbb{E}\left[\left(\frac{(\vert F_s\vert-1)_+}{n}\right)G_{\nu}\right]&=\mathbb{E}\left[\left(\frac{(\vert F_s\vert-1)_+}{n}\right)\sum_{i=1}^{n}  \operatorname{Var}_\mu(g_{V_s\cup \{i\}}\vert X_{V_s})\right]\\
    &=\frac{1}{n}\sum_{i=1}^n \sum_{j\neq i}\mathbb{E}\left[\bm{1}\{i,j\in F_s\} \operatorname{Var}_\mu(g_{V_s\cup \{i\}}\vert X_{V_s})\right]\\
    &\leq \frac{2(n-1)(e^s-1)}{n\alpha} G(s)\\
    &\leq \frac{2(e^s-1)}{\alpha}G(s).
\end{align*}
Plugging this into Equation~\ref{eq:jensen-bound},
\begin{equation*}
    \mathbb{E}\left[\omega((\vert F_s\vert-1)_+/n)G_{\nu} \right]\leq G(s)\omega\left(\frac{2(e^s-1)}{\alpha}\right).
\end{equation*}
Returning all the way to Equation~\ref{eq:mf-ub1}, we conclude that
\begin{equation*}
    \frac{K(s)}{(e^s-1)G(s)}\leq \underbrace{\left(2+\frac{16 \omega(1)e^{\omega(1)/2}}{\alpha}\right)}_{C_{\ref{prop:cov-quadform}}}\cdot\left(\frac{\omega(2(e^s-1)/\alpha)}{e^s-1}\right).
\end{equation*}
We now give a universal lower bound on $\varphi(s)$ as before. By Corollary~\ref{cor:log-decay} (Equation~\ref{eq:phi-lower-bound}) and the Poincar\'e long-time bound, we know for any value of $t\geq 0$,
\begin{align*}
\varphi(s)&\geq r\exp\left(-\int_0^{\infty} \frac{\max\{K(u),0\}}{(e^u-1)G(u)}\mathrm{d}u\right)\\
&\geq r\exp\left(-C_{\ref{prop:cov-quadform}}\int_0^{t}\frac{\omega(2(e^u-1)/\alpha)}{e^u-1}\mathrm{d}u-\int_t^{\infty}\frac{(1/\alpha-1)}{e^u-1}\mathrm{d}u\right)\\
&= r\exp\left(-C_{\ref{prop:cov-quadform}}\int_0^{t}\frac{\omega(2(e^u-1)/\alpha)}{e^u-1}\mathrm{d}u\right)\cdot (1-e^{-t})^{(1/\alpha-1)},
\end{align*}
where we reuse a previous computation to evaluate the second integral. We can now choose any value of $t$ that we want, and for simplicity, we will choose $t$ such that
\begin{equation*}
    e^t-1=\alpha/2.
\end{equation*}
This implies that the first integral ranges from $0$ to $1$: more precisely, if we set $p=(2/\alpha)(e^s-1)$, then $\mathrm{d}p=(2/\alpha)e^s \mathrm{d}s=(2/\alpha)(1+\alpha p/2)\mathrm{d}s$.
\begin{equation*}
    \int_0^t \frac{\omega(2(e^s-1)/\alpha)}{e^{s}-1}\mathrm{d}s=\int_0^1 \frac{\omega(p)}{p(1+\alpha p/2)}\mathrm{d}p\leq L.
\end{equation*}
Thus, we obtain after simplification that
\begin{equation*}
    \varphi(s)\geq r\underbrace{\exp(-C_{\ref{prop:cov-quadform}}L)\cdot \left(\frac{\alpha}{\alpha+2}\right)^{1/\alpha-1}}_{c_{\omega,L,\alpha}}.
\end{equation*}
Integrating this bound as usual leads to the final conclusion.
\end{proof}

\subsection{Application: The Eigenstructure of High-Temperature Graphical Models}
\label{sec:eigenspaces}
While not our main focus, we remark that there is an natural structural consequence on the eigenstructure of high-temperature graphical models that may be of independent interest. The above results actually give (to our knowledge) new bounds on \emph{all} eigenvalues on the Laplacian of the Glauber dynamics for many graphical models that quantitatively matches that of product distributions. While the eigenfunctions themselves do not appear to have a tractable representation, we show that the low-degree polynomials approximate the low-energy eigenspaces. In this sense, high-temperature spin systems approximate the \emph{entire} eigenstructure of product distributions.

To state this more precisely, recall that the Laplacian operator associated with Glauber dynamics with respect to $\mu$ is given by $\mathcal{L}=\sum_{i=1}^n (I-P_i)$, where $P_i$ is the transition that rerandomizes the $i$th spin conditional on all other values. Then the eigenvalues $\lambda_1(\mathcal{L})\leq \lambda_2(\mathcal{L})\leq\ldots$ satisfy $\lambda_1(\mathcal{L})=0$ (attained by constant functions). For product distributions, it is well-known and easy to see (e.g. by simple tensorization) that the eigenvalues are of the form $k$ for $k=0,\ldots,n$ with multiplicities ${n\choose k}$ with eigenvectors being appropriate degree $k$ polynomials. For the uniform distribution, the eigenvectors are simply the monomial functions $x^S$ with eigenvalue $\vert S\vert$ while for general products, one can take similar products of centered versions of the variables. For general distributions, the Poincar\'e inequality is equivalent to the statement that $0<\alpha\leq \lambda_2(\mathcal{L})$; in particular, the second eigenvalue is of the same order as that of products. We upgrade this bound on the first nontrivial eigenvalue to matching two-sided bounds for all eigenvalues. 

For both of the below results, we will define:
\begin{equation*}
    n(k)=1+\sum_{i=0}^{k-1} {n \choose i}=\mathrm{dim}(\mathcal{P}_{k})+1.
\end{equation*}

We now prove the following two results establishing eigenvalue growth matching that of products for a wide variety of high-temperature spin systems:
\begin{theorem}
\label{thm:eigen-graphical}
    Suppose that a distribution $\mu>0$ on $\{\pm 1\}^n$ satisfies $(c,1)$-decay and moreover, satisfies Assumption \hyperref[ass:loc-bd]{$(\mathrm{BD}_{\Delta})$}. Then for any $n(k)\leq j< n(k+1)$,
    \begin{equation*}
        c\cdot k\leq \lambda_j(\mathcal{L})\leq (\Delta+1)\cdot k.
    \end{equation*}
\end{theorem}

When we further assume that $\mu$ is an Ising model, not necessarily of bounded degree, we can still obtain the same conclusion using a similar, but slightly more specialized argument:
\begin{theorem}
\label{thm:eigen-ising}
    Suppose that an Ising model $\mu_{A,\bm{h}}$ with zero-diagonal $A$ satisfies $(c,1)$-decay and let $q=-\lambda_{\min}(A)$ be the smallest eigenvalue of the interaction matrix. Then there is a constant $C_q>0$ such that, for any $n(k)\leq j< n(k+1)$,
    \begin{equation*}
        c\cdot k\leq \lambda_j(\mathcal{L})\leq C_q\cdot k.
    \end{equation*}
\end{theorem}

We first quickly explain the lower bound, as the argument is identical to both of them and only depends on the decay parameter. For general eigenvalues, recall that since $\langle f,\mathcal{L}f\rangle_{\mu}=\mathcal{E}_{\mu}(f,f)$ by definition, the Courant-Fischer variational principle asserts that one may obtain the eigenvalues as follows:
\begin{equation*}
    \lambda_k(\mathcal{L})=\min_{U:\mathrm{dim}(U)=k}\max_{f\in U}\frac{\mathcal{E}_{\mu}(f,f)}{\Var_{\mu}(f)}=\max_{U:\mathrm{dim}(U)=2^n-k}\min_{f\in U}\frac{\mathcal{E}_{\mu}(f,f)}{\Var_{\mu}(f)}.
\end{equation*}

To prove the lower bound in either case when $n(k)\leq j<n(k+1)$, we simply take $U=(\mathcal{P}_{k})^{\perp}$, the space of functions orthogonal to $\mathcal{P}_{k}$ with respect to $\mu$. If $g\perp P_r$ with $\Var_{\mu}(g)=1$, then applying Corollary~\ref{cor:log-reveals-degree} with $(c,1)$-decay shows that
\begin{align*}
\frac{\mathcal{E}_{\mu}(g,g)}{\Var(g)}&=-\Lambda'(0)
=-\lim_{s\to 0}\frac{\Lambda(s)-1}{s}\geq -\lim_{s\to 0} \frac{\exp(-csr)-1}{s}\geq cr.
\end{align*}
Since this holds for all $g\in U$, this identifies a subspace of functions with co-dimension $n(k)-1$ whose Rayleigh quotient is at least $cr$, which proves the desired lower bound by the minimax principle. 

We now turn to showing the corresponding upper bounds in each setting. The general principle is quite simple: we will use structural properties of the distribution to again transfer the analysis to a \emph{product} distribution which we can easily analyze. 

\begin{proof}[Proof of upper bound, \Cref{thm:eigen-graphical}]
To prove the upper bound, it suffices to show the following converse: if $g\in \mathcal{P}_{k}$ with $\Var_{\mu}(g,g)=1$, then $\mathcal{E}_{\mu}(g,g)\leq (\Delta+1) k$. In that case, the Courant-Fischer minimax theorem implies the bound on the eigenvalue since the dimension of this space is $n(k+1)-1$.

To do this, first note that there exists a partition of $[n]$ with $\Delta+1$ color classes $S_1,\ldots,S_{\Delta+1}$ such that each $S_i$ forms an independent set with respect to the dependency graph $H$ under the bounded degree assumption; this can be seen by the greedy coloring algorithm. Let $W_{\ell}=\cup_{j\neq k} S_j$. Then for each color class $\ell\in [\Delta+1]$, we deterministically have
\begin{equation*}
    \mathcal{E}_{\mu(\cdot\vert X_{W_{\ell}})}(g,g)=\sum_{i\in S_{\ell}} \mathbb{E}[\Var(g(X)\vert X_{-i})\vert X_{W_{\ell}}]\leq k\Var(g(X)\vert X_{W_{\ell}});
\end{equation*}
this is simply because conditional on $X_{W_k}$, the spins in $X_{S_k}$ are independent since they form an independent set, and then applying the improved Poincar\'e inequality for product distributions since $g$ remains degree at most $k$ after conditioning. Taking expectations over $X_{W_{\ell}}$ and then summing over $\ell=1,\ldots,\Delta+1$,
\begin{equation*}
    \sum_{i=1}^n \mathbb{E}[\Var(g(X)\vert X_{{-i}})] \leq k\sum_{\ell=1}^{\Delta+1} \mathbb{E}[\Var(g(X)\vert X_{W_{\ell}})]\leq k(\Delta+1)\Var(g(X)).
\end{equation*}
In the last step, we use the law of total variance to undo the conditioning.
\end{proof}

For the proof of Ising, we will need a slightly more involved reduction to product. Thankfully, we can nonetheless do a similar decomposition  using the familiar Hubbard-Stratanovich transform. The new idea in this application, though, is that Hubbard-Stratanovich gives a \emph{coupling} between the conditional distribution of $X_i$ given $X_{-i}$, as is in the Dirichlet form, and $X_i$ given the random field $Y$ where the coordinates become product. Because of this coupling, we can compare the variances \emph{pointwise} to ensure that the conditional variances of $X_{i}$ in both setting remain close given $X_{-i}$, at which point we can easily compare them and apply the product bound. 

\begin{proof}[Proof of upper bound, \Cref{thm:eigen-graphical}]
Again, it suffices to show the following converse: if $g\in \mathcal{P}_{k}$ with $\Var_{\mu}(g,g)=1$, then $\mathcal{E}_{\mu}(g,g)\leq C_q k$; the constant here only depends on $q=-\lambda_{\min}(A)\geq 0$ and not on $c$.

Let $J=A+q I\succeq 0$ by the definition of $q$. Then by Hubbard-Stratanovich, we have a coupling $(X,Y=J^{1/2}X+Z)$ with $X\sim \mu$ and $Z\sim \mathcal{N}(0,I)$ with the property that conditioned on $Y$, the conditional distribution of $X$ is given by the product distribution $\mu_{0,\bm{h}+J^{1/2}Y}$. To compare Dirichlet forms of $g$, we need to compare the conditional variance of $X$ from the two experiments: first conditioning on $X_{-i}$, and also in the product field given by $Y$. 

Formally, note that conditioned on $X_{-i}$, the external field on $X_i$ is proportional to $h_i+A_i^TX_{-i}$ while conditioned on $Y$, it is given by $h_i+J_iX+(J^{1/2})_iZ=h_i+A_i^TX_{-i}+qX_i+\sqrt{q} g$ where $g\sim \mathcal{N}(0,1)$ independent of $X_{-i}$ since $\|(J^{1/2}_i)\|^2=J_{ii}=q$. Because the weights are proportional to the exponential of the fields, it follows by the Holley-Stroock perturbation principle~\cite{holley1986logarithmic} that the variances can differ by at most $\exp(2qX_i+2\sqrt{q}g)$. Therefore,
\begin{equation*}
    \exp(-2\vert q\vert-2\sqrt{q}\vert g\vert)\Var(X_i\vert X_{-i})\leq \Var(X_i\vert Y) 
\end{equation*}
Since this holds pointwise, we may now average over $Y$ conditionally on $X_{-i}$. Since $g$ depends only on the independent Gaussian vector $Z$, we can apply Jensen to obtain the pointwise bound:
\begin{equation}
\label{eq:local-var-hs-bound}
    \exp(-2q-2\sqrt{q}\mathbb{E}[\vert g\vert])\Var(X_i\vert X_{-i})=\underbrace{\exp\left(-2q-\sqrt{8q/\pi}\right)}_{C^{-1}_q}\Var(X_i\vert X_{-i})\leq \mathbb{E}[\Var(X_i\vert Y)\vert X_{-i}]
\end{equation}
To complete the proof, observe that if we let $\mathsf{D}_if(\bm{x}) = \frac{f(1,\bm{x}_{-i})-f(-1,\bm{x}_{-i})}{2}$ denote the discrete derivative operator which is only a function of $\bm{x}_{-i}$, then
\begin{align*}
    \mathcal{E}_{\mu}(g,g)&=\sum_{i=1}^n\mathbb{E}_{(X,Y)}[\Var(g(X)\vert X_{-i})]\\
    &=\sum_{i=1}^n\mathbb{E}_{(X,Y)}[\Var(X_i\vert X_{-i})\cdot (\mathsf{D}_ig(X_{-i}))^2]\\
    &\leq C_q\sum_{i=1}^n\mathbb{E}_{(X,Y)}[\mathbb{E}_Y[\Var(X_i\vert Y)\vert X_{-i}]\cdot (\mathsf{D}_ig(X_{-i}))^2] &&\text{(Eq. \Cref{eq:local-var-hs-bound})}\\
    &\leq C_q\sum_{i=1}^n\mathbb{E}_{(X,Y)}[\Var(X_i\vert Y)\cdot (\mathsf{D}_ig(X_{-i}))^2] &&\text{(Law of Total Variance)}\\
    &=C_q\sum_{i=1}^n\mathbb{E}_{Y}\left[\mathbb{E}_{X\sim \mu_{0,\bm{h}+J^{1/2}Y}}\left[\Var(X_i\vert Y)\cdot (\mathsf{D}_ig(X_{-i}))^2\right]\right]\\
    &=C_q\mathbb{E}_{Y}\left[\mathcal{E}_{\mu_{0,\bm{h}+J^{1/2}Y}}(g,g)\right]\\
    &\leq C_qk\mathbb{E}_Y[\Var(g(X)\vert Y)] &&\text{(product bound)}\\
    &\leq C_q k \Var_{\mu}(g(X)) &&\text{(Law of Total Variance)}.
\end{align*}
Therefore, $\mathcal{P}_{k}$ exhibits a subspace of dimension $n(k+1)-1$ where the Rayleigh quotient is at most $C_q(k)$; by Courant-Fischer, we conclude that if $j\leq n(k+1)-1$, then $\lambda_j(\mathcal{L})\leq  C_q k$ as desired.
\end{proof}

\subsection{Deferred Proofs}
\label{sec:deferred}
In this section, we give the proofs of deferred results from the previous sections.

\begin{lemma}[\Cref{lem:lb-optimizer}, restated]
    For any $\gamma,\eta>0$, 
    \begin{equation*}
        \max_{t\geq 0} \left\{\exp(-\gamma t)(1-\exp(-t))^{\eta}\right\}=\left(\frac{\gamma}{\gamma+\eta}\right)^{\gamma}\left(\frac{\eta}{\gamma+\eta}\right)^{\eta}.
    \end{equation*}
    This value is always at least
    \begin{equation*}
        \left(\frac{\min(\gamma,\eta)}{e(\gamma+\eta)}\right)^{\min(\gamma,\eta)}.
    \end{equation*}
\end{lemma}
    \begin{proof}
        By taking logarithms, it suffices to find the optimizer of $-\gamma t+\eta \log(1-\exp(-t))$. By differentiating and solving, it is easy to see that the optimizer is attained at
        \begin{equation*}
            t^*=\log(1+\eta/\gamma).
        \end{equation*}
        The resulting objective value is the
        \begin{equation*}
\exp(-\gamma t_*)(1-\exp(-t_*))^{\eta}=\left(\frac{\gamma}{\gamma+\eta}\right)^{\gamma}\left(\frac{\eta}{\gamma+\eta}\right)^{\eta},
        \end{equation*}
        as claimed.

        Let $m=\min\{\gamma,\eta\}$ and $M=\max\{\gamma,\eta\}$. The exact value can be written as
        \[
        \left(\frac{m}{m+M}\right)^m
        \left(1+\frac{m}{M}\right)^{-M}.
        \]
        Since $\log(1+x)\leq x$, the second factor is at least $e^{-m}$. Hence the exact value is at least
        \[
        \left(\frac{m}{e(m+M)}\right)^m,
        \]
        as claimed.
    \end{proof}

    \begin{lemma}[\Cref{lem:gauss-shift}, restated]
Let $Z\sim\mathcal N(0,I_n)$, let $\bm{x}\in\mathbb R^n$,
and let $f:\mathbb R^n\to\mathbb R$ be continuously differentiable, with the expectations below finite.
For any $t\in[0,1]$, let
\begin{equation*}
P_\pm=\mathcal N(\pm \bm{x},I_n),
\qquad P_t=tP_++(1-t)P_-.
\end{equation*}
Then for all $t\in [0,1]$

\begin{align*}
t(1-t)\left(
\mathbb E_{P_+}f-\mathbb E_{P_-}f
\right)^2\leq
4e^{\|\bm{x}\|^2/2}
\mathbb E_{P_t}\left[(\bm{x}^\top\nabla f)^2\right].
\end{align*}
\end{lemma}
\begin{proof}
    The proof is simply to rewrite the difference in expectations in integral from by the fundamental theorem of calculus. Formally:
    \begin{align*}
    \left(
\mathbb E_{P_+}f-\mathbb E_{P_-}f
\right)^2&=(\mathbb{E}_G[f(G+\bm{x})]-\mathbb{E}_G[f(G-\bm{x})])^2\\
&=\left(\mathbb{E}\left[\int_{-1}^1 \bm{x}^T\nabla f(G+s\bm{x})\mathrm{d}s\right]\right)^2\\
    &\leq 2\int_{-1}^1 \mathbb{E}\left[\left(\bm{x}^T\nabla f(G+s\bm{x})\right)^2\right]\mathrm{d}s \quad\text{(Cauchy-Schwarz)}\\
    &\leq 4\exp(\|\bm{x}\|^2/2)\left(\mathbb{E}\left[\left(\bm{x}^T\nabla f(G+\bm{x})\right)^2\right]+\mathbb{E}\left[\left(\bm{x}^T\nabla f(G-\bm{x})\right)^2\right]\right).
\end{align*}

To final step is simply a pointwise comparison of Gaussian densities for $G+s\bm{x}$ with $G+\pm \bm{x}$. Since orthogonal Gaussian directions are independent, it suffices to show that when $n=1$, the density of a standard Gaussian shifted by $sx$ is bounded by $\exp(x^2/2)$ times the sum of the densities of shifts at $\pm 1$.
To see this, note that for any $s$, the density is proportional to.
\begin{align*}
     \exp(-(g-sx)^2/2)&=\exp(-(sx)^2/2)\exp(-g^2/2)\exp(gsx).
\end{align*}
Therefore, by canceling terms the inequality is the same as
\begin{equation*}
    \exp(gsx-(sx)^2/2)\leq \exp(-gx)+\exp(gx),
\end{equation*}
which is clearly true since $s\in [-1,1]$.

Finally, multiplying by $t(1-t)\leq t,1-t$ on both sides, we obtain the final claim.
\end{proof}

\bibliographystyle{alpha}
\bibliography{refs}

\end{document}